\documentclass[journal]{IEEEtran}

\usepackage{cite}
\usepackage{amsmath,amssymb,amsfonts}

\usepackage{times}
\usepackage[bookmarks=true]{hyperref}
\usepackage{units} %

\usepackage{graphicx}

\usepackage{amsthm}

\usepackage{bm}

\usepackage{color}
\definecolor{blue-violet}{rgb}{0.54, 0.17, 0.89}
\definecolor{darkgreen}{rgb}{0.0, 0.2, 0.13}
\definecolor{dartmouthgreen}{rgb}{0.05, 0.5, 0.06}

\usepackage{xspace}

\let\labelindent\relax
\usepackage{enumitem}

\usepackage{amssymb}
    
\usepackage{algorithm}
\usepackage[noend]{algpseudocode}

\usepackage{booktabs}

\newtheorem{theorem}{Theorem}

\newcommand{\rev}[1]{\textcolor{black}{#1}}

\newcommand\mydots{\hbox to 1em{.\hss.\hss.}}

\newcommand{\algName}{\textsc{SEELS}\xspace}
\newcommand{\algname}{\textsc{SEELS}\xspace}
\newcommand{\seels}{\textsc{SEELS}\xspace}

\newcommand{\ccocp}{\textsc{CC-OCP}\xspace}

\newcommand{\ReachOCP}{\textsc{Reach-OCP}\xspace}
\newcommand{\ExploreOCP}{\textsc{Explore-OCP}\xspace}
\newcommand{\EOCP}{(\textbf{EOCP})\xspace}
\newcommand{\alpaca}{\textsc{ALPaCA}\xspace}
\newcommand{\randup}{\textsc{RandUP}\xspace}

\newcommand{\x}{\bm{x}} %
\newcommand{\ac}{\bm{u}} %
\newcommand{\z}{\bm{z}} %
\newcommand{\bmu}{\boldsymbol\mu} %
\newcommand{\bSigma}{\boldsymbol\Sigma}

\newcommand{\bQ}{\bm{Q}}
\newcommand{\bR}{\bm{R}}

\newcommand{\ep}{\bm{\epsilon}} %
\newcommand{\feat}{\bm{\phi}} %
\newcommand{\param}{\bm{\xi}} %
\newcommand{\bLambda}{\boldsymbol{\Lambda}} 
 
\newcommand{\W}{\bm{W}} %
\newcommand{\cost}{\ell}

\newcommand{\costinfo}{\ell_{\textrm{info}}}
\newcommand{\f}{\bm{f}} %
\newcommand{\h}{\bm{h}} %
\newcommand{\g}{\bm{g}} %

\newcommand{\xdim}{n}
\newcommand{\udim}{m}
\newcommand{\phidim}{d}

\newcommand{\N}{\mathcal{N}}

\newcommand{\Linv}{\boldsymbol{\Lambda}^{-1}} %
\renewcommand{\k}{\bm{\theta}} %
\newcommand{\kstar}{\bm{\theta}^*} %
\newcommand{\kstari}{\bm{\theta}_{i}^{*}} %
\newcommand{\kbar}{\bar{\bm{\theta}}} %
\newcommand{\confset}{\mathcal{C}} %

\newcommand{\bp}{\boldsymbol{p}}

\newcommand{\bq}{\bm{q}}
\newcommand{\bv}{\mathbf{v}}

\newcommand{\bF}{\mathbf{F}}
 
\newcommand{\bJ}{\mathbf{J}}

\newcommand{\bS}{\mathbf{S}}
\newcommand{\ba}{\mathbf{a}} %
\newcommand{\bA}{\mathbf{A}} %
\newcommand{\bb}{\mathbf{b}} %
\newcommand{\bpos}{\mathbf{p}} %
\newcommand{\thetadot}{\dot{\theta}} %

\newcommand{\Epsilon}{\mathcal{E}}

\newcommand{\ninfo}{n_{\textrm{info}}}
\newcommand{\Ninfo}{N_{\textrm{info}}}

\newcommand{\G}{\bm{G}}

\newcommand{\cE}{\mathcal{E}}

\newcommand{\X}{\mathcal{X}}

\newcommand{\Xsafe}{\mathcal{X}_{\text{free}}}
\newcommand{\Xobs}{\mathcal{X}_{\text{obs}}}
\newcommand{\Obs}{\mathcal{O}}
\newcommand{\Xgoal}{\mathcal{X}_{\text{goal}}}
\newcommand{\U}{\mathcal{U}}

\newcommand{\E}{\mathbb{E}}
\newcommand{\R}{\mathbb{R}}
 
\newcommand{\Prob}{\mathbb{P}} 

\newcommand{\maxeig}[1]{\bar{\lambda}(#1)}
\newcommand{\mineig}[1]{\underline\lambda(#1)}
\newcommand{\diag}{\textrm{diag}}

\newcommand{\D}{\mathcal{D}}

\newcommand{\norm}[1]{\left\lVert#1\right\rVert}

\newtheoremstyle{exampstyle}
  {1em plus .2em minus .1em}%
  {1em plus .2em minus .1em}%
  {} %
  {} %
  {\bfseries} %
  {.} %
  {.5em} %
  {} %

\newtheorem{mylemma}{Lemma}[section]

\theoremstyle{exampstyle}
\newtheorem{mycorollary}{Corollary}[section]

\theoremstyle{exampstyle}

\theoremstyle{exampstyle}
\newtheorem{preremark3}{Theorem}

\theoremstyle{exampstyle}

\theoremstyle{exampstyle}

\theoremstyle{exampstyle}
\newtheorem{myassumption}{Assumption}%

\usepackage{multirow}
\usepackage{tabularx}

\newcolumntype{C}{>{\centering\arraybackslash}p{35mm}}
\newcolumntype{L}{>{\centering\arraybackslash}p{15mm}}
\newcolumntype{S}{>{\centering\arraybackslash}p{8mm}}
\newcolumntype{H}{>{\centering\arraybackslash}p{14mm}}
\newcolumntype{G}{>{\centering\arraybackslash}p{16mm}}
\newcolumntype{F}{>{\centering\arraybackslash}p{12mm}}
\newcommand\Tstrut{\rule{0pt}{2.6ex}}         %
\newcommand\Bstrut{\rule[-0.9ex]{0pt}{0pt}}   %

\usepackage{stackengine}

\usepackage{wrapfig}

\usepackage{placeins}

\definecolor{somegray}{rgb}{0.5, 0.5, 0.5}
\newcommand{\darkgrayed}[1]{\textcolor{somegray}{#1}}
\makeatletter
\newcommand*\titleheader[1]{\gdef\@titleheader{#1}}
\AtBeginDocument{%
  \let\st@red@title\@title
  \def\@title{%
    \vskip-2em
    \bgroup\normalfont\large\centering\@titleheader\par\egroup
    \vskip1.5em\st@red@title}
}
\makeatother

\titleheader{\darkgrayed{This paper has been accepted for publication in the\\
IEEE Transactions on Robotics (T-RO).
\copyright 2022 IEEE}}

\title{Safe Active Dynamics Learning and Control:  
A Sequential Exploration-Exploitation Framework}

\author{Thomas Lew, Apoorva Sharma, James Harrison, Andrew Bylard, and Marco Pavone
\thanks{The authors are with 
Stanford University, Stanford, CA 94305-4035 USA. Contact: \texttt{\{thomas.lew, apoorva, jharrison, bylard, pavone\}@stanford.edu}.}
}

\begin{document}
\maketitle

%
%
%
%

\markboth{}{Lew \MakeLowercase{\textit{et al.}}: Safe Active Dynamics Learning and Control: A Sequential Exploration-Exploitation Framework}

\maketitle

\begin{abstract}
Safe deployment of autonomous robots in diverse scenarios requires agents that are capable of efficiently adapting to new environments while satisfying constraints. 
In this work, we propose a practical and theoretically-justified approach to maintaining safety in the presence of dynamics uncertainty. 
Our approach leverages Bayesian meta-learning with last-layer adaptation. The expressiveness of neural-network features trained offline, paired with efficient last-layer online adaptation, enables the derivation of tight confidence sets which contract around the true dynamics as the model adapts online. We exploit these confidence sets to plan trajectories that guarantee the safety of the system. Our approach handles problems with high dynamics uncertainty, where reaching the goal safely is potentially initially infeasible, by first \textit{exploring} to gather data and reduce uncertainty, before autonomously \textit{exploiting} the acquired information to safely perform the task. 
Under reasonable assumptions, 
we prove that our framework guarantees the high-probability satisfaction of all constraints at all times jointly, i.e. over the total task duration. 
This theoretical analysis also motivates two regularizers of last-layer meta-learning models that improve online adaptation capabilities as well as performance by reducing the size of the confidence sets. 
We extensively demonstrate our approach in simulation and on hardware.
\end{abstract}


\IEEEpeerreviewmaketitle

\section{Introduction}

Deploying autonomous robotic systems in safety-critical applications requires agents that are capable of adapting to partially unknown environments while satisfying constraints. 
For example, an autonomous robot assisting astronauts in space must be able to handle a priori unknown payloads while respecting velocity constraints and avoiding collisions with obstacles. 
This requires quantifying the uncertainty in the environment and factoring it into the decision process to take reliable actions. 
As the initial levels of uncertainty may be very high as the agent encounters a new environment, this also requires choosing actions that actively reduce the level of uncertainty to eventually be capable to safely perform the given task. 
For instance, a robot grasping a new payload should first identify the properties of its coupled dynamics before attempting to transport the payload to a destination.

Despite recent rapid progress on facets of this problem in the fields of data-driven control, machine learning, and reinforcement learning, the overall problem of safely controlling a robotic system in uncertain environments remains a challenge. Specifically, it is difficult to (1) obtain statistical models which are general yet offer a tight characterization of uncertainty, (2) design control laws that formally guarantee the safety of uncertain nonlinear dynamical systems, and (3) ensure that these guarantees hold even as the system adapts and reduces uncertainty using online data. 
Designing a robotic system that is capable of actively and safely exploring its environment  without human intervention remains an open challenge.

\begin{figure}[!t]%
\centering
\includegraphics[width=1\linewidth,trim=0 2 0 5, clip]{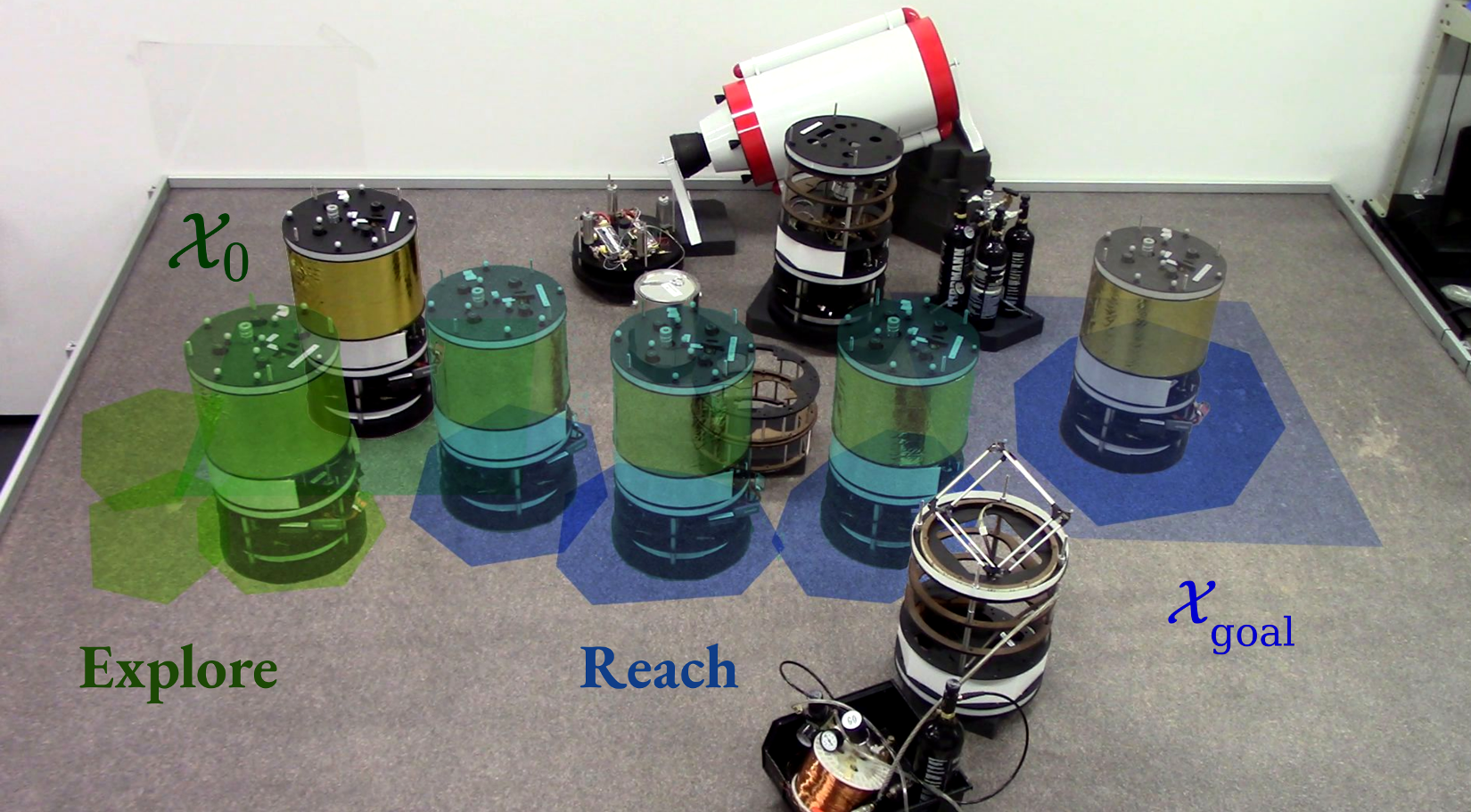}
\vspace{-4mm}
\caption{We propose a learning and control framework to tackle problems where uncertainty over a system's dynamics is potentially initially too high to safely perform a given task (e.g., transporting an uncertain payload). 
By leveraging meta-learning and full-horizon reachability-aware chance-constrained optimal control, our algorithm enables robotic systems (e.g., the planar spacecraft free-flyer platform shown above) to actively \textcolor{dartmouthgreen}{\textit{explore}} to infer dynamics and reduce uncertainty, until it is possible to safely \textcolor{blue}{\textit{reach}} the goal $\Xgoal$. We prove that our framework guarantees high-probability  satisfaction of all constraints at all times jointly and that safety is guaranteed even if no feasible solution to the problem exists, %
e.g., if the only path to the goal is blocked by an obstacle. 
\vspace{-5mm}
}
\label{fig:hardware_intro_motivating}
\end{figure}

\textbf{Contributions:} We propose a practical framework for the safe learning and control of uncertain nonlinear systems. 
\rev{Our} approach is capable of tackling chance-constrained problems that are initially infeasible due to high dynamics uncertainty by autonomously exploring and learning online 
to reduce uncertainty before performing the task, \rev{as shown in Figure \ref{fig:hardware_intro_motivating}}. 
This framework is enabled by the following contributions:
\begin{itemize}[leftmargin=3.5mm]
  \setlength\itemsep{0.5mm}
    \item \textit{Dynamics modeling}: 
    We show that last-layer meta-learning models can learn tight and structured  representations of uncertain dynamics  from data and can efficiently reduce uncertainty online.  Critically, we derive adaption guarantees of such models in the form of confidence sets 
    guaranteed with high probability to contain the true dynamics throughout the online learning process. 
    These guarantees rely on two realistic and intuitive assumptions on the quality of the \textit{offline} meta-training process which can be empirically validated before deploying the system. This theoretical analysis also suggests two regularizers \rev{to meta-train adaptive}  models with improved performance and more expressive features. %
    \item \textit{Autonomous exploration}:  
    We combine these confidence sets with reachability-aware trajectory optimization to tackle problems with high levels of uncertainty that are initially infeasible. 
Specifically, we use the current uncertainty level to autonomously switch from active learning, to reduce dynamics uncertainty (the exploration phase), to executing the desired task (the exploitation phase) once it is feasible. 

    \item \textit{Safety}: 
    In this work, we define safety in terms of the satisfaction of all constraints over the entire operation time. As such, 
    we prove that our framework guarantees the high-probability satisfaction of all constraints at all times \textit{jointly}, i.e. over the total task duration. This contrasts with prior work which instead enforces pointwise chance constraints (see Section \ref{sec:related_work}). 
    We also prove that the algorithm is continually recursively feasible with high probability,  
    which is crucial to enable autonomous operation. 
    Our guarantees are independent of the (unknown) total time required to solve the task, and safety still holds even if the perfect-information problem is impossible to solve, e.g., if the only path to the goal is blocked by an obstacle.
\end{itemize}

We validate our approach on two challenging highly uncertain nonlinear systems 
and show that we are able to autonomously and safely infer dynamics to accomplish goal-reaching tasks. 
We also provide hardware experiments demonstrating the reliability and real-time capabilities of our approach, as well as the practical applicability of dynamics meta-learning models with last-layer adaptation.

\textbf{Organization:} We first discuss related work in Section \ref{sec:related_work}. Section \ref{sec:prob_form} describes our chance-constrained problem formulation. We present our contributions to the meta-learning model used to capture uncertain dynamics in Section \ref{sec:background} and our control approach in Section \ref{sec:seels}.  
Finally, we present results in Sections \ref{sec:results:sim} and \ref{sec:hardware} and conclude in Section \ref{sec:conclusion}.

{

\textbf{Notation}: 
$\chi^2_d(p)$ denotes  the $p$-th quantile of the $\chi^2$ distribution with $d$ degrees of freedom. 
For any $\ba,\bb \in \R^d$ and $\bA$ a $d{\times} d$ positive definite matrix, $\|\ba\|_{\bA}^2 =  \ba^\top\bA\ba$
and $\bar\lambda(\bA)$ ($\underline{\lambda}(\bA)$) denotes the maximum (minimum) eigenvalue of $\bA$. 
$\langle \ba, \bb \rangle$ denotes the inner product of $\ba$ with $\bb$ (for $\ba,\bb \in \R^d$,  $\langle \ba, \bb \rangle= \ba^\top \bb$).  
\rev{We denote a conjunction (logical \scalebox{0.95}{\textrm{AND}}) by $\wedge$ and a disjunction (logical \scalebox{0.95}{\textrm{OR}}) by $\vee$.}
}

\section{Related Work}\label{sec:related_work}

Our approach is enabled by two key components: a dynamics meta-learning model with adaptation guarantees and a sequential exploration and exploitation control algorithm. 

\textbf{Choosing a model:} In contrast to model-free approaches to learning-based control, model-based methods generally provide better sample efficiency while enabling guarantees on constraint satisfaction and stability \cite{recht2019, deisenroth2015}. 
These model-based methods rely on the choice of a dynamics model---e.g., neural networks \cite{levine2016end}, squared-exponential kernel Gaussian processes \cite{deisenroth2015}, or linear models \cite{Coulson2018DataEnabledPC}---each with associated strengths and weaknesses.
Recent work in the controls community 
has leveraged behavioral systems theory to guarantee stability and probabilistic constraint satisfaction of a non-parametric \rev{model predictive control (MPC)} scheme \cite{Coulson2020Distributional,Berberich2020Robust,Berberich2020}. 
Although such methods have been shown to perform well for nonlinear systems \cite{Coulson2018DataEnabledPC}, their guarantees currently do not extend beyond linear systems
\rev{(see also \cite{Berberich2021} for a recent analysis of the nonlinear setting). }
Alternatively, assuming 
linearly parameterizable dynamics with known nonlinear basis functions allows the design of stable adaptive controllers \cite{Slotine1987,Lopez2021} and planning algorithms with adaptation guarantees 
\cite{Mania2020,Kakade2020}. 
When such structure is not known a priori, purely data-driven  models such as neural networks can be used to derive nonlinear controllers that guarantee stability \cite{shi2019ICRAnnlander}. 
However, these methods require collecting a dataset characterizing the system dynamics throughout the operational regime and would need retraining with a new dataset if the environment or the system changes. Generally, this %
is prohibitively expensive in terms of data requirements.

To tackle problems where data is scarce, meta-learning \rev{\cite{schmidhuber1987evolutionary,santoro2016meta,amit2018meta,finn2017model,RichardsAzizanEtAl2021}} aims to train a model from data over different tasks, making it capable of rapid adaptation given limited data from a given new task.  
In the context of dynamics learning, a task corresponds to a dynamical system and meta-learning consists of training a model capable of efficiently fitting the true system's dynamics with limited data. 
Prior work has successfully applied meta-learning to learning-based control  \cite{finn2017model,NagabandiICLR2019,Belkhale2021}: performing gradient descent on meta-learned neural networks at run time enables fast adaptation with limited data. 
However, the non-convexity of the online adaptation objective %
makes it difficult to provide adaptation guarantees for such gradient descent approaches. \rev{Instead of performing gradient descent online, meta-learning has also been combined with adaptive controllers \cite{RichardsAzizanEtAl2021} yielding efficient closed-loop tracking performance. Further analysis is still necessary to provide formal performance guarantees for this approach in safety-critical applications. }

\rev{To obtain performance guarantees, model adaptivity, and data efficiency,}
our approach \rev{
combines} meta-trained neural network features with online last layer adaptation \cite{HarrisonSharmaEtAl2018}:
offline, data from related environments is used to train a neural network to produce features tailored to the structure of the system; 
online, only adapting the last layer yields a sample-efficient adaptation process with strong guarantees on statistical performance. 
By taking a Bayesian perspective on meta-learning, our model also yields a calibrated prior uncertainty representation that enables quantifying the initial uncertainty of the system. 
Importantly, building on this initial uncertainty quantification, the linear structure of the last layer adaptation process enables constructing confidence sets for the parameters of the model that hold throughout the online learning process.

Within the realm of Bayesian modeling, a common modeling choice is Gaussian processes (GPs) \cite{williams2006gaussian}. GPs with squared-exponential kernels have been widely used to represent dynamics for learning-based control and exploration, as they can represent any bounded continuous function arbitrarily well given enough data. 
While bounds providing similar guarantees as our confidence sets  %
can be derived for these models, such bounds are generally too conservative \cite{Fiedler2021} 
so that scaling constants are heuristically selected %
in experiments \cite{berkenkamp2017safe,koller2018}.  
Alternatively, it is common to assume that the true system dynamics lie in a reproducing kernel Hilbert space (RKHS) determined by finite-dimensional basis functions %
\cite{Mania2020,Kakade2020}. Importantly, this approach enables deriving tighter
bounds over possible models \cite{abbasi2011improved} which we regularize for and use directly in this work. 
In contrast to prior work, we explicitly learn these basis functions and regularize these features to improve generalization of the model.  Moreover, we quantify prior uncertainty in an offline meta-training procedure, creating a model which is both calibrated and expressive enough to represent possible systems. This also allows verifying that representation error is small \textit{offline} before deploying the system\rev{, see Figure \ref{fig:concept}}.

\textbf{Safe dynamics learning and control}  
amounts to the design of a controller that guarantees satisfaction of all constraints. 
Given a statistical model of the system subject to external random disturbances, it is common to use chance constraints to ensure that all specifications are satisfied with high probability $(1\,{-}\,\delta)$, where $\delta\,{>}\,0$ is a small tolerable probability of failure. 
Such chance constraints are often preferred to robust constraints ($\delta\,{=}\,0$), as statistical models predicting unbounded distributions over possible outcomes are typically used in such \rev{applications \cite{blackmore2010,IvanovicPavone2019,hewing2018cautious,CastilloRAL2020}}, in which case requiring zero probability of failure is infeasible. 
Typically, \textit{pointwise} chance constraints  
$\Prob(\x_t\,{\in}\,\X_i) \,{\geq}\, (1{-}\delta)$
are enforced for each $i$-th constraint $\x_t\,{\in}\,\X_i$ at each time $t\,{\geq}\, 0$  \cite{CastilloRAL2020,LewBonalli2020,hewing2018cautious,
Polymenakos2020,
Khojasteh_L4DC20,
ChengKhojasteh2020}, where $\x_t\,{\in}\,\R^n$ denotes the state of the system at time $t$ and $\X_i\,{\subset}\,\R^n$ denotes a subset of feasible states. %
Unfortunately, this approach only guarantees that each $i$-th constraint is satisfied at each timestep $t$ with high probability and does not guarantee safety over the entire operation time, i.e., that all constraints are satisfied at all times with high probability  \cite{FreyRSS2020,SchmerlingPavone2017}. 
This is particularly an issue when the time to complete a given task is a priori unknown, which %
motivates using a single \textit{joint} chance constraint instead, taking the form $\Prob(\x_t\,{\in}\,\X_i\, \forall i\, \forall t) \,{\geq}\, (1{-}\delta)$. 

Instead of allocating the total allowed probability of failure $\delta$ to each $i$-th constraint and each time $t$ using Boole's inequality \cite{blackmore2011chance}, we take a frequentist viewpoint and use the confidence sets over the model parameters to plan robust trajectories that are safe with respect to all possible parameters in these confidence sets. %
This approach is key to guaranteeing the satisfaction of all constraints \textit{trajectory-wise} %
with probability at least $(1{-}\delta)$ despite an unknown time to complete the task and  while switching between exploration to reduce dynamics uncertainty and exploitation to complete the given task.

Reinforcement learning (RL) can also be effective for controlling uncertain systems \cite{levine2016end, Hwangbo2019}, and model-based methods in particular enable an agent to consider its uncertainty over dynamics when choosing actions 
\cite{deisenroth2015}. 
However, standard model-based reinforcement learning methods do not provide sufficient guarantees for maintaining safety during operation as they commonly penalize constraint violation \cite{Hwangbo2019}, whereas safety-aware RL typically considers finite-dimensional state-action spaces, as using continuous state-action spaces remains an open area of research \cite{BerkenkampThesis2018,Chow2020SafeRL}. In contrast to traditional RL, we meta-learn a model and uncertainty representation \textit{offline}, fine-tune it \textit{online} via last-layer adaptation only, and use trajectory optimization and reachability analysis for planning and control. 
This structure allows our framework to rapidly infer an accurate dynamics model while providing strong safety guarantees throughout the control task.

An outline of our active dynamics learning and control framework first appeared in \cite{LewSharmaHarrisonRSS2019}. However, it used pointwise chance constraints evaluated approximately and thus did not provide sufficient guarantees of safety. Subsequently, \cite{NakkaRAL2020} developed a similar trajectory optimization approach to learn dynamics. %
By replacing uncertain parameters with Brownian motion (which has independent increments), this approach makes an approximation that neglects time correlations of parameter uncertainty over the state trajectory. This limits the applicability of the theoretical analysis, as is discussed in more detail in \cite{LewSDE2021}. 
Learning-based control algorithms typically tackle the problems of stabilization around an operating point, 
planning to a goal, 
active system identification \cite{ekalaccuracycobot,Zhang2021}, and 
dual control \cite{klenske2016dual,arcari2019,Kakade2020,Mania2020}. 
These approaches outperform methods that neither consider uncertainty nor leverage learning-based components. Still, they rely on the assumption that the problem is initially feasible given the current system's uncertainty characterization (e.g., learning-based MPC requires the first optimization problem to be feasible). 
In contrast to prior work, we consider the problem of \textit{eventually} performing a task which is potentially initially infeasible with the current information about the system. Actively learning the system's properties is thus \textit{necessary} to reduce uncertainty before performing the task. Importantly, we explicitly guarantee the satisfaction of all safety constraints at all times jointly with high probability, 
even if the problem is infeasible with perfect information.

\begin{figure*}[t]
\begin{minipage}{1\textwidth}
    \centering
    \includegraphics[width=1\linewidth]{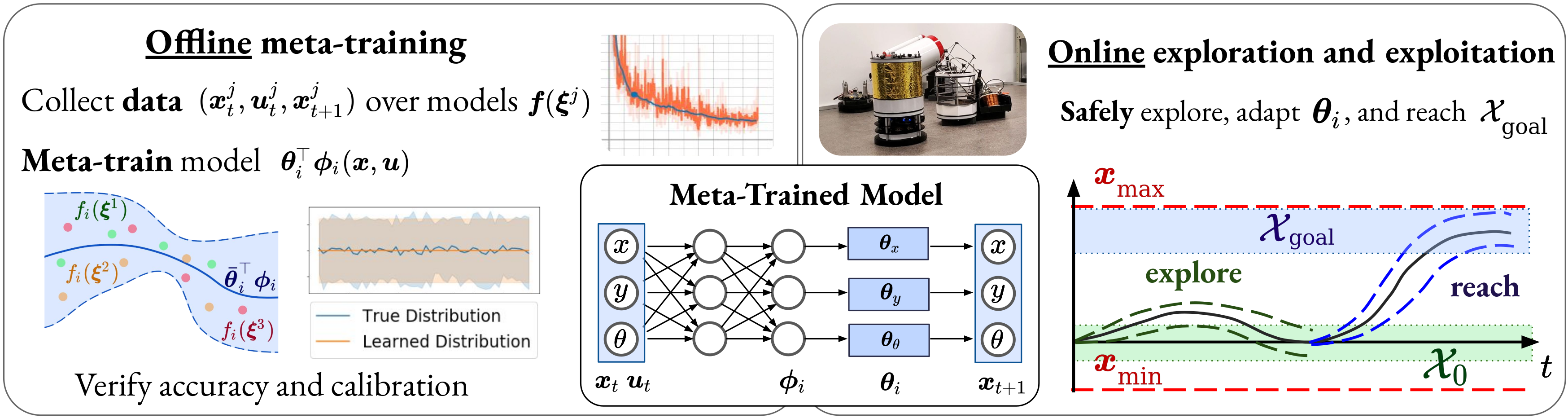}
    \vspace{-4mm}
    \caption{To guarantee trajectory-wise constraints satisfaction and reach a goal region $\Xgoal$ despite uncertain dynamics $\f$ in \eqref{eq:full_problem_dynamics}, our framework consists of an offline phase  \rev{(left)}, where a dataset over multiple models is used to meta-train a model with a calibrated uncertainty characterization\rev{, so that the learned distribution over the dynamics $\k_i^\top\feat_i(\x,\ac)$ accurately represents the true distribution of the dynamics $\f(\param)$}. 
    Then, the system is deployed \rev{(right)} and adapts the last layers $\k_i$ of the meta-trained model \rev{by safely and autonomously exploring} the environment to decrease its uncertainty, \rev{before safely reaching} the goal. 
    }
    \label{fig:concept}
\end{minipage}
\vspace{-2mm}
\end{figure*}

\section{Problem Formulation}\label{sec:prob_form}

The goal of this work is to enable robots to safely navigate from an initial state $\rev{\x(0)\in\R^n}$ to a goal region $\rev{\Xgoal\subseteq\R^n}$ despite highly uncertain dynamics while minimizing a chosen cost metric $\ell(\cdot)$ (e.g., fuel consumption). 
We write the state of the agent at time $t\,{\in}\,\mathbb{N}$ as $\x_t\in\R^{\xdim}$ and the control input as $\ac_t\in\R^{\udim}$. 
The system follows dynamics $\x_{t+1} = \h(\x_t, \ac_t) + \g(\x_t,\ac_t,\param) + \ep_t$, where $\h$ is known and corresponds to prior knowledge of the system, whereas $\g$ is unknown with unobserved parameters $\param$, 
and \rev{$\ep_t$} are stochastic disturbances. 
At the beginning of each episode $j$, the parameters are sampled $\param^j \,{\sim}\, p(\param)$ and are fixed for the episode duration\rev{, so that the uncertain dynamics $\h(\cdot,\cdot)+\g(\cdot,\cdot,\param^j)$ of the system are time-invariant}. These parameters \rev{$\param$} correspond to unknown features that vary between episodes (e.g., the mass of a payload) and introduce correlations over the state trajectory that one must account for to guarantee safety. 
We assume that the disturbances $\ep_t=(\rev{\epsilon_{1,t}},{\mydots},\rev{\epsilon_{n,t}})$ are uncorrelated over $t\,{\in}\,\mathbb{N}$ and $i\,{=}\,1,\mydots,\xdim$ 
and that each $\rev{\epsilon_{i,t}}$ is $\sigma_i$-subgaussian and bounded such that $\rev{\epsilon_{i,t}}\in\mathcal{E}_i$, where $\rev{\mathcal{E}_i\subset\R}$ is a bounded subset.  

Critically, this algorithm should guarantee
\textit{trajectory-wise} constraints satisfaction:  obstacles are avoided ($\x_t\notin\Xobs$ with \rev{$\Xobs\subset\R^n$} the obstacle set) and system constraints are respected ($\x_t\in\X$, $\ac_t\in\U$ with \rev{$\X\subset\R^n$}, \rev{$\U\subset\R^m$} the feasible state and control sets\rev{, which represent, for example, velocity bounds and torque-limited actuators, see Section \ref{sec:results:sim}}) at all times, and the goal $\Xgoal$ is eventually reached. 
Due to the system stochasticity and uncertain dynamics, enforcing all constraints with probability one %
may be challenging without further assumptions (e.g., bounded model mismatch). 
Instead, we enforce all constraints via a single \textit{joint chance constraint} with probability level $(1\,{-}\,\delta)$, where $\delta\,{>}\,0$ is a small tolerable probability of failure. 
The full problem is expressed as
\\[-2mm]

\centerline{\textbf{\hspace{2mm}Chance-Constrained Optimal Control Problem}  \textbf{(\ccocp)}
\vspace{-4mm}}
\begin{subequations}
\label{eq:full_problem}
\begin{gather}
\mathop{\text{min}}_{\x,\ac}
\ \ 
\E \bigg( \sum_{t=0}^{\rev{T-1}} \cost(\x_t, \ac_t)\bigg)%
\label{eq:expected_cost}
\quad \quad
\text{s.t.}
\quad \quad 
\x_0 = \x(0),
\\[2mm]
\hspace{-1mm}
\label{eq:full_problem_dynamics}
\x_{t+1} = \h(\x_t, \ac_t) + \g(\x_t,\ac_t,\param) + \ep_t, 
\ \ 
\mbox{\footnotesize $t\,{=}\,0, \mydots,\rev{T}{-}1,$}
\hspace{-1mm}
\\[2mm]
\label{eq:joint_chance_constraint}
\hspace{-2mm}
\Prob\bigg(
\bigwedge_{t=1}^{\rev{T}} \x_{t} \,{\in}\, \Xsafe 
\ {\wedge}
\bigwedge_{t=0}^{\rev{T}-1}\ac_{t} \,{\in}\, \U 
\ {\wedge}\  
\x_{\rev{T}} \,{\in}\, \Xgoal
\bigg) \geq 1\,{-}\,\delta
,
\hspace{-1mm}
\end{gather}
\end{subequations} 
where $\x\,{=}\,(\x_0,\mydots,\x_{\rev{T}})$ and  
$\ac\,{=}\,(\ac_0,\mydots,\ac_{\rev{T}-1})$ denote the state and control trajectories, 
$\Xsafe\,{=}\,\X \,{\setminus}\, \Xobs\,{\subseteq}\,\R^n$ is the safe set, 
$\Xgoal\,{\subseteq}\, \Xsafe\,{\subseteq}\,\R^n$ is the goal region, 
 $\rev{T}\,{\in}\,\mathbb{N}$ is the total problem duration (possibly infinite), 
 and $\x(0)\,{\in}\,\R^n$ is the initial state of the system.
This problem formulation can be equivalently described as a constrained Markov decision process \rev{(CMDP)} \cite{Geibel2005} with a continuous state and action space, nonlinear stochastic dynamics%
, and a non-convex cost function. 

Satisfying trajectory-wise safety constraints with unknown dynamics is extremely difficult without further information. Thus, our approach relies on the following assumption \cite{koller2018}:
\vspace{-2mm}

\begin{myassumption}[\textbf{A1}]\label{assum:invariant_X0}
The initial state $\x(0)\,{\in}\,\X_0\,{\subseteq}\,\Xsafe$, where $\X_0$ is a \rev{robust} control invariant set for \rev{some} feedback controller $\rev{\boldsymbol\pi}:\X_0\rightarrow\U$\rev{: }%
\rev{if} $\x_t\in\X_0$ \rev{and} $\ac_t=\boldsymbol\pi(\x_t)$\rev{, then }
$\x_{t+1}=\eqref{eq:full_problem_dynamics}\in\X_0$ \rev{for all  $\param\sim p(\param)$ and all $\ep_t$ with $\epsilon_{i,t}\in\Epsilon_i$}.
\end{myassumption}
\vspace{-2mm}

This assumption reflects \rev{prior knowledge} that the system 
is initially stable and satisfies all constraints 
under a nominal controller $\boldsymbol\pi$ (e.g., regulated to a stable point using a simple linear feedback law). %
\rev{%
Computing robust control invariant sets for general nonlinear uncertain systems (e.g., by leveraging additional prior knowledge of the structure and smoothness of the true system and of the support of the distribution $p(\param)$) is an active area of research, see for instance \cite{Akametalu2014,Berkenkamp2016ROA,Chen2018,Jankovic2018,Chen2021}.}

In this work, we make no assumptions on boundedness or smoothness properties of $\g$ (e.g., known Lipschitz constants). 
Instead, we assume that we have access to a dataset of trajectories  generated from sampled parameters $\param^j$. Such information may come from previous operation of a robot in similar environments or data generated from parameterized simulations. This motivates our use of meta-learning to encode this information and characterize the dynamics uncertainty.

\section{Meta-Learning \& Adaptation Guarantees}\label{sec:background}
\subsection{Bayesian meta-learning}
Our approach leverages a model for the unknown dynamics $\g$. To this end, we employ the Bayesian meta-learning architecture presented in \cite{HarrisonSharmaEtAl2018, HarrisonSharmaEtAl2019} referred to as \alpaca. Meta-learning (or ``learning-to-learn'') \cite{schmidhuber1987evolutionary, finn2017model, santoro2016meta} aims to train a model to be capable of rapid adaptation\footnote{The exact mechanism of ``adaptation'' to a task is at the core of the meta-learning algorithm. In MAML \cite{finn2017model} it consists of a gradient step, in recurrent models it occurs via the hidden state dynamics \cite{hochreiter2001learning}, and in \alpaca \cite{HarrisonSharmaEtAl2018} the update consists of Bayesian linear regression on the last layer.} within a distribution of \textit{tasks}. 
In this problem setting, a task corresponds to a dynamical system $\g(\cdot,\cdot,\param)$ parameterized by the unknown parameters $\param$.
\alpaca models the unknown dynamics as $\hat{\g}{=}(\hat{g}_1,{\scalebox{0.9}{$\mydots$}},\hat{g}_n)$, where each $\hat{g}_i{:}\, \scalebox{0.94}{$\R^n{\times}\,\R^m{\times}\,\R^d\,{\rightarrow}\,\R$}$ is defined as
\begin{equation}\label{eq:alpaca_model}
\hat{g}_i(\x, \ac, \k_i) = \k_i^\top\feat_i(\x,\ac), \quad 
i\,{=}\,1,\mydots,n,
\end{equation}
where %
each %
$\feat_i: \R^n\,{\times}\,\R^m\,{\rightarrow}\,\R^{\phidim}$ is a feed-forward neural network and $\k_i\,{\in}\,\R^\phidim$ corresponds to the weights of the last layer. 
The uncertainty in the space of dynamics functions is encoded through a normal distribution on each %
last layer \rev{as} $\k_i \sim \N(\kbar_{i}, \sigma_i^2 \Linv_{i})$\rev{, where $\kbar_i\in\R^d$ denotes the mean parameters and $\bLambda_i\in\R^{d\times d}$ the positive definite precision matrix}. 
As such, each component $\hat{g}_i(\x,\ac, \k_i) = \k_i^\top\feat_i(\x,\ac)$ defines a Gaussian process (GP) with mean function %
$\bar\k_i^\top\feat_i(\x,\ac)$ and kernel %
$\sigma_i^2 \feat_i(\x,\ac)^\top \Linv_{i} \feat_i(\x',\ac')$. 
\rev{As is common in the literature \cite{Akametalu2014,Berkenkamp2016ROA,koller2018}, we opt for independent distributions for the parameters $\k_i$\footnote{\rev{I.e., we choose a Gaussian distribution for the parameters $(\k_1,\mydots,\k_n)$, of mean $(\kbar_1,\mydots,\kbar_n)$ and diagonal covariance with diagonal entries $\sigma_i\bLambda_i^{-1}$. }}. This choice leads to substantial computational benefits during meta-training and to better performance at run-time over the alternative of  maintaining a joint distribution over all the parameters $(\k_1,\dots,\k_n)$, see \cite{HarrisonSharmaEtAl2019,harrison2021uncertainty}. }

The basis functions $\feat_i$ could be chosen as independently-parameterized neural networks. Instead, in this work, 
we use fully-connected $L$-layer neural network basis functions 
$\feat_i(\z)=\rev{\W^{L+1}_i}\tanh\big(\rev{\W^L}\tanh\circ\cdots\circ \tanh(\rev{\W^1}\z)\big)$\rev{\footnote{\rev{We omit biases in the notation for brevity; our implementation includes a bias term in each layer.}}}. Sharing the first $L$ weight matrices across all $n$ state dimensions allows for efficient scalability of the model and faster offline meta-training compared to training $n$ separate neural networks $\feat_i$. Choosing the bounded $\tanh$ activation functions yields smoother mean and variance predictions than common alternatives such as \textrm{ReLU} activations  \cite{HarrisonSharmaEtAl2019,Snoek2015}. 
Finally, the last-layer matrices \rev{$\W_i^{L+1}$} provide an additional degree of freedom which enables each $i$-th dimension of the model to select relevant features%
\footnote{The matrices \rev{$\W_i^{L+1}$} %
also stabilize the offline meta-training process: by weighting intermediate features differently (e.g., when predicting quantities of different units, such as a position versus an angular velocity), this dimension-wise rescaling is not reflected in the precision matrices  $\bLambda_i$ which are then better conditioned. 
By yielding a larger model class, this additional degree of freedom also allows further regularization to meta-train tighter uncertainty representations, see Section \ref{sec:regularizers}.}.
We found that using independent neural networks or sharing less than $L$ weight matrices did not lead to better performance, although our analysis applies to such models as well.

The linear structure of this model allows for efficient online updates whose behavior is well understood.
Given transitions from the system $\{(\x_0, \ac_0, \x_1), \dots, (\x_t, \ac_t, \x_{t+1}) \}$ \rev{and prior parameters $(\kbar_{i,0}, \bLambda_{i,0})$}, 
\rev{the normal distribution of the last-layer parameters is updated according to Bayesian linear regression: the posterior precision matrices and mean vectors are given by}
\begin{equation}\label{eq:bayesian_updates}
\begin{array}{l}
    \bLambda_{i,t} = \Phi_{i,t-1}^\top \Phi_{i,t-1} \,{+}\, \bLambda_{i,0}, 
    \\[1mm]
    \kbar_{i,t} = \bLambda_{i,t}^{-1} ( \Phi^\top_{i,t-1} \G_{i,t} \,{+}\, \bLambda_{i,0} \kbar_{i,0} ), 
\end{array}
    \ \  
    i\,{=}\,1,\mydots,\xdim,
\end{equation}
where $\G_{i,t}^\top {=}\, [x_{i,1} {-} h_i(\x_{0},\ac_{0}), \mydots, x_{i,t} {-} h_i(\x_{t{-}1},\ac_{t{-}1})] \,{\in}\, \R^{t}$ and
$\Phi_{i,t{-}1}^\top {=}\, [\feat_i(\x_0, \ac_0), \mydots, \feat_i(\x_{t{-}1},\ac_{t{-}1})]$  $\,{\in}\, \R^{\phidim \times t}$. 
\rev{Since each $i$-th model $\hat{g}_i$ in \eqref{eq:alpaca_model} is linear in the uncertain parameters $\k_i$, the posterior predictive state distribution %
is also Gaussian. 
Specifically, $(\x_{t+1} \mid \ac_{0:t},\x_{0:t}) \sim \N(\bmu_{t+1}, \bSigma_{t+1})$, where the entries of $\bmu_{t+1}$ and the diagonal covariance matrix $\bSigma_{t+1}$ are
\begin{equation}
\label{eq:alpaca_posterior_predictive}
\begin{array}{l}
\mu_{i,t+1} = h_i(\x_{t},\ac_{t}) + \kbar_{i,t}^\top \feat_i(\x_{t},\ac_{t}) ,
\ \ \ 
i\,{=}\,1,\mydots,\xdim,
\\[1mm]
\Sigma_{i,t+1} = \sigma_i^2 \big( 1 + \feat_i(\x_{t},\ac_{t})^\top \Linv_{i,t} \feat_i(\x_{t},\ac_{t}) \big).
\end{array}
\end{equation}
}

\textbf{Offline}, this model is meta-trained on a dataset of trajectories corresponding to different system dynamics sampled from the distribution over possible systems. 
By backpropagating through the posterior predictive distribution to learn the features $\feat_i$ and the prior parameters $(\kbar_{i,0}, \bLambda_{i,0})$, 
this model translates a dataset of uncertain trajectories into a learned feature space and a calibrated uncertainty characterization. 

More precisely, given $J$ different sampled dynamics and a meta-training horizon $\rev{T_{\textrm{ML}}}$, we consider as meta-training objective the joint marginal log likelihood across the data%
\begin{equation}\label{eq:alpaca_loss}
\mathcal{L}(\D^{\scalebox{0.6}{$J$}}_{\scalebox{0.6}{$\rev{T_{\textrm{ML}}}$}}; \kbar, \bLambda, \feat)
=
\sum_{j=1}^J \sum_{t=0}^{\rev{T_{\textrm{ML}}}{-}1}
\log p(\x_{t+1}^j \mid \ac_{0:t}^j,\x_{0:t}^j),
\end{equation}
where 
    $\D^J_{\rev{T_{\textrm{ML}}}}\,{=}\, \{(\x_t^j, \ac_t^j, \x_{t+1}^j)_{t=0}^{\rev{T_{\textrm{ML}}}{-}1}\}_{j=1}^J$ is the meta-training dataset 
    and
    each $j$-th trajectory $(\x_t^j, \ac_t^j, \x_{t+1}^j)_{t=0}^{\rev{T_{\textrm{ML}}}{-}1}$ 
    corresponds to the $j$-th sampled system $\g(\cdot,\cdot,\param^j)$ parameterized by $\param^j \,{\sim}\, p(\param)$.  
\rev{Since each posterior distribution in \eqref{eq:alpaca_loss} is Gaussian as defined by \eqref{eq:alpaca_posterior_predictive}, the total meta-training loss %
can be easily differentiated with respect to the model parameters using modern automatic differentiation packages \cite{HarrisonSharmaEtAl2018}. During offline meta-training, we optimize this loss using stochastic gradient descent, evaluating \eqref{eq:alpaca_loss} on minibatches subsampled from $\D^J_{T_{\textrm{ML}}}$.}

\vspace{1mm}

\textbf{Online}, using transition tuples $\{ (\x_\tau, \ac_\tau, \x_{\tau+1}) \}_{\tau=0}^{t-1}$ from the true system, we only adapt the last layer parameters $\k_i$ using \eqref{eq:bayesian_updates}. This restriction on the \textit{online} behavior of the model allows deriving strong adaptation guarantees under reasonable assumptions on the quality of the \textit{offline} training process.

\subsection{Probabilistic adaptation guarantees}\label{sec:probabilistic_adapt_guarantees}
Our first contribution is providing probabilistic adaptation guarantees for this model in the form of uniformly calibrated confidence sets which hold under the following assumptions, which are reasonable for common robotic applications. 

\vspace{-2mm}

\begin{myassumption}[Capacity of meta-learned \rev{model})  (\textbf{A2}]\label{assum:capacity_model}
For all $\param$, \rev{$i \,{=}\, 1,{\mydots},n$}, there exists $\kstari(\param) \in \R^{\phidim}$ such that 
$$ \kstari(\param)^\top \feat_i(\x, \ac)  
=
g_i(\x,\ac,\param) \quad  
\forall\x \in \X, \
\forall\ac \in \U.
$$
\end{myassumption}

\begin{myassumption}[Calibration of meta-learned prior) (\textbf{A3}]\label{assum:calib}
\hspace{1.2cm}\\
For $\param \,{\sim}\, p(\param)$, all $i\,{=}\,1,{\mydots},\xdim$, 
and $\delta_i\,{=}\,\delta/(2\xdim)$, \\[-3.5mm]
$$\Prob\big(
\|\kstari(\param) - \kbar_{i,0}\|_{\bLambda_{i,0}}^2 \leq \sigma_i^2 \chi^2_{\phidim}(1-\delta_i)
\big)
\geq
(1\,{-}\,\delta_i).$$ 
\end{myassumption}

\vspace{-1mm}

These two key assumptions on the quality of the \textit{offline} meta-learning process state that 
the meta-learning model is capable of fitting the true dynamics (\textbf{A\ref{assum:capacity_model}}) and 
that the prior uncertainty characterization is conservative (\textbf{A\ref{assum:calib}}).%
Since the features $\feat_i$ and the parameters $(\kbar_{i,0}, \bLambda_{i,0})$ of the Gaussian prior \scalebox{0.95}{$\k_i \,{\sim}\, \N(\kbar_{i,0}, \sigma_i^2 \Linv_{i,0})$} are chosen to maximize the log-likelihood of the dataset under $\param^j \,{\sim}\, p(\param)$, 
the offline meta-training procedure presented in    \cite{HarrisonSharmaEtAl2018} is well-suited to yield a model satisfying \textbf{A\ref{assum:capacity_model}} and \textbf{A\ref{assum:calib}}. 
Importantly, these assumptions can be empirically verified through predictive performance on a validation dataset, and techniques such as temperature scaling can be used to ensure calibration in a post-hoc manner \cite{kuleshov2018accurate}. 
\rev{As discussed in details in Appendix  \ref{sec:appendix:proof_missmatch},} \textbf{A\ref{assum:capacity_model}} can be relaxed to hold approximately for some approximation error \rev{$\Delta_i>0$}: similar to \cite{Fiedler2021}, the resulting analysis   would yield larger confidence sets but is beyond the scope of this work.  
\textbf{A\ref{assum:calib}} states that offline meta-training yields a model with a prior variance $\sigma_i^2 \Linv_{i,0}$ that is large enough: with probability at least $(1\,{-}\,\delta_i)$ over $\param ^j\,{\sim}\, p(\param)$, the true parameters $\kstari(\param)$ belong to the ellipsoidal $(1\,{-}\,\delta_i)$-credible sets predicted by the model. 
We note that \textbf{A\ref{assum:calib}} is weaker than assuming that the model is calibrated (i.e., that \textbf{A\ref{assum:calib}} jointly holds for all values of $\delta_i$), a topic extensively studied in the literature \cite{kuleshov2018accurate}. 
In addition to enabling a theoretical analysis, we also use \textbf{A\ref{assum:capacity_model}} and \textbf{A\ref{assum:calib}} to derive two regularizers improving the performance and adaptation capabilities of \alpaca; we refer the reader to Section \ref{sec:regularizers}.

These assumptions, together with the linear structure \eqref{eq:alpaca_model} and adaptation rule \eqref{eq:bayesian_updates} of the meta-learned model, allow us to define adaptive confidence sets over the model parameters $\k_i$ that are guaranteed to contain the true parameters $\kstari$ with high probability, even as the model adapts using \rev{subgaussian}  process-noise-corrupted measurements.

\begin{theorem}[Uniformly Calibrated Confidence Sets]
\label{thm:conf-sets}
Consider the true system  \eqref{eq:full_problem_dynamics}, modeled using the meta-learning model \eqref{eq:alpaca_model}, which is sequentially updated with online data from \eqref{eq:full_problem_dynamics} using \eqref{eq:bayesian_updates}, leading to the updated parameters $(\kbar_{i,t}, \bLambda_{i,t})$ for each dimension $i{=}1,\mydots,\xdim$. 
Let
\begin{gather}
    \beta_{i,t}^\delta 
    \,{=}\,  \sigma_i \Bigg(
    \sqrt{ 2 \log \left(\frac{1}{\delta_i}\frac{\det(\bLambda_{i,t})^{\nicefrac{1}{2}} }{\det(\bLambda_{i,0})^{\nicefrac{1}{2}}}
    \right) }
    {+}
    \sqrt{\frac{\bar{\lambda}(\bLambda_{i,0})}{\underline\lambda(\bLambda_{i,t})}  \chi^2_{\phidim}(1{-}\delta_i)} \Bigg)
    \label{eq:beta_bound}
    \\[-1mm]
    \text{and}\quad
    \confset_{i,t}^{\delta}(\kbar_{i,t},\bLambda_{i,t}) \,{=}\, \{ \k_i\,{\in}\,\R^d \, {\mid} \,
        \norm{\k_i {-} \kbar_{i,t}}_{\bLambda_{i,t}} {\leq}\,  \beta_{i,t}^\delta
     \}.
     \label{eq:conf_set_kappa}
\end{gather}
Then, under Assumptions \ref{assum:capacity_model} and \ref{assum:calib}, for $\delta_i=\delta/(2\xdim)$,
\begin{gather}
\Prob \left(
    \kstari \in \confset_{i,t}^{\delta}(\kbar_{i,t},\bLambda_{i,t})
    \ \ 
    \forall t\,{\geq}\, 0
    \right) \geq (1-2\delta_i). 
\end{gather}
\end{theorem}
A proof of this result\rev{, which follows from \cite[Theorem 2]{abbasi2011improved},} is available in the Appendix. 
Using Boole's inequality\rev{,  the $\sigma_i$-subgaussianity of the disturbances $\ep_t$,} and Assumptions \ref{assum:capacity_model} and \ref{assum:calib}, Theorem \ref{thm:conf-sets} is derived by adapting results from the literature on linear contextual bandits \cite{abbasi2011improved}, which applies a stopping time argument to bound the probability of the \textit{bad} event that the true parameters lie outside $\confset_{i,t}^{\delta}$ \textit{at some time $t$}. 
Theorem \ref{thm:conf-sets} provides confidence sets over model parameters which hold \textit{uniformly over all future times}. 
This is critical to ensure satisfying the trajectory-wise chance constraint 
\eqref{eq:joint_chance_constraint} despite an unknown final time $N$. 
To derive Theorem \ref{thm:conf-sets}, we assume there exist fixed true parameters $\kstari$ corresponding to the true dynamics (\textbf{A\ref{assum:capacity_model}}), and we take a frequentist viewpoint such that the confidence set $\smash{\confset_{i,t}^\delta}$ is a stochastic function of the transition tuples observed online. 
 By defining $\smash{\confset_{i,t}^\delta}$ through time-dependent values of $\smash{\beta_{i,t}^\delta}$, we ensure that the event that $\smash{\confset_{i,t}^\delta}$ excludes the true parameters $\kstari$ \textit{at any time $t$} occurs with probability less than $(\delta/n)$.

Theorem \ref{thm:conf-sets} is similar to results derived for kernel-based GPs \cite{Srinivas2010,Fiedler2021} which rely on a good choice of kernel to satisfy assumptions similar to \textbf{A\ref{assum:capacity_model}} and \textbf{A\ref{assum:calib}}; we discuss similarities in the next section. 
Specifically, the scaling factor $\beta_{i,t}^\delta$ in Theorem \ref{thm:conf-sets} also appears when analyzing kernel-based GPs, but the value of $\beta$ of such models is often too large for practical use and is set to a lower value for experiments, see \cite{berkenkamp2017safe,koller2018}. 
In contrast, our 
model operates within a finite-dimensional feature space and has values of $\beta$ that are practically usable.

\subsection{Practical implications for better meta-training}\label{sec:regularizers}
These theoretical results offer guidance on how to add regularization to the offline meta-learning process for the purposes of safe learning-based control. Specifically, we propose to add two regularizers to the offline meta-training of \alpaca. 
This regularization does not compromise on safety as long as the regularized model still meets Assumptions \ref{assum:capacity_model} and \ref{assum:calib}.

\vspace{1mm}

\textbf{Orthogonal features are better:} 
Our theory relies on \textbf{A\ref{assum:capacity_model}}, which holds as long as each component $g_i(\cdot,\cdot,\param)$ of the true dynamics lies in the Hilbert space
$$\mathcal{H}_{\feat_i}=\{\hat{g}_{i,\k_i}{:}\ \R^{n}\,{\times}\,\R^m\,{\rightarrow}\,\R \ |\  \hat{g}_{i,\k_i}=\k_i^\top \feat_i, \ \k_i\in\R^d\}.$$ 
Equivalently, \textbf{A\ref{assum:capacity_model}} states that each $g_i(\cdot,\cdot,\param)$ belongs to the RKHS with reproducing kernel 
$\sigma_i^2 \feat_i(\cdot,\cdot)^{\top} \bLambda_{i}^{-1}\feat_i(\cdot,\cdot)$. 
Together, \textbf{A\ref{assum:capacity_model}} and \textbf{A\ref{assum:calib}} state that with probability at least $(1-\delta_i)$, each component $g_i(\cdot,\cdot,\param)$ belongs to \rev{a} bounded \rev{subset of the}  RKHS \rev{$\mathcal{H}_{\feat_i}$, defined as}
\begin{align*}
\rev{\mathcal{B}_{\feat_i}^{\mathcal{H}}}
\,{=}\,
\bigg\{\hat{g}_{i,\k_i}{:}\,
\R^{n}{\times}\,\R^m{\rightarrow}\,\R
\ \bigg| %
\begin{array}{l}
\hat{g}_{i,\k_i}=\k_i^\top \feat_i, 
\ \ \ 
\k_i\in\R^d
\\
\|\k_i \,{-}\, \kbar_{i{,}0}\|_{\bLambda_{i,0}}^2 {\leq}\, \sigma_i^2 \chi^2_{\phidim}(1{-}\delta_i)
\end{array}
\hspace{-2mm}
\bigg\}.
\end{align*}
\textbf{A\ref{assum:capacity_model}} and \textbf{A\ref{assum:calib}} are thus comparable to statements in related work on learning-based control with squared-exponential kernel GPs. For instance, \cite{berkenkamp2017safe,koller2018} assume the true dynamics lie in a RKHS of known bounded norm, see also \cite{Srinivas2010}.

This observation suggests that we should meta-learn basis functions $\feat_i$ corresponding to a \textit{large} space of models $\mathcal{H}_{\feat_i}$. 
Since $\mathcal{H}_{\feat_i}$ contains linear combinations of the elements of $\feat_i$, \textbf{A\ref{assum:capacity_model}} thus  motivates learning features that are \textit{orthogonal}. To this end, we propose augmenting the meta-learning objective with the following regularizer:
\begin{equation*}
\mathcal{L}_\mathrm{\perp reg}(\W) = \alpha_\mathrm{\perp reg} \sum_{i=1}^n \sum_{\ell=1}^{L+1} 
    \|
    \rev{\bm{I}^\ell}-\rev{\W_i^{\ell\top}\W_i^\ell}
    \|_F^2,
    \vspace{-1mm}
\end{equation*}
where $\alpha_\mathrm{\perp reg}\,{>}\,0$ controls the strength of this regularization, 
the \rev{$\W_i^\ell$} are the weights of the neural network basis functions $\feat_i$, and each \rev{$\bm{I}^\ell$} is an identity matrix of appropriate dimensions. 
As derived, theoretically analyzed, and empirically validated in \cite{Jia2019}, this regularizer improves generalization.  
Also, it incurs marginal computational overhead compared to the orthogonality regularizer for \alpaca previously proposed in \cite{BanerjeeHarrisonEtAl2020} to encourage interpretability of the model. 
Beyond generalization, for our model with last layer adaptation, orthogonal features also lead to better statistical properties when adapting each $\k_i$ through least-squares estimation with \eqref{eq:bayesian_updates}, see \cite{GreeneEconometrics}. %

\vspace{1mm}

\textbf{$\boldsymbol\beta$-regularization:} 
Another aspect which can be regularized during offline meta-learning to improve downstream performance is the scaling factor $\beta_i$, which controls the size of the confidence sets for the model parameters $\k_i$. 
With expressive neural network features $\feat_i$, multiple settings of the neural network weights and of the prior parameters $(\smash{\kbar_{i,0}, \bLambda_{i,0}})$ can yield a model satisfying Assumptions \ref{assum:capacity_model} and \ref{assum:calib}. As larger values of $\beta_i$ yield more conservative confidence sets, it is preferable to select features and prior parameters that lead to lower values of $\beta_i$. To do so, we propose adding regularization to encourage meta-learning representations with lower values of $\beta_i$ without compromising on safety.

We note from \eqref{eq:beta_bound} that 
the value of each $\beta_i$ depends on the ratio between the maximum and minimum eigenvalues of the prior and posterior precision matrices $\bLambda_i$. 
If $\bar\lambda(\bLambda_{i,0}) \le 1$  (which can be enforced by rescaling the features $\feat_i$), then 
\begin{align*}
\frac{
\bar\lambda(\bLambda_{i,0})}{\underline\lambda(\bLambda_{i,t})}
=
\frac{\bar\lambda(\bLambda_{i,t}^{-1})}{\underline\lambda(\bLambda_{i,0}^{-1}) }
\leq 
\bar\lambda(\bLambda_{i,t}^{-1})
\underline\lambda(\bLambda_{i,0}^{-1})
\leq 
\bar\lambda(\bLambda_{i,t}^{-1})
\bar\lambda(\bLambda_{i,0}^{-1})\rev{,}
\end{align*} 
\rev{for any $t\in\mathbb{N}$}. 
Furthermore, $\smash{\bar\lambda(\bLambda)
    \leq
    \sqrt{\textrm{Tr}(\bLambda^\top\bLambda)}
}$. 
Combining with the above, we propose to regularize  an upper bound of the ratios $\bar\lambda(\bLambda_{i,0})/\underline\lambda(\bLambda_{i,\rev{T_\textrm{ML}}})$ during offline meta-training, namely: 
\begin{equation}\label{eq:beta_regularizer}
    \mathcal{L}_\mathrm{reg}(\bLambda_{0}) = \alpha_\mathrm{reg} \sum_{i=1}^{\xdim} 
    \textrm{Tr}(\bLambda_{i,\rev{T_\textrm{ML}}}^{-T}\bLambda_{i,\rev{T_\textrm{ML}}}^{-1})
    \textrm{Tr}(\bLambda_{i,0}^{-T}\bLambda_{i,0}^{-1}),
\end{equation}
where $\alpha_\mathrm{reg}>0$ controls the strength of this regularization and is selected using a validation dataset and $\rev{T_\textrm{ML}}$ is the meta-training horizon. 
As the model is directly parameterized by the inverse of the precision matrices $\bLambda_i$ \cite{HarrisonSharmaEtAl2019}, this regularizer can easily be added to the standard training loss \eqref{eq:alpaca_loss}.

From \eqref{eq:beta_bound}, we observe that $\beta_i$ also depends on the ratio of determinants of the prior and posterior precision matrices $\smash{\det(\bLambda_{i,t})/\det(\bLambda_{i,0})}$. 
Although a convex regularizer for this term can be derived, we found that including it  did not lead to performance improvements. 
This ratio can be interpreted as capturing the amount of information that the model has gathered online, which is independent of the structure of the prior model. 
Before learning, this ratio is $1$, so the other term composed of the ratio of eigenvalues dominates $\beta_i$. 
We observed that it is during these early stages that the meta-training model and its bounds $\beta_i$ are most conservative, which could explain the importance of the regularizer in \eqref{eq:beta_regularizer}, whereas regularizing the ratio of determinants appears to make little difference.

\section{Sequential Exploration and Exploitation for Learning Safely (SEELS)}\label{sec:seels}

In order to ensure safety at all times with high probability and eventually reach $\smash{\Xgoal}$, 
we  provide a conservative deterministic reformulation of \ccocp, 
propose an algorithm tackling problems with high uncertainty, 
analyze its safety and feasibility properties, and 
discuss implementation.

\subsection{Relaxation of \ccocp using reachability analysis}

In order to reason about how parameter uncertainty manifests in  terms of potential system behavior in the state space (and whether it might violate safety constraints), we translate confidence sets over parameters into confidence \textit{tubes} over trajectories. 
We stress that confidence tubes jointly depending on the parameters are required to satisfy the trajectory-wise chance constraint \eqref{eq:joint_chance_constraint}, since enforcing a pointwise chance constraint at each timestep as in \cite{LewBonalli2020,hewing2018cautious,
Polymenakos2020,
Khojasteh_L4DC20,
ChengKhojasteh2020}  
does not guarantee safety of the whole trajectory \cite{FreyRSS2020,SchmerlingPavone2017}.

\begin{figure*}[!t]%
\begin{minipage}{1\textwidth}
\centering
\includegraphics[width=0.245\linewidth,trim=52 33 35 0, clip]{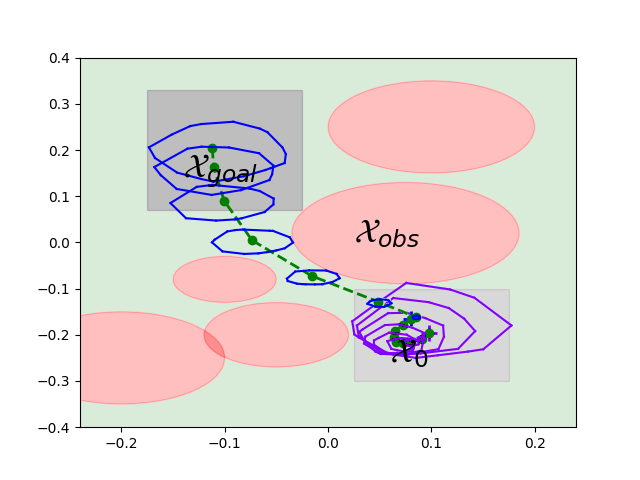}
\includegraphics[width=0.245\linewidth,trim=52 33 35 0, clip]{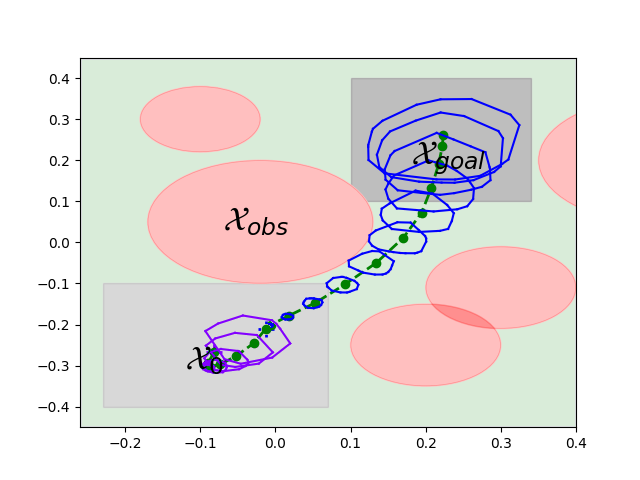}
\includegraphics[width=0.245\linewidth,trim=52 33 35 0, clip]{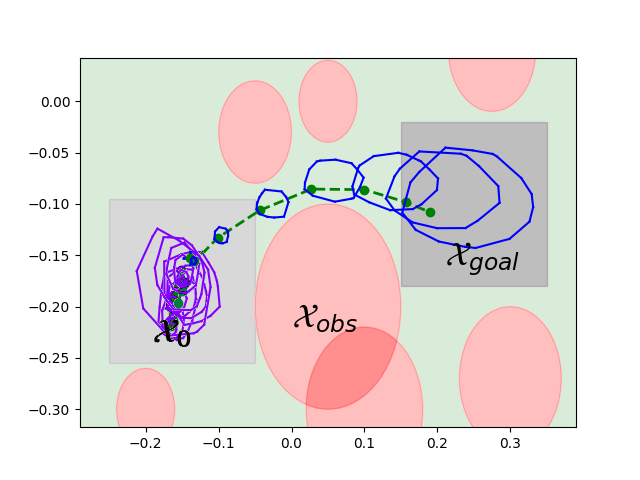}
\includegraphics[width=0.245\linewidth,trim=52 33 35 0, clip]{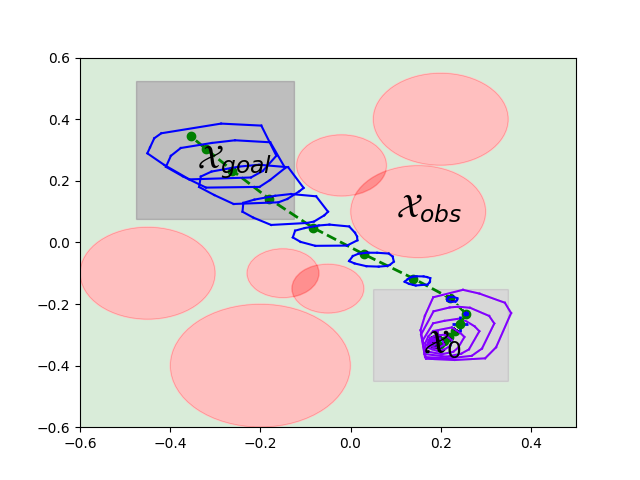}
\caption{Initially, uncertainty is too high to safely reach the goal. Instead, we plan safe information-gathering trajectories to infer the dynamics and reduce uncertainty.  
Once planning to $\Xgoal$ is feasible, the robot can safely reach the goal while satisfying all constraints with high probability.  
Color legend: 
\textcolor{dartmouthgreen}{true trajectory}, 
\textcolor{blue-violet}{reachable sets for exploration} \eqref{eq:confidence_tubes}, 
\textcolor{blue}{reachable sets for exploitation}. This experiment is described in Section \ref{sec:sim:ff}.}
\label{fig:ff_sim_exps}
\end{minipage}
\end{figure*}

We construct these tubes by leveraging the confidence sets $\confset_{i,t}^\delta$.
Indeed, 
a direct consequence of Theorem \ref{thm:conf-sets} is that%
$$
\Prob\big(
\exists \k_i \in \confset_{i,t}^{\delta}(\kbar_{i,t},\bLambda_{i,t})
    \ 
    \forall t%
    \
    \textrm{s.t.}
    \
    \k_i^\top\feat_i \,{=}\, g_i(\cdot,\cdot,\rev{\param})
    \big) 
    \,{\geq}\,(1-2\delta_i). 
$$
This implies that a control trajectory respecting all constraints for all $\k_i\in\confset_{i,t}^{\delta}$ and $\rev{\epsilon_{i,t}}\in\mathcal{E}_i$ also satisfies \eqref{eq:joint_chance_constraint}. 
Thus, to conservatively reformulate \eqref{eq:joint_chance_constraint}, we propose performing reachability analysis using the sets $\confset_{i,t}^{\delta}$ and $\mathcal{E}_i$.
Specifically, given a sequence of open-loop control inputs 
$\ac\,{=}\,(\ac_0,\mydots,\ac_{N{-}1})$, we define the sequence of reachable sets\rev{, for $k{=}1,{\mydots},N$, as}\footnote{This definition closely follows \cite{randSets} for the specific case of a sequence of open-loop control inputs. One can account for a nominal feedback controller as a simple extension to reduce the tube size \cite{koller2018,randSets,LewBonalli2020,hewing2018cautious}. 
In our simulated experiments, we omit feedback to better demonstrate the adaptation capabilities of the meta-training model, show the tightness of the confidence sets, and better verify the safety claims of the framework.}
\begin{gather}
\hspace{-1.4mm}
\X_k^{t,\delta}(\ac) \,{=}\, \bigg\{ 
	\x_k {=} \f(\cdot, \ac_{{\scalebox{.7}{$k{-}1$}}}, \k, \ep_{{\scalebox{.7}{$k{-}1$}}})
	\circ
	\mydots
	\circ
	\f(\x_0, \ac_0, \k, \ep_0)
	\nonumber
\\[-1mm]
\hspace{1mm}
\big|\  
\x_0\,{=}\,\x(t), \,
\k_i\,{\in}\,\confset_{i,t}^{\delta}, \,
\rev{\epsilon_{i,s}}\,{\in}\,\mathcal{E}_i,
\, \mbox{\footnotesize $s{=}1,{\mydots},k{-}1, 
\,  i{=}1,{\mydots},\xdim$} 
\bigg\}
\label{eq:confidence_tubes}
,
\end{gather}
\rev{where }$f_i(\x,\ac,\k,\rev{\ep_t})=h_i(\x, \ac) + \k_i^\top\hspace{-1mm} \feat_i(\x,\ac) + \rev{\epsilon_{i,t}}$.

By quantifying the uncertainty of the system dynamics and of the online learning process using confidence sets over the model parameters $\k$ and by subsequently transforming these sets into reachable sets in the state space, we relax the original chance-constrained problem into a deterministic one:

\vspace{2mm}
\centerline{\rev{$\ReachOCP(N)$}
\vspace{-4mm}
}
\begin{gather*}
\hspace{-1mm}
\mathop{\text{min}}_{\bmu,\ac}
\ 
\sum_{k=0}^{\rev{N-1}} \cost(\bmu_k, \ac_k) 
\
\text{s.t.}
\begin{array}{l}
\bigwedge_{\scalebox{.8}{$k{=}1$}}^N \X_{k}^{t,\delta}\rev{(\ac)} \,{\subseteq}\, \Xsafe, 
\ 
\X_{N}^{t,\delta}\rev{(\ac)} \,{\subseteq}\, \Xgoal,
\\[2mm]
\bigwedge_{\scalebox{.8}{$k{=}0$}}^{N{-}1}
\ac_{k}\,{\in}\,\U,  
\ 
\X_{0}^{t,\delta} \,{=}\, \{\x(t)\},
\end{array}
\end{gather*}
where $\bmu=(\smash{\bmu_0,{\mydots},\bmu_{N}})$ \rev{with $\bmu_0=\x(t)$} are the centers of the reachable sets $\smash{\{\X_k^{t,\delta}\rev{(\ac)}\}_{k{=}1}^N}$ in \eqref{eq:confidence_tubes}, which depend on the control trajectory $\ac$. 
Since the cost function $\ell$ typically penalizes the deterministic open-loop control inputs $\ac$, 
we rewrite the expected cost in \eqref{eq:expected_cost} using a mean-equivalent reformulation as in \cite{LewBonalli2020}\footnote{As we use a sampling-based approach to compute the reachable sets \cite{randSets}, the variance associated with these samples could be used to approximate the variance of the cost, or to minimize a given risk metric \cite{SinghChowEtAl2018b}.}. 
\rev{This approach of considering a tube of reachable states and of evaluating the cost along a nominal trajectory is common in  the literature on robust MPC, see e.g., \cite{RawlingsMayne2013,Limon2009,koller2018}}. Note that the constraint-satisfaction guarantees of our approach do not depend on the chosen reformulation of \eqref{eq:expected_cost}.

\subsection{\algname: algorithm for safe learning and reaching $\Xgoal$}

Due to high dynamics uncertainty, tight control constraints, and long planning horizons, \ReachOCP may be infeasible. 
This motivates a safe learning-based exploration-exploitation framework to sequentially reduce uncertainty and eventually reach $\Xgoal$. 
Our approach is based on a repeated two-phase approach: 
when \ReachOCP is feasible, we enter the \textit{exploitation} phase and plan a safe trajectory to $\Xgoal$ with the current model uncertainty.
In the \textit{exploration} phase, we instead strictly perform safe exploration, planning an information-gathering trajectory that returns to the initial safe invariant set $\X_0$.  
The corresponding robust optimal control problem is written as 

\vspace{2mm}
\centerline{\rev{$\ExploreOCP(N)$}
\vspace{-4mm}
}
\begin{gather*}
\hspace{-1mm}
\mathop{\text{max}}_{\bmu,\ac}
\ 
\sum_{k=0}^{\rev{N-1}} (\cost{+}\costinfo)(\bmu_k, \ac_k) 
\
\text{s.t.}
\begin{array}{l}
\bigwedge_{\scalebox{.8}{$k{=}1$}}^N \X_{k}^{t,\delta} \,{\subseteq}\, \Xsafe, 
\ 
\X_{N}^{t,\delta} \,{\subseteq}\, \X_0,
\\[2mm]
\bigwedge_{\scalebox{.8}{$k{=}0$}}^{N{-}1}
\ac_{k}\,{\in}\,\U,  
\ 
\X_{0}^{t,\delta} \,{=}\, \{\x(t)\},
\end{array}
\end{gather*}
where the sequence of reachable sets  $\{\X_k^{t,\delta}\}_{k{=}1}^N$ satisfies \eqref{eq:confidence_tubes}. 
\ExploreOCP is similar to \ReachOCP but instead uses $\X_0$ as the final set, thus ensuring the system is safe for the next phase with high probability. Also, \ExploreOCP uses an information-gathering cost function $\costinfo$ to encourage visiting states which reduce the remaining uncertainty in the dynamics \begin{equation}
   \label{eq:info_cost}
   \costinfo(\x, \ac) = \frac{1}{2}\sum_{i=1}^\xdim \log(1+\feat_i(\x,\ac)^\top \bLambda_{i,t}^{-1} \feat_i(\x,\ac)).
\end{equation}
In the next section, we derive $\costinfo$ from the mutual information between the unknown dynamics and the observations.

\begin{figure}[!t]
\centering
\vspace{-3mm}
\begin{minipage}{0.98\linewidth}
\begin{algorithm}[H]
\caption{Sequential Exploration and Exploitation for Learning Safely (\algname)}\label{alg:explo_reach_algalg}
\vspace{1mm}
\textbf{Input}: Meta-training model satisfying A.\ref{assum:capacity_model} and A.\ref{assum:calib} \rev{with prior parameters $(\kbar_{i,0},\bLambda_{i,0})_{i=1}^n$, 
min/maximum  exploration and reaching horizons $(\underline{N}_{\textrm{info}},\overline{N}_{\textrm{info}},\underline{N}_{\textrm{reach}},\overline{N}_{\textrm{reach}})$}

\vspace{1mm}
\rev{\textbf{Output}: Updated model parameters $(\kbar_{i,t},\bLambda_{i,t})_{i=1}^n$}
\vspace{1mm}
\begin{algorithmic}[1]
\While {$\x(t) \notin \Xgoal$}
\For {$\rev{N_r} \in \{\underline{N}_{\textrm{reach}},\mydots,\overline{N}_{\textrm{reach}}\}$} \Comment{\textit{Try reaching}}
    \If {\rev{$\ReachOCP(N_r)$} feasible} 
            \State \rev{$(\bmu,\ac) \gets$ Solve$(\ReachOCP(N_r))$}
        \State Apply $\rev{\ac}$ to true system
        \Comment{\textit{Reach}}
	    \State \textrm{Break}
    \EndIf
\EndFor
\State $\vec{\ell}_{\textrm{info}},\vec{\ac}_{\textrm{info}}\gets\{0,\mydots,0\},\{\emptyset,\mydots,\emptyset\}$
\For {$N_{\textrm{e}} \in \{\rev{\underline{N}_{\textrm{info}},\mydots,\overline{N}_{\textrm{info}}}\}$} \Comment{\textit{Explore}}
    \If {\rev{$\ExploreOCP(N_{\textrm{e}})$} feasible} 
\State \rev{$(\bmu^{\textrm{e}},\ac^{\textrm{e}}) \gets$ Solve$(\ExploreOCP(N_{\textrm{e}}))$}
\vspace{1mm}\State \rev{$\vec{\ell}_{\textrm{info}}[N_{\textrm{e}}]\gets\sum_{k=0}^{N_{\textrm{e}}{-}1}\costinfo(\bmu_k^{\textrm{e}},\ac_k^{\textrm{e}})$}
\State \rev{$\vec{\ac}_{\textrm{info}}[N_{\textrm{e}}]\gets\ac^{\textrm{e}}$}
    \EndIf
\EndFor
\State \rev{$N_{\textrm{e}}^{\textrm{best}}\gets \mathop{\arg\max} \vec{\ell}_{\textrm{info}}$}
    \Comment{\textit{Get best} \rev{$N_{\textrm{e}}$}}
    \State \rev{Start recording trajectory data}
\State Apply \rev{$\vec{\ac}_{\textrm{info}}[N_{\textrm{e}}^{\textrm{best}}]$} %
to true system
    \State \rev{Save data $\D_e=\{(\x_{t+k}, \ac_{t+k}, \x_{i,t+k+1})\}_{k{=}0}^{\rev{N_{\textrm{e}}^{\textrm{best}}}{-}1}$}
\State Update parameters $(\rev{\kbar_{i,t+N_{\textrm{e}}^{\textrm{best}}},\bLambda_{i,t+N_{\textrm{e}}^{\textrm{best}}}})$ with $\rev{\D_e}$ 
\State \rev{$t\gets t+N_{\textrm{e}}^{\textrm{best}}$}
\EndWhile %
\State\Return $\rev{(\kbar_{i,t},\bLambda_{i,t})_{i=1}^n}$
\end{algorithmic}
\end{algorithm}
\end{minipage}
\vspace{-4mm}
\end{figure}

\textbf{SEELS:} We summarize our approach in Algorithm \ref{alg:explo_reach_algalg} (\algName) \rev{and show illustrative simulation results in Figure \ref{fig:ff_sim_exps}. \algname} consists of iteratively learning a model of the dynamics by solving \ExploreOCP before reaching $\Xgoal$ when \ReachOCP admits a feasible solution. \rev{We perform a search over multiple horizons $N_r$ and $N_{\textrm{e}}$, as the feasibility of each problem depends on the optimization  horizon (see Sections \ref{sec:seels:info} and \ref{sec:seels:implementation}))}. 
This split yields a tractable sequence of trajectory optimization problems, although it induces sub-optimality relative to the computationally intractable problem of simultaneous exploration and exploitation \cite{bar1974dual}. 
Since \textit{safely} learning nonlinear dynamics with minimal assumptions remains an open area of research (see Section \ref{sec:related_work}), quantifying this sub-optimality gap (perhaps by computing the regret of the algorithm) is beyond the scope of this work.

\subsection{Information objective}\label{sec:seels:info}

During the exploration phase, we perform trajectory optimization with an objective function $\costinfo$ that encourages visiting states and taking actions that reduce uncertainty over the unknown dynamics. To do so, a natural objective to maximize is the mutual information between the unknown dynamics $\g$ and the observations of the state. 
This cost characterizes the \textit{information gain} \cite{MacKay1992,Srinivas2010,Chowdhury2017} about $\g$, gained from applying the control inputs $\ac$ to the true system from a state $\x$ and observing the next state corrupted by random disturbances.

To derive this objective, we leverage the Bayesian interpretation of \alpaca which is used during offline meta-training to obtain a model that has the correct structure (Assumption \ref{assum:capacity_model}) and is calibrated (Assumption \ref{assum:calib}).  
Specifically, \alpaca is trained assuming zero-mean Gaussian-distributed observation noise $\rev{\epsilon_{i,t}}$ of variance $\smash{\sigma_i^2}$ and a Gaussian prior over model parameters $\smash{\k_i}$. %
Using \eqref{eq:bayesian_updates} and $t$ transitions from the system, this yields the posterior distribution  $\smash{\k_i\sim\N(\kbar_i, \sigma_i^2\bLambda_{i,t}^{-1})}$. 
In this setting, the marginal distribution over observations $\smash{\tilde{x}^{+}_i} =  \smash{\rev{x_{i,t{+}1}}-h_i(\x_t,\ac_t)=\k_i^\top\feat_i(\x_t,\ac_t)+\rev{\epsilon_{i,t}}}$ given an arbitrary state $\x_t$ and control input $\ac_t$ is a normal distribution $\N(\kbar_i \feat_i, (1+\feat_i^\top \bLambda^{-1}_{i,t} \feat_i)\sigma_i^2)$, where $\feat_i\,{=}\,\feat_i(\x_t,\ac_t)$.

For this formulation, the mutual information $\mathcal{I}$ between the observation $\smash{\tilde{\x}^{+}\,{=}\, \x_{t{+}1} {-} \h(\x_{t},\ac_{t})}$ and the true unknown dynamics $\g$ can be derived in closed-form. 
It is defined using the entropy $\mathcal{H}(\cdot )$, which for a random variable $\smash{\x^{+}}\sim\mathcal{N}(\bmu , \bSigma )$ evaluates to $\mathcal{H}(\x) = (1/2)\text{log}(\text{det}(2\pi e \bSigma) )$. Hence, the information gain from observing the scalar random variable $\smash{x^{+}_i}$ can be expressed as 
$\mathcal{I}(x^{+}_i ; g_i) = \mathcal{H}(x^{+}_i) - \mathcal{H}(x^{+}_i | g_i)
= \frac{1}{2}( 
\log(\text{var}(x^{+}_i)) -\log(\text{var}(x^{+}_i | g_i)
)
= \frac{1}{2}( 
\log((1+\feat_i^\top\bLambda_{i,t}^{-1}\feat_i)\sigma_i^2)) -\log(\sigma_i^2))
)
= 
\frac{1}{2}( 
\log(1+\feat_i^\top\bLambda_{i,t}^{-1}\feat_i))$.  
This quantity expresses the information gain from observing each dimension $i$ of the state, which are modeled independently in our formulation. By summing $\mathcal{I}(x^{+}_i ; g_i)$ over $i=1,\mydots,n$, we arrive at the exploration objective $\costinfo$ in \eqref{eq:info_cost}.

The objective $\costinfo$ is a function of the current information state of the model, specified by the updated precision matrices $\smash{\bLambda_{1,t}}, \mydots, \smash{\bLambda_{\xdim,t}}$. It explicitly encourages taking actions and visiting states in the feature space spanned by $\feat$ which have highest variance $\smash{\bLambda_{i,t}^{-1}}$. The resulting observations are then the most informative in terms of reducing uncertainty over $\g$.

Note that the expected information gain along a trajectory is not simply the sum of the expected information gains per transition, as expressed in \ExploreOCP when summing \eqref{eq:info_cost} over $k=0,\mydots,\rev{N-1}$. 
However, correctly computing the expected information gain along the trajectory would require factoring in model updates along the trajectory; we find that considering the sum of single-transition information gain with the current precision matrices $\bLambda_{i,t}$ is sufficient in guiding exploration. 
Also, while in this work we assume bounded (non-Gaussian) noise corrupting our measurements, %
we find that using this information objective %
works well in practice to encourage exploration. 
The problem of optimal exploration is beyond the scope of this framework.

Finally, we note that this cost function $\costinfo$ does not suffer from computational complexity scaling with the amount of data, as is the case for similar objectives derived for squared-exponential kernel GPs \cite{koller2018,Srinivas2010}. This computational efficiency comes from choosing finite-dimensional basis functions to parameterize the model class for the true dynamics $\g$.

\subsection{Probabilistic safety and feasibility guarantees}

Next, we prove that 
\seels enjoys probabilistically guaranteed safety and ensures that \ExploreOCP remains feasible at all times with high probability, which is crucial to enable autonomous operation. Our next result also states that provided that \ReachOCP is eventually feasible, the agent will safely reach $\Xgoal$ with high probability. This result characterizes our proposed approach as a feasible solution to \ccocp.

\begin{theorem}[Probabilistic recursive feasibility and safety]\label{thm:probabilistic_feasibility_safety}
With the confidence sets defined in \eqref{eq:conf_set_kappa} and the reachable sets satisfying \eqref{eq:confidence_tubes}, \rev{apply Algorithm \ref{alg:explo_reach_algalg}} %
to sequentially solve \ExploreOCP and \ReachOCP and to control the uncertain system \eqref{eq:full_problem_dynamics}. 
Then, under Assumptions \ref{assum:invariant_X0}\rev{-}\ref{assum:calib}\rev{:}
\begin{itemize}[leftmargin=3.5mm]
  \setlength\itemsep{0.5mm}
\item \rev{With probability at least $(1-\delta)$, t}here exists \rev{an} optimization horizon $\rev{N_{\textrm{e}}}$ ensuring the feasibility of \rev{$\ExploreOCP\rev{(N_{\textrm{e}})}$ in} each %
exploration \rev{phase in line $9$ of Algorithm \ref{alg:explo_reach_algalg}}.
\item \rev{With probability at least $(1-\delta)$, during exploration, the state trajectory satisfies all constraints $\x_t\in\Xsafe$ at all times $t$ and finishes in $\X_0$ at the end of each phase.}
\item  \rev{With probability at least $(1-\delta)$, a}ssuming that \ReachOCP is eventually feasible,  
the system is guaranteed to satisfy \eqref{eq:joint_chance_constraint}, i.e., to be safe at all times and eventually reach $\Xgoal$. 
\end{itemize}

\end{theorem}

\begin{proof} 
First, we prove that \seels is recursively feasible with high probability, before proving that the system satisfies \rev{all constraints during exploration and that  \eqref{eq:joint_chance_constraint} is satisfied if \ReachOCP is eventually feasible.}

\vspace{1mm}

\textit{Probabilistic recursive feasibility}: Let $\ninfo$ be the number of exploration phases before \ReachOCP becomes feasible\footnote{\rev{$\ninfo$ is a random variable that can be defined as a stopping time with respect to the natural filtration associated to the stochastic process $\x_t$}. This result also holds if \ccocp is not feasible and the algorithm can never solve $\ReachOCP$  (e.g., if $\Xgoal$ is surrounded by obstacles). Indeed, if the algorithm is stuck in an infinite number of exploration steps, the last inequality of this proof still holds for $\ninfo\rightarrow\infty$, by Theorem \ref{thm:conf-sets}.}. %
Also, let $\Ninfo^j$, and $t_j =\sum_{l=1}^{j-1} \Ninfo^l$ be, respectively, the planning horizon and the start time index of each $\ExploreOCP_j$, \rev{where $j=1,\mydots,\ninfo$}.  
For conciseness, define $\EOCP_j$ for $\{\ExploreOCP_{j} \ \textrm{is feasible}\} $, corresponding to the event that the $j$-th exploration problem is feasible. 
Then, %
\begin{align*}
\Prob\bigg(
\bigwedge_{j=\rev{1}}^{\ninfo}
\EOCP_j
\bigg)
&\geq
\Prob\bigg(
\bigwedge_{j=\rev{1}}^{\ninfo}
\EOCP_j, \ 
\x_{t_{\ninfo}} \in \X_0, \ 
\bigg)
\\[-3mm]
&
\hspace{-1.3cm}
=
\Prob\bigg(
\EOCP_{\ninfo} \ | \
\bigwedge_{j=\rev{1}}^{\ninfo{-}1}
\EOCP_j, \
\x_{t_{\ninfo}} {\in}\, \X_0
\bigg)
\\[-3mm]
&\cdot
\Prob\bigg(
\bigwedge_{j=\rev{1}}^{\ninfo{-}1}
\EOCP_j, \
\x_{t_{\ninfo}} {\in}\, \X_0
\bigg)
.
\end{align*}
By Assumption \ref{assum:invariant_X0}, given that $\x_{t_j}\in\X_0$, $\ExploreOCP_j$ is feasible for any $j$-th exploration phase.  
Indeed, choose $\smash{\ninfo^j=1}$ for $\ExploreOCP_j$. 
Then, $\smash{\ac_0^j=\boldsymbol\pi(\x_{t_j})}$ is a feasible solution to $\ExploreOCP_j$. 
Thus, the event $\{\EOCP_j \ | \
\x_{t_j} \in \X_0\}$ holds with probability one.  
In particular, this implies that the event  $$\{\EOCP_{\ninfo} \ | \
\bigwedge_{j=\rev{1}}^{\ninfo{-}1}
\EOCP_j,
\x_{t_{\ninfo}} {\in}\, \X_0\}$$ holds with probability one.  

Next, we %
leverage our confidence sets:
\begin{align*}
\Prob\bigg(
\bigwedge_{j=\rev{1}}^{\ninfo}
\EOCP_j
\bigg)
&\geq
\Prob\bigg(
\bigwedge_{j=\rev{1}}^{\ninfo{-}1}
\EOCP_j, \
\x_{t_{\ninfo}} {\in}\, \X_0
\bigg)
\\[-2mm]
&
\hspace{-2cm}\geq
\Prob\bigg(
\bigwedge_{j=\rev{1}}^{\ninfo{-}1}
\EOCP_j, \
\x_{t_{\ninfo}} {\in}\, \X_0, \ 
\kstari\in\confset_{i,t_{\ninfo{-}1}}^{\delta} \ \forall i
\bigg)
\\[-2mm]
&
\hspace{-2cm}
=
\Prob\bigg(
\x_{t_{\ninfo}} \in \X_0 \ | \
\bigwedge_j \EOCP_j, \
\kstari\in\confset_{i,t_{\ninfo{-}1}}^{\delta} \forall i
\bigg)\,
\\
&\cdot
\Prob\bigg(
\bigwedge_j \EOCP_j, \
\kstari\in\confset_{i,t_{\ninfo{-}1}}^{\delta} \forall i
\bigg).
\end{align*}
By construction of the reachable sets $\{\X_k^{t_j,\delta}\}_{k{=}1}^{\Ninfo^{j}}$, by Assumption \ref{assum:capacity_model} and by definition of $\ExploreOCP_j$ (since $\X_{\Ninfo^j}^{t_{j},\delta}\subseteq\X_0$), we have that $\x_{t_{j+1}} \in \X_0 $ given that $\ExploreOCP_j$ is feasible and  $\kstari\in\confset_{i,t_j}^{\delta} \ \forall i$, for any $j$-th exploration problem. 

Thus, the first term $\{\x_{t_{\ninfo}} \in \X_0 \ | \
\bigwedge_j \EOCP_j, \
\kstari\in\confset_{i,t_{\ninfo{-}1}}^{\delta} \forall i
\}$ holds with probability one, and 
\begin{align*}
\Prob\bigg(
\bigwedge_{j=\rev{1}}^{\ninfo}
\EOCP_j
\bigg)
&\geq
\Prob\big(
\bigwedge_{j=\rev{1}}^{\ninfo{-1}} \EOCP_j, \
\kstari\in\confset_{i,t_{\ninfo{-}1}}^{\delta} \ \forall i
\big).
\end{align*}
Since $\EOCP_0$ is feasible with probability one as $\x_0\in\X_0$, 
and by reasoning by induction for all $j=\ninfo,\mydots,0$, we obtain that 
\begin{align*}
\Prob\bigg(
\bigwedge_{j=\rev{1}}^{\ninfo}
\EOCP_j
\bigg)
&\geq
\Prob\big(
\bigwedge_{j=\rev{1}}^{\ninfo{-}1} 
\kstari\in\confset_{i,t_j}^{\delta} \ \forall i
\big)
\geq
(1-\delta),
\end{align*}
where the last inequality comes from Theorem \ref{thm:conf-sets}.

\vspace{1mm}
\textit{Probabilistic safety}: \rev{Before proving the second statement of Theorem \ref{thm:probabilistic_feasibility_safety}, we prove the third statement by assuming that $\ninfo<\infty$, which holds if $\ReachOCP$ is eventually feasible}. Let also $N_{\textrm{reach}}$, and $t_{f}$ be, respectively, the planning horizon and the start time index of $\ReachOCP$.  
For conciseness, define $\smash{\x_k^{t_j}=\x_{t_j+k}}$, corresponding to the state at time $(t_j{+}k)$ in the $j$-th phase. 
Further, define the event that the trajectory during the $j$-th exploration phase (or exploitation phase) satisfies all constraints as
\vspace{-3mm}
\begin{align*}
\{\x_{\textrm{info}}^j \,{\in}\, \X_{\textrm{info}}^j\}
&\,{=}\,
\bigg\{
\bigwedge_{k=1}^{\Ninfo^j}\big(\x_{k}^{t_j} \,{\in}\, \Xsafe\big) 
\wedge \big(\x_{\Ninfo^j}^{t_j} \,{\in}\, \X_0\big)
\bigg\}
,
\\ 
\{\x_{\textrm{reach}} \,{\in}\, \X_{\textrm{reach}}\}
&\,{=}\,
\bigg\{
\bigwedge_{k=1}^{N_{\textrm{reach}}}\big(\x_{k}^{t_f} \,{\in}\, \Xsafe\big) 
\wedge \big(\x_{N_{\textrm{reach}}}^{t_f} \,{\in}\, \Xgoal\big)
\bigg\}
,
\end{align*}
where $j\,{=}\,1,\mydots,\ninfo$. 
With this notation, we rewrite the probabilistic safety constraint of the original problem as
\begin{align}
\eqref{eq:joint_chance_constraint} 
&=
\scalebox{0.9}{$
\Prob\bigg(
\overset{\ninfo}{\underset{j=1}\bigwedge}
\{\x_{\textrm{info}}^j \,{\in}\, \X_{\textrm{info}}^j\}
\wedge
\{\x_{\textrm{reach}} \,{\in}\, \X_{\textrm{reach}}\}
\bigg)
$}
\triangleq
\Prob\big(
\{\textrm{Success}\}
\big)
\nonumber
\\
&\hspace{-3mm}\geq 
\Prob\Big(
\{\textrm{Success}\}
\, |\, 
\kstari\,{\in}\,\confset_{i,t}^{\delta} \, \forall t \, \forall i\Big)
\Prob\Big(\kstari\,{\in}\,\confset_{i,t}^{\delta} \, \forall t \, \forall i\Big)
\label{eq:safety_upper_prob_confsets}
,
\end{align}
where $t=t_1,\mydots,t_{\ninfo},t_f$, and $i=1,\mydots,n$.\footnote{The inequality follows from $\eqref{eq:joint_chance_constraint} {=} \scalebox{0.8}{$
\Prob\Big(
\{\textrm{Success}\}
\, |\, 
\kstari\,{\in}\,\confset_{i,t}^{\delta} \, \forall t\forall i\Big)
\Prob\Big(\kstari\,{\in}\,\confset_{i,t}^{\delta} \, \forall t\forall i\Big)
$}$
\scalebox{0.8}{$+ \,
\Prob\Big(
\{\textrm{Success}\}
\, |\, 
\kstari\notin\confset_{i,t}^{\delta} \, \forall t, \forall i\Big)
\Prob\Big(\kstari\notin\confset_{i,t}^{\delta} \, \forall t, \forall i\Big)
$}
${\geq}\, \eqref{eq:safety_upper_prob_confsets}$.}
Again, by Assumption \ref{assum:capacity_model}, the meta-learning model can fit the true dynamics. Hence, if the true parameters are within the confidence sets $\confset_{i,t}^\delta$, then, the reachable sets $\X_{k}^{t,\delta}$ defined in \eqref{eq:confidence_tubes} necessarily contain the state trajectory of the true system. Hence,
\begin{equation}
\Big\{\X_{k}^{t,\delta} \subset \Xsafe\Big\}
=
\Big\{\x_k(\k^*)\in\Xsafe
\, |\, 
\kstari\in\confset_{i,t}^{\delta}, \ \forall i\Big\}
.
\end{equation}
By definition of \ExploreOCP and \ReachOCP, the reachable sets are subsets of the safe set (with probability one), and %
\begin{align*}
\Prob\Big(\x_k^t(\k^*)\,{\in}\,\Xsafe, \ 
	&k{=}1,{\mydots},N \, |\, \kstari\in\confset_{i,t}^{\delta}, \ i{=}1,{\mydots},\xdim\Big)
\\
&\hspace{-1cm}=
\Prob\Big(\X_{k}^{t,\delta} \,{\subset}\, \Xsafe,  k{=}1,{\mydots},N\Big)
= 1,
\end{align*}
which also holds for the final constraints $\x_N^t\in\X_0$ and $\x_N^{t_f}\in\Xgoal$. 
Thus, 
 $   \Prob\big(
\{\textrm{Success}\}
\, |\, 
\kstari\in\confset_{i,t}^{\delta} \ \forall t \, \forall i\big) = 1. 
	$
Combining this result with \eqref{eq:safety_upper_prob_confsets}, we obtain that the system satisfies all constraints and eventually reaches $\Xgoal$ with probability at least that of the probability of the model parameters belonging to the confidence sets, i.e., 
$
\eqref{eq:joint_chance_constraint} %
\geq
\Prob\big(\kstari\in\confset_{i,t}^{\delta} \ \forall t\, \ \forall i\big).
$
This last term holds with probability greater than  $(1{-}\delta)$. Indeed, using (a) \rev{a union bound (Boole's inequality)} and (b) Theorem \ref{thm:conf-sets}, we obtain
\begin{align*}
&\eqref{eq:joint_chance_constraint}\,{\geq}\,
\Prob\Big(\kstari\,{\in}\,\confset_{i,t}^{\delta} \, \forall t \forall i\Big) 
\,{=}\,
1 -
\Prob\bigg(\hspace{-1pt}
\bigvee_{i=1}^{n}\bigvee_{t}
\kstari{\notin}\confset_{i,t}^{\delta}
\bigg) 
\\[-1mm]
&\hspace{2mm}\mathop{\geq}^{(a)}
1 \,{-}\,
\sum_{i=1}^{n}\Prob\Big(\bigvee_{t}
\kstari\,{\notin}\,\confset_{i,t}^{\delta}
\Big) 
\,{=}\, 
1 {-}
\sum_{i=1}^{n}\Big(
1 {-} 
\Prob\Big(\bigwedge_{t}
\kstari\,{\in}\,\confset_{i,t}^{\delta}
\Big) 
\Big)
\\[-1mm]
&\hspace{2mm}\mathop{\geq}^{(b)}
1 -
\sum_{i=1}^{n}\big(
1 - 
(1-2\delta_i)
\big)
=
1 -
\sum_{i=1}^{n}\big(
2\delta_i
\big)
=
(1-\delta)
\end{align*}
since $\delta_i\,{=}\,\delta/(2n)$. This concludes the \rev{third statement}  of Theorem \ref{thm:probabilistic_feasibility_safety}. 
\rev{The second statement of Theorem \ref{thm:probabilistic_feasibility_safety} follows by relaxing the assumption that $\ninfo<\infty$ and by redefining the event $\{\textrm{Success}\}
=
\bigwedge_{j=1}^{\ninfo}
\{\x_{\textrm{info}}^j \,{\in}\, \X_{\textrm{info}}^j\}
$. In this case, all inequalities follow similarly from Theorem \ref{thm:conf-sets}.}
\end{proof}

Key to this proof are Theorem \ref{thm:conf-sets} and the definition of the reachable sets in \eqref{eq:confidence_tubes}. 
We stress that our safety probability $(1\,{-}\,\delta)$ is independent of the (unknown) time to reach $\Xgoal$, which would not be the case if  pointwise chance constraints were used instead of \eqref{eq:joint_chance_constraint}.

\rev{To solve \ccocp by satisfying all constraints and reaching $\Xgoal$ with high probability}, Theorem \ref{thm:probabilistic_feasibility_safety} relies on \ReachOCP eventually becoming feasible. If the original problem is feasible with perfect knowledge of the dynamics, this assumption holds if the objective used for exploration leads to actions that continually reduce uncertainty in dynamics, see \cite{Mania2020}. Related conditions on observability and persistence of excitation \rev{are available in the literature, see} \cite{Berberich2020Robust, Coulson2018DataEnabledPC}. \rev{W}ith full state information and long exploration horizons using our information cost \rev{with Algorithm \ref{alg:explo_reach_algalg}, we observed in our experiments that \ReachOCP eventually becomes feasible if the original problem is truly feasible with perfect knowledge of the dynamics, see Sections \ref{sec:results:sim} and \ref{sec:hardware}}.
In situations where the problem is not feasible in the first place (e.g., an obstacle blocks the only path to $\Xgoal$), our algorithm defaults to exploring within a feasible neighborhood around $\X_0$, in which case the system is guaranteed to satisfy all safety constraints at all times with probability greater than $(1-\delta)$.

By assuming bounded disturbances $\rev{\ep_t}$ and exploiting confidence sets over the model parameters which hold jointly for all times with high probability (Theorem \ref{thm:conf-sets}), we can guarantee the feasibility of \ExploreOCP during exploration. This contrasts with related work in the MPC literature which provides probabilistic recursive feasibility over a finite horizon only  \cite{Onoresolvability2012}. 
This aspect is crucial to enable autonomous operation and reliability of the approach.

\subsection{Implementation and practical considerations}\label{sec:seels:implementation}
Implementation of the algorithm is complicated by challenges in reachability analysis and non-convex optimization. 

\textbf{Reachability Analysis}: Computing the reachable sets in \eqref{eq:confidence_tubes} is difficult due to the non-convexity of the features $\feat$. Methods reasoning about single-step set propagation (e.g., \cite{koller2018}) are generally too conservative \cite{randSets} as they neglect time correlations induced by the parameters $\k_i$. Moreover, the online updates to the model parameters preclude exact methods which perform computations offline to compute reachable sets, e.g., \cite{HJIoverviewBansal2017,fan2020deep}. 
Neural network verification techniques \cite{IvanovVerisig2019} are tailored to offline computations and are not yet fast enough for online trajectory optimization. 
Finally, methods using Lipschitz continuity can provide conservative reachable set approximations \cite{koller2018}, but such techniques would be so conservative for the systems we consider in this work that the agent would never deem reaching the goal to be feasible, even if one had access to the true Lipschitz constant of the system. The main limitation of \cite{koller2018} is that it adopts a single-step worst case analysis and does not account for the time correlations of the parameters \cite{randSets}.  

For these reasons, we leverage a recently-derived sampling-based uncertainty propagation scheme for reachability analysis (\randup) \cite{randSets}. By sampling \rev{$M$} parameter \rev{ and disturbance tuples} $\rev{(\k_i^j,\epsilon_{i,k}^j, \scalebox{0.85}{$1\leq i\leq\xdim, \, 0\leq k\leq N-1$}})$ within their \rev{associated bounded sets $\confset_{i,t}^\delta$ and $\Epsilon_i$}, %
computing \rev{the resulting} reachable states \rev{$\x_k^j$ for these tuples according to $x_{i,k+1}^j=h_i(\x_k^j,\ac_k^j)+\k_i^{j\top}\feat_i(\x_k^j,\ac_k^j)+\epsilon_{i,k}^j$}, and approximating the reachable sets \rev{$\X_k^{t,\delta}$} in \eqref{eq:confidence_tubes} \rev{with the convex hull of the samples $\{\x_k^j\}_{j=1}^M$,} \randup provides a scalable approach to compute these tubes with minimal assumptions on the system's dynamics, enabling the use of arbitrary neural network features $\feat$. 
It is fast enough to be used with common robotic systems (computing a trajectory requires less than a second for a 13-dimensional spacecraft system with a Python implementation \cite{randSets}) and can be further accelerated through parallelization on GPUs. 
Although \randup lacks finite-sample guarantees of safety (see also \cite{LewJansonEtAl2021} for a recent extension with finite-sample guarantees), 
asymptotic guarantees can be derived using random set theory \cite{randSets}, and finite-sample approximations are sufficient to ensure safety in practice, as shown in our results. 

\textbf{Optimization-based planning}:
Using \randup, we reframe a generally intractable stochastic optimal control problem into \ExploreOCP and \ReachOCP, which are non-convex optimal control problems. Efficiently computing solutions to this class of problems is an active field of research. In this work, we use a direct method based on sequential convex programming (SCP) \cite{LewBonalli2020}. 
By solving a sequence of convexified versions of the original problem, SCP-based methods can run in real time and provide theoretical guarantees of local optimality \cite{MaoSzmukEtAl2016,BonalliCauligiEtAl2019}. %
In this work, we initialize each method with an infeasible straight-line trajectory, solve each convexified problem using \textrm{OSQP} \cite{StellatoBanjacEtAl2017}, and provide further details in Appendix \ref{sec:obs_avoidance}. 

Additionally, due to uncertainty, the feasibility of each problem depends on the optimization horizon $N$. 
For this reason, Algorithm \ref{alg:explo_reach_algalg} performs a search over a predefined range of planning horizons (lines $2$ and $\rev{8}$). 
For exploitation, it selects the first feasible solution if one exists, although other criteria could be used, e.g., minimal control cost. 
For exploration, we select the trajectory which leads to the largest expected information gain. 
Indeed, due to tight control constraints and safety constraints, a longer horizon does not necessarily lead to higher information gain. 
This heuristic works well in practice, and future work will adopt a continuous-time problem formulation with free final time, which is an active area of research  \cite{BonalliCauligiEtAl2019}.

\section{Simulation studies}\label{sec:results:sim}
\subsection{Robotic manipulation in a cluttered environment}

\begin{figure}[!b]
\centering
\includegraphics[width=0.7\linewidth]%
{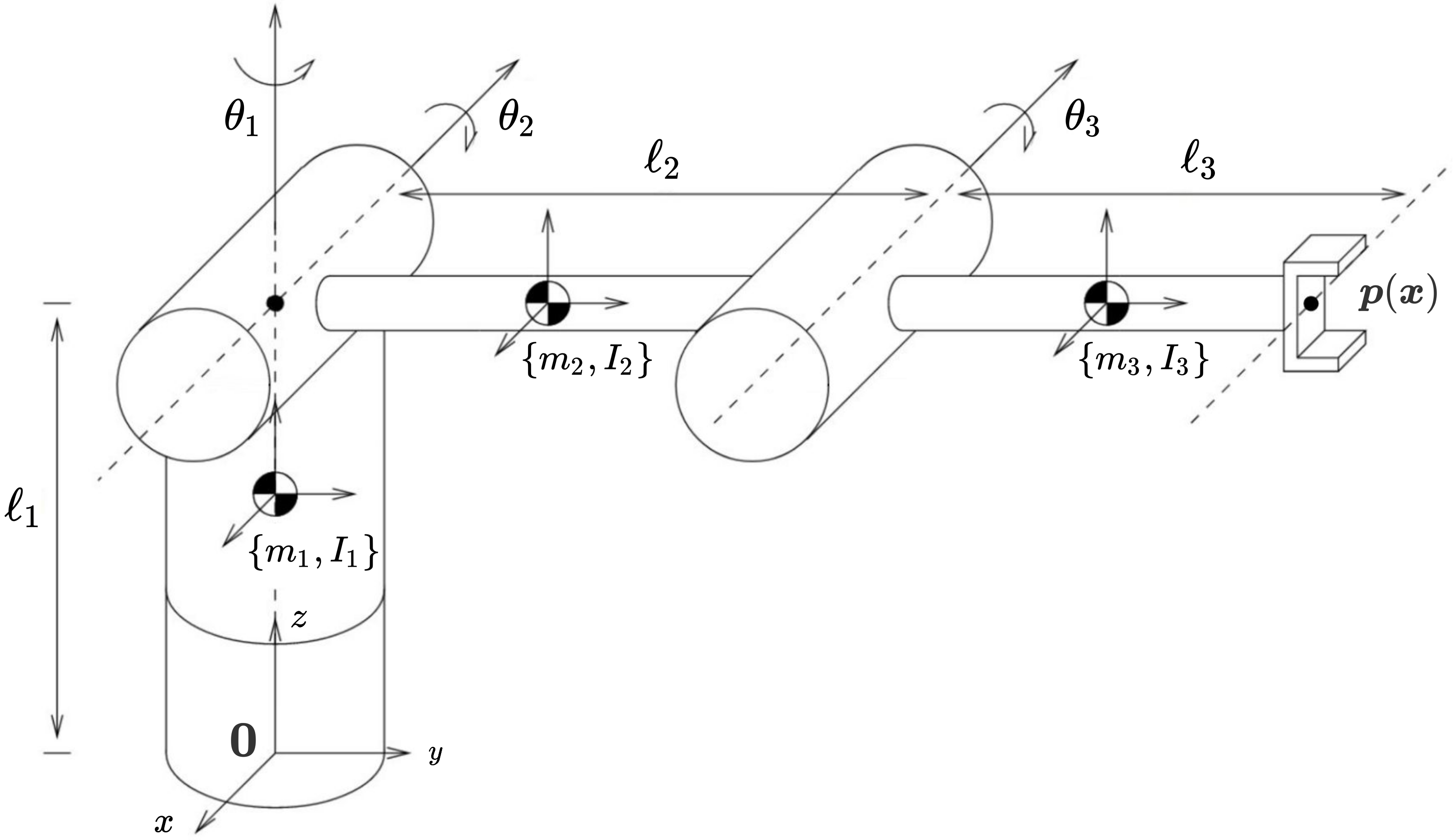}
\centering
\caption{Three-link open-chain manipulator, figure adapted from \cite{MurrayManipulation1994}.
}
\label{fig:manipulator}
\end{figure}

First, consider the three-link open-chain torque-controlled manipulator shown in Figure \ref{fig:manipulator} \cite{MurrayManipulation1994}. The state of the system is $\x\,{=}\,(q_1,q_2,q_3,$ $\dot{q}_1,\dot{q}_2,\dot{q}_3)$ and its control is $\ac=(\tau_1,\tau_2,\tau_3)$, where $q_i$ is the angle of the $i$-th joint and $\tau_i$ is its torque. Its true dynamics follow
\begin{equation}\label{eq:ct_dynamics_manip}
M({\bq})\ddot{\bq}+C({\bq},\dot{\bq})\dot{\bq}+N({\bq})=\ac,
\end{equation}
where \scalebox{0.95}{$\bq = (q_1,q_2,q_3)$}. 
The expressions for
$M({\bq}),C({\bq},\dot{\bq})$, and $N({\bq})$ depend on the lengths \scalebox{0.95}{$
l_1 \,{=}\, 1\,\textrm{m}$} and \scalebox{0.95}{ 
$l_2\,{=}\,
l_3\,{=}\,0.5\,\textrm{m}$}, masses $m_i$, and inertias $I_{i}$ of each $i$-th link of the manipulator (see \cite{MurrayManipulation1994}). By introducing uncertainty in each $m_i$ and $I_i$, this experiment corresponds to a scenario where the robot must manipulate uncertain payloads and identify its inertial properties. 
Although manipulators are typically controlled using trajectory generation and high-frequency low-level tracking (rather than open-loop control), 
the instability of this system makes it a useful case study for our framework.

Starting at  
$\x(0)\,{=}\,\mathbf{0}\,{\in}\,\X_0
\,{=}\,\scalebox{0.9}{$\{\x \, | \,
|{q}_1|{\leq}\frac{\pi}{8},
\, |{q}_{2,3}|{\leq} \frac{\pi}{6}, 
\, |\dot{q}_{1,2,3}|{\leq} \frac{\pi}{6}
\}$}$, 
we consider the problem of 
moving the arm to the goal region 
$\Xgoal\,{=}\,\scalebox{0.9}{$\{\x \, | \,
|{q}_1 {-}\frac{\pi}{2}|{\leq}$}$ 
$\scalebox{0.9}{$\frac{\pi}{8},
\, |{q}_{2,3}|{\leq} \frac{\pi}{10}, 
\, |\dot{q}_{1,2,3}|{\leq} \frac{\pi}{8}
\}$}$. 
We consider two spherical obstacles 
$\Obs_i\,{\subset}\,\R^3$ which should be avoided by the end-effector. 
This results in two obstacle avoidance constraints $\x\not\in\Obs_i$ written as 
$\Obs_i\,{=}\,\scalebox{0.95}{$\{\x %
\, | \,  \|\bp_{\textrm{e}}(\x){-}\bp_i\|_2\,{\leq}$}$ \scalebox{0.95}{$\frac{1}{4}\}$}, 
 where 
 \scalebox{0.9}{$\bp_1\,{=}\,(\frac{\sqrt{2}}{2},\frac{-\sqrt{2}}{2},1.4)$},  
\scalebox{0.9}{$\bp_2\,{=}\,(\frac{\sqrt{2}}{2},\frac{-\sqrt{2}}{2},0.6)$}, and
$\bp_{\textrm{e}}:\R^6\,{\rightarrow}\,\R^3$ maps the state $\x$ to the end-effector position 
 \cite{MurrayManipulation1994}. 
We also account for joint angle and velocity limits  
$\x\,{\in}\,\X$ 
so that 
\scalebox{0.95}{$\Xsafe = \X{\setminus}({\cup_i}\Obs_i)$}, with 
$\X = \scalebox{0.9}{$\{\x \, | \,
q_1{\in}[\frac{-\pi}{2},\pi], |q_{2,3}|{\leq}\frac{\pi}{2},
\, 
|\dot{q}_i|\,{\leq}\, \frac{\pi}{8}
\}$}$. 
Control bounds are set to
$\mathcal{U} = \scalebox{0.95}{$\{\ac\,|\,|\tau_i|\,{\leq}\, 2 \,\textrm{N}{\cdot}\textrm{m}\}$}$.

\begin{figure}[!t]
\begin{minipage}{0.47\linewidth}
\centering
\includegraphics[width=1\linewidth,trim=52 35 35 50, clip]{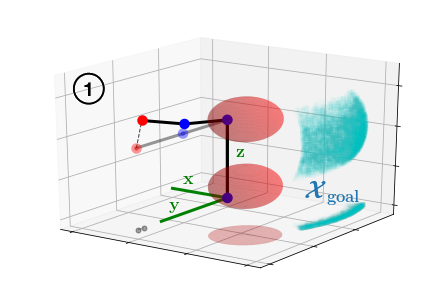}

\vspace{1mm}

\includegraphics[width=1\linewidth,trim=52 35 35 50, clip]{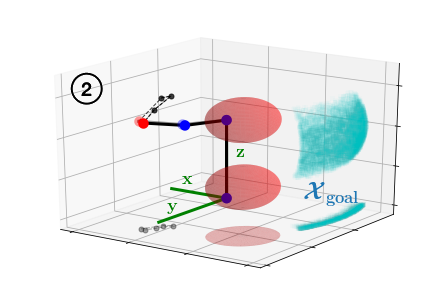}

\vspace{1mm}

\includegraphics[width=1\linewidth,trim=52 35 35 50, clip]{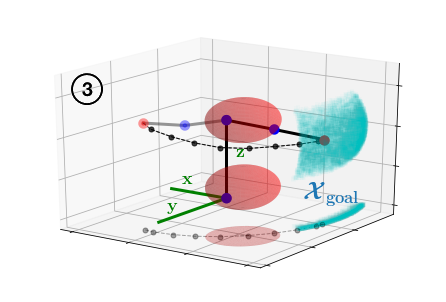}
\vspace{-2mm}
\end{minipage}
\hspace{0.02\linewidth}
\begin{minipage}{0.47\linewidth}
\centering
\includegraphics[width=1\linewidth,trim=52 35 35 50, clip]{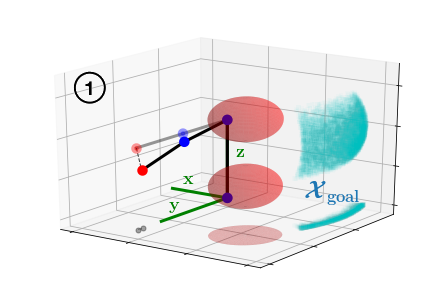}

\vspace{1mm}

\includegraphics[width=1\linewidth,trim=52 35 35 50, clip]{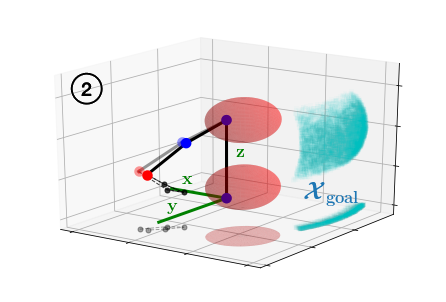}

\vspace{1mm}

\includegraphics[width=1\linewidth,trim=52 35 35 50, clip]{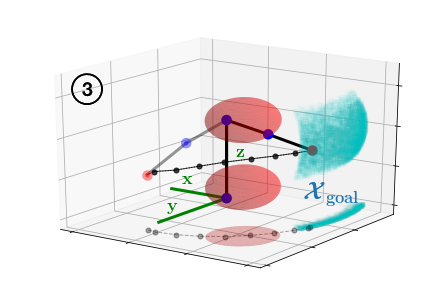}
\vspace{-2mm}
\end{minipage}
\caption{Simulated manipulator experiments for a light (left) and heavy (right) arms: the robot is able to infer its dynamics through active exploration (1,2), and eventually reach $\Xgoal$ (3), while satisfying constraints at all times.  A black dot represents the position of the end-effector every $1 \textrm{s}$. Projections onto the $x\,{-}\,y$ plane of the end-effector trajectory, obstacles, and goal region are also shown.}
\label{fig:manip:examples_lowhigh_mass}
\end{figure}

\begin{figure}[!t]
\centering
\includegraphics[width=0.62\linewidth,trim=0 0 0 0, clip]{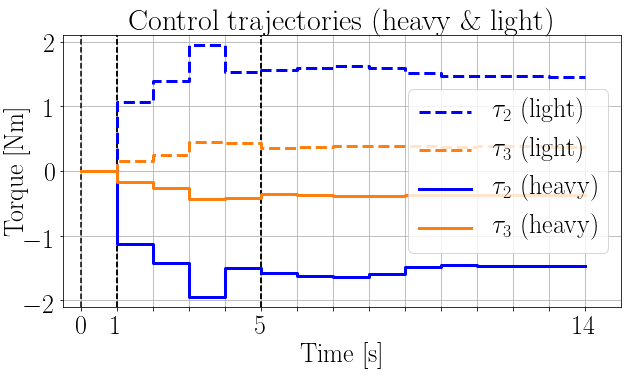}
\includegraphics[width=0.36\linewidth,trim=0 0 0 0, clip]{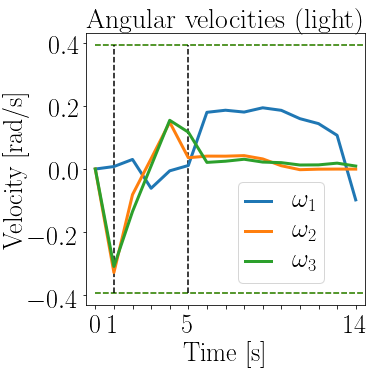}
 \vspace{-4mm}
\caption{Simulated manipulator experiments: control trajectories for both light and heavy manipulators (left) and angular velocities (right, light arm).
}
\label{fig:manip:trajs}
\end{figure}

We simulate \eqref{eq:ct_dynamics_manip} with an Euler discretization scheme with timestep $50\,\textrm{ms}$. 
To slow down the unstable dynamics of this system, we use a damping controller with approximate gravity compensation $\ac\,{=}\,\hat{N}({\bq},\dot{\bq})\,{-}\,K_d\dot{\bq}\,{+}\,\ac_{\textrm{seels}}$, where $\hat{N}$ is computed using nominal mass and inertia values,  $K_d\,{=}\,[0,4,1]$, and $\ac_{\textrm{seels}}$ is the open-loop control trajectory computed with our framework, where each input lasts $1\,\textrm{s}$ and $\ac_{\textrm{seels}}\,{\in}\,\mathcal{U}$.
We use a nominal model $\h$ of \eqref{eq:ct_dynamics_manip} with \scalebox{0.95}{$C({q},\dot{q})\,{=}\,\mathbf{0}$} and  \scalebox{0.95}{$(m_i,I_{xi}\,{=}\,I_{yi},I_{zi})\,{=}\,(2,0.05,0.01)$}. 
We then train an \alpaca model with $L=2$ layers of width $128$ each and an output dimension $d=64$ for each $\feat_i$. 
We randomize parameters according to uniform distributions with
$m_i\,{\in}\,[1.75,2.25]\,\textrm{kg}$, 
$I_{xi}\,{=}\,I_{yi}\in[0.04,0.06]\,\textrm{kg}{\cdot}\textrm{m}^2$,  
$I_{zi}\,{\in}\,[0.008,0.012]\,\textrm{kg}{\cdot}\textrm{m}^2$, with \scalebox{0.9}{$i\,{=}\,1,2,3$}. We assume 
$\sigma_i\,{=}\,10^{-3}$, and 
$|\rev{\epsilon_{i,t}}|\leq \smash{(\sigma_i^2\chi_1^2(0.95))^{1/2}}$, \scalebox{0.9}{$i\,{=}\,1,{\mydots},6$}. 
We use $\beta_i$ and orthogonality regularization,  
$M\,{=}\,2500$ samples for reachability analysis 
and directly use our theoretically computed bounds within our safe learning algorithm, i.e., we use \eqref{eq:beta_bound} to sample model parameters. 
Further details about the trajectory optimization are provided in Appendix \ref{sec:obs_avoidance}.

Figure \ref{fig:manip:examples_lowhigh_mass} (left) shows the results for low mass and inertia \scalebox{0.9}{$(m_i,I_{xi} = I_{yi},I_{zi})\,{=}\,(1.85,45{\cdot}10^{-3},8.5{\cdot}10^{-3})$}. Note that lower mass and inertia would cause the robot to leave $\X_0$ in $1\textrm{s}$, making the problem infeasible. In this experiment, the end-effector initially rises (due to the lower-than-expected mass/inertia), does a second exploration maneuver, then reaches the goal, always satisfying constraints. 
Figure \ref{fig:manip:examples_lowhigh_mass} (right) shows results for a heavier arm. %
Again, all constraints are satisfied and the goal is reached safely.  
For both scenarios, we compare to a method which only considers uncertainty from the additive disturbances $\rev{\ep_t}$, thereby deciding that directly reaching $\Xgoal$ is safe. This naive approach fails to reach the goal and violates constraints, which demonstrates the need for an accurate model
and sequential online learning to reliably solve this problem.

We plot control trajectories and angular velocities in 
Figure \ref{fig:manip:trajs}. This shows that steering the system with lower mass values requires less control effort, 
and 
 angular velocities $\smash{(\dot{q}_2,\dot{q}_3)}$ are larger during the second exploration phase (from $t\,{=}\,1$\textrm{s} until $t\,{=}\,5$\textrm{s}) in order to efficiently identify the unknown parameters of the manipulator. 
 Also, note the approximate symmetry between the two control trajectories due to the different mass and inertial properties.

\subsection{Safely transporting an uncertain payload}\label{sec:sim:ff}

Next, we verify our approach on a nonlinear six-dimensional planar free-flyer robot navigating in a cluttered environment. 
We consider the problem of cargo transport, in which the robot is attached to an uncertain payload that results in changes to the inertial properties of the system. This system mimics a cargo-unloading scenario that is one plausible near-term application of autonomous robots onboard the International Space Station \cite{AstrobeeSoftware2018,ekalaccuracycobot}. 

The state of the system is given by $\x = (\bpos,\theta,\bv,\omega)\in\R^{6}$, with $\bpos,\bv\in\R^2$ the planar position and velocity and $\theta,\omega\in\R$ the heading and angular velocity, respectively. 
For safety, we constrain $|v_i|\,{\leq}\, 0.2$ \textrm{m/s}, and $|\omega|\,{\leq}\, 0.25$ \textrm{rad/s}. 
The robot is controlled using two pairs of gas thrusters and a reaction wheel. The system's control inputs are denoted by $\ac\,{=}\,(F_x,F_y,M)\in\R^3$, where $\U=[-\bar{u}_i,\bar{u}_i]^{\times 3}$ represent the limited control authority, with $\bar{u}_{1,2}\,{=}\,0.15$ \textrm{N} and $\bar{u}_3\,{=}\,0.01$ \textrm{Nm}. 
The payload causes a change in mass/inertia properties and causes the center of mass to be offset at $\bpos_0\in\R^2$.  
The continuous-time nonlinear dynamics of the system (which we write as $\dot{\x}=\mathbf{f}_t(\cdot)$) 
are given by $\dot{\bpos} \,{=}\, \bv$, $\thetadot \,{=}\, \omega$, and  %
\begin{equation}
\label{eq:ct_dynamics_ff_planar}
m\dot{\bv}
\,{=}\,
\bF
{-}
\dot{\omega}\begin{bmatrix}
-p_{\mathrm{o}y}\\ p_{\mathrm{o}x}
\end{bmatrix} %
{+}
\omega^2 \bpos_{\mathrm{o}},
\ \,
J\dot{\omega} 
\,{=} \,
M 
{-} 
p_{\mathrm{o}x}F_y 
{+}
p_{\mathrm{o}y}F_x
.
\end{equation}
To capture payload uncertainty, we randomize the mass $m$, inertia $J$, and center of mass $\bpos_0$ of the coupled system as 
\begin{gather*}
m \sim \mathrm{Unif}(25,60) \,\textrm{kg}, 
\ \ 
J \sim \mathrm{Unif}(0.30, 0.70) \,\textrm{kg${\cdot}$m${}^2$}, 
\\
p_{\mathrm{o}i} \sim \mathrm{Unif}(-7.5,7.5)\,\textrm{cm}, \ i\,{\in}\,\{x,y\}
.
\end{gather*}
Using a zero-order hold on the controls and a forward Euler discretization scheme with $\Delta t\,{=}\,3 \, \textrm{s}$, we discretize \eqref{eq:ct_dynamics_ff_planar} as
\begin{equation}\label{eq:ff:dt_dynamics}
\x_{t+1} =  \x_t + \Delta t {\cdot} \mathbf{f}_t(\x_t,\ac_t, m, \mathbf{J},\bpos_{\mathrm{o}})
+ \ep_t
.
\end{equation}
We assume that the $\rev{\epsilon_{i,t}}$ are $\sigma_{i}$-subgaussian, each bounded as 
\scalebox{0.95}{$|\rev{\epsilon_{i,t}}|\,{\leq}\, \smash{(\sigma_i^2\chi_1^2(0.95))^{1/2}}$}, where 
\scalebox{0.95}{$\sigma_{1,2}^2=10^{-6}$},  
\scalebox{0.95}{$\sigma_{3,6}^2=10^{-5}$}, and 
\scalebox{0.95}{$\sigma_{4,5}^2=10^{-7}$}. 
We use this discrete-time system in simulation experiments and to collect training data for offline meta-learning. 
We use a nominal model $\h$ of the system using \eqref{eq:ff:dt_dynamics} with $(\bar{m},\bar{J},\bar{\bpos}_0)=(35,0.4,\mathbf{0})$, which corresponds to a double-integrator model. 
To represent the unknown model mismatch $\g(\cdot,\cdot,\param)$, 
we train an \alpaca model as described in \cite{HarrisonSharmaEtAl2019} for $6000$ iterations for all experiments, 
with $L=2$ layers of width $64$ each and an output dimension $d=32$ for each $\feat_i$.

For trajectory optimization, we use standard linear-quadratic final and stagewise costs on states and controls to minimize control cost and deviation to $\X_0$ or $\Xgoal$ depending on the phase. 
Specifically, we maximize the information cost  \eqref{eq:info_cost} while minimizing control effort and   penalizing high velocities and the final distance to $\x_{\textrm{g}}$, the center of either $\X_0$, or $\Xgoal$: 
\vspace{-2mm}
\begin{align}
\vspace{-3mm}
\min_{\bmu,\ac}
\sum_{k=0}^{N-1}    
\Big(
-\alpha_{\textrm{info}}\costinfo(\bmu_k,\ac_k)
&+
\bmu_k^\top\bQ\bmu_k + \ac_k^\top\bR\ac_k
\Big)
\nonumber
\\[-2mm]
&\hspace{-2cm}+ 
(\bmu_N\,{-}\,\x_{\textrm{g}})^\top\bQ_N(\bmu_N\,{-}\,\x_{\textrm{g}})
.
\vspace{-1mm}
\end{align}
We set $\alpha_{\textrm{info}}\,{=}\,0.025$ for exploration, whereas 
$\alpha_{\textrm{info}}\,{=}\,0$ when reaching $\Xgoal$. 
We set 
$\bQ\,{=}\,\scalebox{0.95}{$\textrm{diag}([0,0,0,1,1,10])$}$, 
$\smash{\bR\,{=}\,\scalebox{0.95}{$\textrm{diag}([10,10,10])$}}$, and $\smash{\bQ_N\,{=}\,\scalebox{0.95}{$10^3\textrm{diag}([1,1,0.1,10,10,10])$}}$ for both \ExploreOCP and \ReachOCP.

To validate our framework, 
we run \algname on 250 randomized problems with different dynamics parameters $\param$, 
four different obstacle configurations and initial/final conditions, as shown in Figure \ref{fig:ff_sim_exps}. 
With $M=2500$ and both orthogonality and $\beta$-regularization, we obtain a success rate (all constraints in \eqref{eq:joint_chance_constraint} are satisfied, the optimizer finds a feasible solution at each step of \algname, and the system reaches $\Xgoal$) at $93.2\%$ for $\delta\,{=}\,0.1$ (with $\delta\,{=}\,0.1$, the success rate should be at least $90\%$ according to Theorem \ref{thm:probabilistic_feasibility_safety}), at $90.5\%$ for $\delta\,{=}\,0.2$, and at $88.8\%$ for $\delta\,{=}\,0.5$. This experiment shows that the trajectory-wise chance constraint \eqref{eq:joint_chance_constraint} is conservatively satisfied in practice and verifies the probabilistic safety and recursive feasibility results of Theorem  \ref{thm:probabilistic_feasibility_safety}.

\textbf{Sensitivity analysis}: 
we perform an ablation study to evaluate the sensitivity of our approach to different parameters: 
to $\delta$, 
to the effect of the $\beta$-regularizer, 
to the number of samples  $M$ for reachability analysis with \randup, 
and 
to the noise intensities $\sigma_i$.   
For each parameter set, we perform 250 randomized experiments on the environments shown in Figure  \ref{fig:results_simple}. 
As with the manipulator experiments, we compare our approach to a baseline.   
Results are shown in Figure \ref{fig:sim_ff:success_with_without_beta}, in Tables \ref{fig:results:stats_smallSigEps} and \ref{fig:results:stats_highSigEps} and are discussed below:

\begin{table}[!t]
\footnotesize
\centering
\begin{tabular}{cccccccc}
\toprule 
\Tstrut
\textrm{Small $\sigma_i$} 
&
\textrm{\# Explore}
& $\x\,{\notin}\,\Xobs$
& $\x\,{\in}\,\X_{\textrm{min/max}}$
& $\x_N\,{\in}\,\Xgoal$
& $\x\,{\in}\,\X_{\textrm{all}}$
\Bstrut 
\\
\midrule
\Tstrut
\algName, \scalebox{0.9}{$\delta{=}0.1$} & 
{\scalebox{0.8}{$2.3 \,{\pm}\, 0.01$}} & 
{\scalebox{0.8}{$97.6\,{\pm}\, 1.9 \%$}} & 
{\scalebox{0.8}{$97.6\,{\pm}\, 1.9 \%$}} & 
{\scalebox{0.8}{$96.8\,{\pm}\, 2.2 \%$}} & 
{\scalebox{0.8}{$93.2\,{\pm}\, 3.1 \%$}} 
\Bstrut 
\\ 
\Tstrut
\algName, \scalebox{0.9}{$\delta{=}0.2$} & 
{\scalebox{0.8}{$2.43 \,{\pm}\, 0.19$}} &  
{\scalebox{0.8}{$95.6\,{\pm}\, 2.5 \%$}} & 
{\scalebox{0.8}{$98.8\,{\pm}\, 1.3 \%$}} & 
{\scalebox{0.8}{$98.8\,{\pm}\, 1.3 \%$}} & 
{\scalebox{0.8}{$93.6\,{\pm}\, 3.0 \%$}} 
\Bstrut 
\\ 
\Tstrut
\algName, \scalebox{0.9}{$\delta{=}0.5$} & 
{\scalebox{0.8}{$2.22 \,{\pm}\, 0.18$}} &  
{\scalebox{0.8}{$94.8\,{\pm}\, 2.7 \%$}} & 
{\scalebox{0.8}{$98.8\,{\pm}\, 1.3 \%$}} & 
{\scalebox{0.8}{$97.2\,{\pm}\, 2.0 \%$}} & 
{\scalebox{0.8}{$93.2\,{\pm}\, 3.1 \%$}} 
\Bstrut 
\\ 
\Tstrut
\hspace{-1mm}
\scalebox{1}{\rev{Mean-Equiv.}} 
\hspace{-1mm} & 
{\scalebox{0.8}{$ 0 $}}   &  
{\scalebox{0.8}{$39.6 \,{\pm}\, 6.0\%$}} & 
{\scalebox{0.8}{$99.6 \,{\pm}\, 0.8\%$}} & 
{\scalebox{0.8}{$22.8 \,{\pm}\, 5.2\%$}}   & 
{\scalebox{0.8}{$19.6 \,{\pm}\, 4.9\%$}}  
\Bstrut 
\\ 
\Tstrut
\hspace{-1mm}
\scalebox{1}{\rev{RFF, \scalebox{0.9}{$\delta{=}0.1$}}} 
\hspace{-1mm}  & %
{\scalebox{0.8}{\rev{$1.0 \,{\pm}\, 0.0\%$}}}   &  
{\scalebox{0.8}{\rev{$84.0 \,{\pm}\, 0.3\%$}}} & 
{\scalebox{0.8}{\rev{$97.6 \,{\pm}\, 0.1\%$}}} & 
{\scalebox{0.8}{\rev{$68.8 \,{\pm}\, 0.4\%$}}}   & 
{\scalebox{0.8}{\rev{$59.6 \,{\pm}\, 0.4\%$}}}     
\\ 
\bottomrule 
\end{tabular}
\vspace{0mm}
\caption{Results for 250 randomized experiments for different values of $\delta$, with low noise levels $\rev{\sigma_i}$,
$\beta$-reg., and $\smash{M=1000}$. $\x\,{\in}\,\X_{{all}}$ denotes success: all constraints are satisfied and $\x_N\,{\in}\,\X_{{goal}}$.  
}
\label{fig:results:stats_smallSigEps}
\footnotesize
\centering
\begin{tabular}{cccccccc}
\toprule 
\Tstrut
\textrm{High $\sigma_i$} 
& \textrm{\# Explore}
& $\x\,{\notin}\,\Xobs$
& $\x\,{\in}\,\X_{\textrm{min/max}}$
& $\x_N\,{\in}\,\Xgoal$
& $\x\,{\in}\,\X_{\textrm{all}}$
\Bstrut 
\\
\midrule
\Tstrut
\algName, \scalebox{0.9}{$\delta{=}0.1$} & 
{\scalebox{0.8}{$2.4 \,{\pm}\, 0.14$}} & 
{\scalebox{0.8}{$92.4\,{\pm}\, 3.3 \%$}} & 
{\scalebox{0.8}{$99.2\,{\pm}\, 1.1 \%$}} & 
{\scalebox{0.8}{$95.6\,{\pm}\, 2.5 \%$}} & 
{\scalebox{0.8}{$90.0\,{\pm}\, 3.7 \%$}} 
\Bstrut 
\\ 
\Tstrut
\algName, \scalebox{0.9}{$\delta{=}0.2$} & 
{\scalebox{0.8}{$2.32 \,{\pm}\, 0.13$}} &  
{\scalebox{0.8}{$91.6\,{\pm}\, 3.4 \%$}} & 
{\scalebox{0.8}{$100 \,{\pm}\, 0 \%$}} & 
{\scalebox{0.8}{$95.6\,{\pm}\, 2.5 \%$}} &
{\scalebox{0.8}{$89.6\,{\pm}\, 3.8 \%$}} 
\Bstrut 
\\ 
\Tstrut
\algName, \scalebox{0.9}{$\delta{=}0.5$} & 
{\scalebox{0.8}{$1.98 \,{\pm}\, 0.11$}} &  
{\scalebox{0.8}{$87.6\,{\pm}\, 4.1 \%$}} & 
{\scalebox{0.8}{$99.2\,{\pm}\, 1.1 \%$}} & 
{\scalebox{0.8}{$90.8\,{\pm}\, 3.6 \%$}} & 
{\scalebox{0.8}{$82.8\,{\pm}\, 4.7 \%$}} 
\Bstrut 
\\ 
\Tstrut
\hspace{-1mm}
\scalebox{1}{\rev{Mean-Equiv.}} 
\hspace{-1mm}  & %
{\scalebox{0.8}{$ 0 $}}   &  
{\scalebox{0.8}{$58.8 \,{\pm}\, 6.1\%$}} & 
{\scalebox{0.8}{$99.6 \,{\pm}\, 0.8\%$}} & 
{\scalebox{0.8}{$39.2 \,{\pm}\, 6.0\%$}}   & 
{\scalebox{0.8}{$37.2 \,{\pm}\, 6.0\%$}}    
\Bstrut 
\\ 
\Tstrut
\hspace{-1mm}
\scalebox{1}{\rev{RFF, \scalebox{0.9}{$\delta{=}0.1$}}} 
\hspace{-1mm}  & %
{\scalebox{0.8}{\rev{$1.0 \,{\pm}\, 0.0\%$}}}   &  
{\scalebox{0.8}{\rev{$84.0 \,{\pm}\, 0.3\%$}}} & 
{\scalebox{0.8}{\rev{$95.6 \,{\pm}\, 0.2\%$}}} & 
{\scalebox{0.8}{\rev{$65.2 \,{\pm}\, 0.4\%$}}}   & 
{\scalebox{0.8}{\rev{$54.4 \,{\pm}\, 0.4\%$}}}   
\Bstrut 
\\ 
\bottomrule 
\end{tabular}
\vspace{0mm}
\caption{Results for 250 randomized experiments for different values of $\delta$, with high noise levels $\rev{\sigma_i}$, $\beta$-reg., and $M=2500$.}%
\label{fig:results:stats_highSigEps}
\vspace{-6mm}
\end{table}

\begin{figure*}[!t]
\begin{minipage}{1\textwidth}
\centering
\vspace{-4mm}
\includegraphics[width=0.28\linewidth,trim=10 0 5 0, clip]{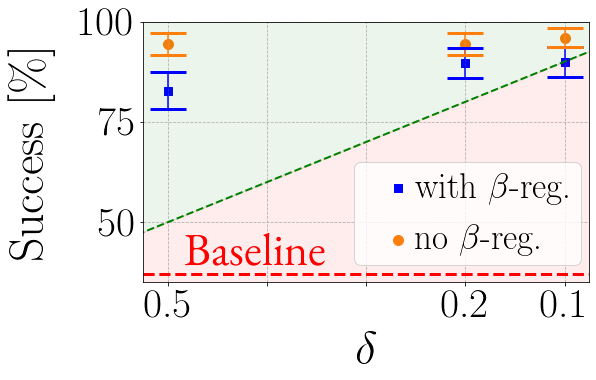}
\hspace{1mm}
\includegraphics[width=0.28\linewidth,trim=0 0 5 0, clip]{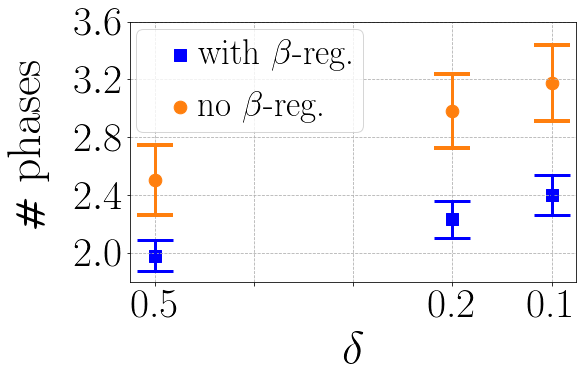}
\hspace{1mm}
\includegraphics[width=0.39\linewidth,trim=0 0 0 0, clip]{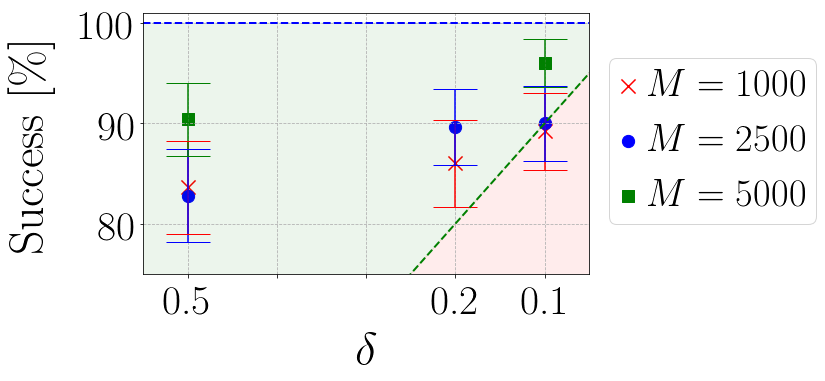}
\vspace{-3mm}
\caption{Simulation results for ablation study, using 250 free-flyer experiments in randomized environments (see Fig. \ref{fig:results_simple}) with high noise levels $\sigma_i$, with and without $\beta$-regularization (left and middle figures, using $M = 2500$) and for different number values of $M$ (right). 
On plots showing success percentages (all constraints are satisfied and $\x_N{\in}\Xgoal$), the green region denotes results where the success percentage is at or above the desired probability of success  $(1{-}\delta)$ (see Theorem \ref{thm:probabilistic_feasibility_safety}) and the red region indicates where the true probability of success is lower than  desired. Error bars correspond to $95\%$ confidence intervals.
}
\label{fig:parameters_sweep}
\label{fig:sim_ff:success_with_without_beta}
\end{minipage}
\begin{minipage}{1\textwidth}
\centering
\includegraphics[width=0.198\linewidth,trim=52 33 35 0, clip]{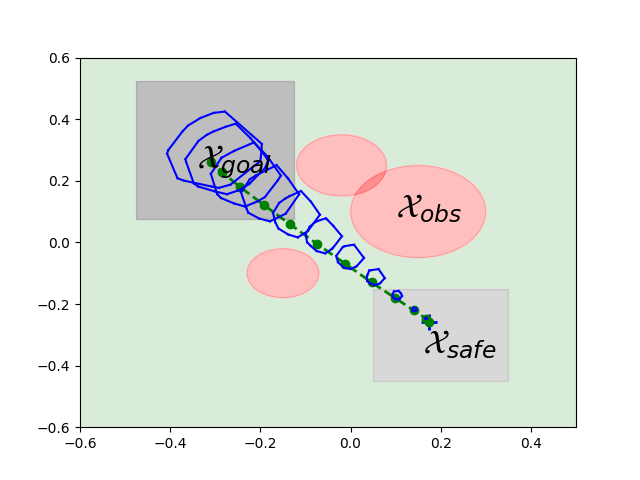}
\includegraphics[width=0.198\linewidth,trim=52 33 35 0, clip]{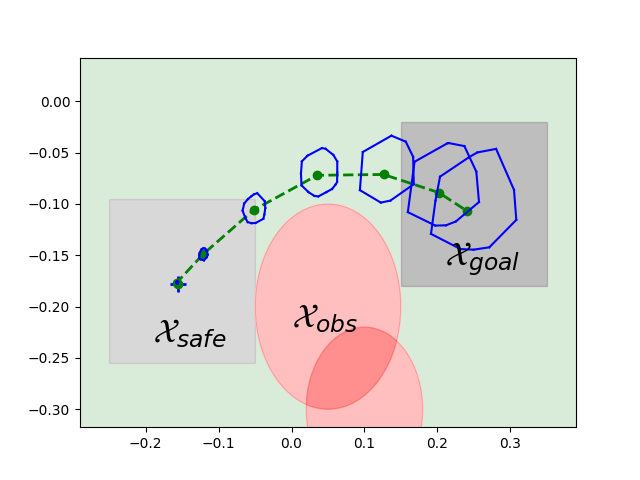}
\includegraphics[width=0.198\linewidth,trim=52 33 35 0, clip]{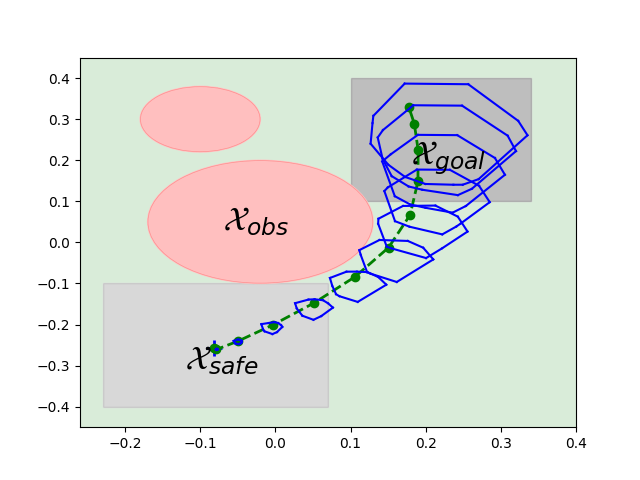}
\includegraphics[width=0.198\linewidth,trim=52 33 35 0, clip]{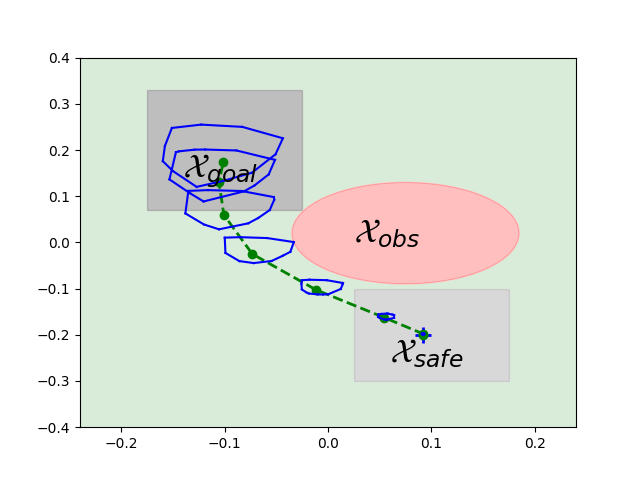}
\hspace{2mm}
    \includegraphics[width=0.15\linewidth,trim=40 40 30 30, clip]{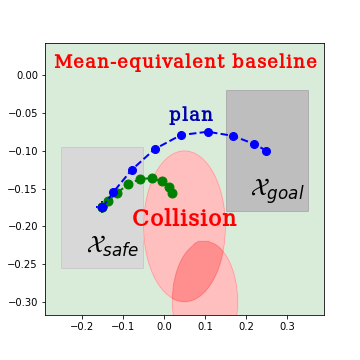}
\caption{\textbf{Left} four figures: scenarios considered for the ablation study with representative \textcolor{blue}{reaching trajectories} obtained with \seels and \textcolor{dartmouthgreen}{true executed trajectories}. 
\textbf{Right}, mean-equivalent baseline that uses the mean parameters $\kbar_i$ and only considers uncertainty from the disturbances $\ep_k$ for planning. Due to high dynamics  uncertainty, attempting to reach $\Xgoal$ is initially unsafe: the baseline violates velocity and final constraints and collides with an obstacle. }
\label{fig:results_simple}
\end{minipage}
\end{figure*}

\begin{itemize}[leftmargin=4mm]
  \setlength\itemsep{0.5mm}
  	\item \textit{\rev{Mean-equivalent} baseline}: we compare to a method which only considers uncertainty from the additive disturbances $\rev{\ep_t}$ and uses the mean parameters $\kbar_i$ for planning\rev{, with meta-learned \alpaca features $\feat_i$}. This baseline thus assumes reaching $\Xgoal$ directly is safe. This naive approach violates the trajectory-wise chance constraint \eqref{eq:joint_chance_constraint}, again demonstrating the need for uncertainty quantification and sequential active learning to reliably solve this problem.
  	\item \rev{\textit{Random Fourier features (RFF) baseline}: we replace \alpaca with a Bayesian model with $d=128$ random Fourier features \cite{Rahimi2007} as described in \cite{Wilson2020}. 
  	  Offline, prior to running Algorithm \ref{alg:explo_reach_algalg}, we select length-scale hyper-parameters and calibrate the priors $\N(\kbar_{i,0}, \sigma_i^2 \Linv_{i,0})$ over the weights to obtain an accurate uncertainty representation satisfying Assumption \ref{assum:calib}. 
  	    We observe that this approach does not achieve the desired success rates, which we attribute to substantial violations of Assumption \ref{assum:capacity_model}. Making accurate multi-step open-loop  predictions appears challenging without expressive meta-learned neural network features.}
    \item \textit{$\delta$}: 
	as a consequence of Theorem \ref{thm:probabilistic_feasibility_safety}, the algorithm's conservatism can be tuned by the
choice of $\delta$, which is supported by our results. 
In particular, by opting for a lower probability of safety $1{-}\delta$, the goal is reached faster on average. %
Importantly, \seels conservatively satisfies the chance constraint \eqref{eq:joint_chance_constraint} for these different parameters. 
    \item \textit{$M$}: increasing $M$ leads to increased success rate and safety probability, which is supported by the results shown in Figure \ref{fig:parameters_sweep} (right). 
    Theoretically, under \textbf{A\ref{assum:invariant_X0}-\ref{assum:calib}},  
success with high probability (as stated in Theorem  \ref{thm:probabilistic_feasibility_safety}) is guaranteed as long as the number of samples $M$ is large enough (a consequence of \cite[Theorem 2]{randSets}), which is supported by our experiments. 
In practice, one should choose the largest value of $M$ to satisfy a given computational budget and multiple techniques could be used to improve the quality of the approximation, e.g., parallelization on GPUs and importance-sampling approaches \cite{williams2017information}, which will be considered in future work. We refer the reader to \cite{randSets} and \cite{LewJansonEtAl2021} for further details. 
    \item \textit{$\beta$-regularization} 
    reduces conservatism while still guaranteeing probabilistic safety in practice, i.e. the model still satisfies \textbf{A\ref{assum:capacity_model}} and \textbf{A\ref{assum:calib}}. 
    Indeed, Figure \ref{fig:sim_ff:success_with_without_beta} shows that the goal is reached faster on average and the system still conservatively satisfies \eqref{eq:joint_chance_constraint}. 
    This also implies that the model class with last-layer adaptation of \alpaca is expressive enough to enable meta-training of models fulfilling \textbf{A\ref{assum:capacity_model}} and \textbf{A\ref{assum:calib}} with varying levels of conservatism, which could motivate future research on better meta-training of such models with different architectures and regularizers.
    \item \textit{$\sigma_i$}: we also consider two noise levels:
\begin{enumerate}
    \item $\sigma_i^2=10^{-7}$ for  \mbox{\footnotesize $i{=}1,2,4,5$}, and 
$\sigma_i^2=10^{-6}$ for  \mbox{\footnotesize$i{=}3,6$}.
\item $\sigma_i^2=10^{-6}$ for  \mbox{\footnotesize $i{=}1,2$},  
$\sigma_i^2=10^{-5}$ for  \mbox{\footnotesize$i{=}3,6$}, and 
$\sigma_i^2=10^{-7}$ for  \mbox{\footnotesize$i{=}4,5$}.
\end{enumerate}

Results for these different noise levels for different $\delta$ are reported in Tables \ref{fig:results:stats_smallSigEps} and  \ref{fig:results:stats_highSigEps}. 
From Table \ref{fig:results:stats_smallSigEps}, we see that the performance and overall probability of safety for small $\sigma_i$ is not sensitive to the chosen value of $\delta$. We conjecture that failures are mostly due to under-approximations from the approximate computation of the reachable sets with \randup. For higher noise levels, it is evident that the conservatism of the algorithm can be tuned by choosing a different value for $\delta$, since failures come from statistical errors from updating the model with noisy data (see Theorem \ref{thm:conf-sets}). We also observe faster times to reach $\Xgoal$ when opting for lower probability of safety. In all scenarios, \algName solves the problem with probability at least $(1\,{-}\,\delta)$, verifying the results of Theorem \ref{thm:probabilistic_feasibility_safety}.
\end{itemize}

\section{Hardware experiments}\label{sec:hardware}

Next, we verify our framework via hardware experiments. 
We consider the planar spacecraft free-flyer platform controlled by eight thrusters as shown in Figure \ref{fig:thrusters}. 
\begin{figure}[!b]
\centering
\vspace{-3mm}
\includegraphics[width=0.99\linewidth]{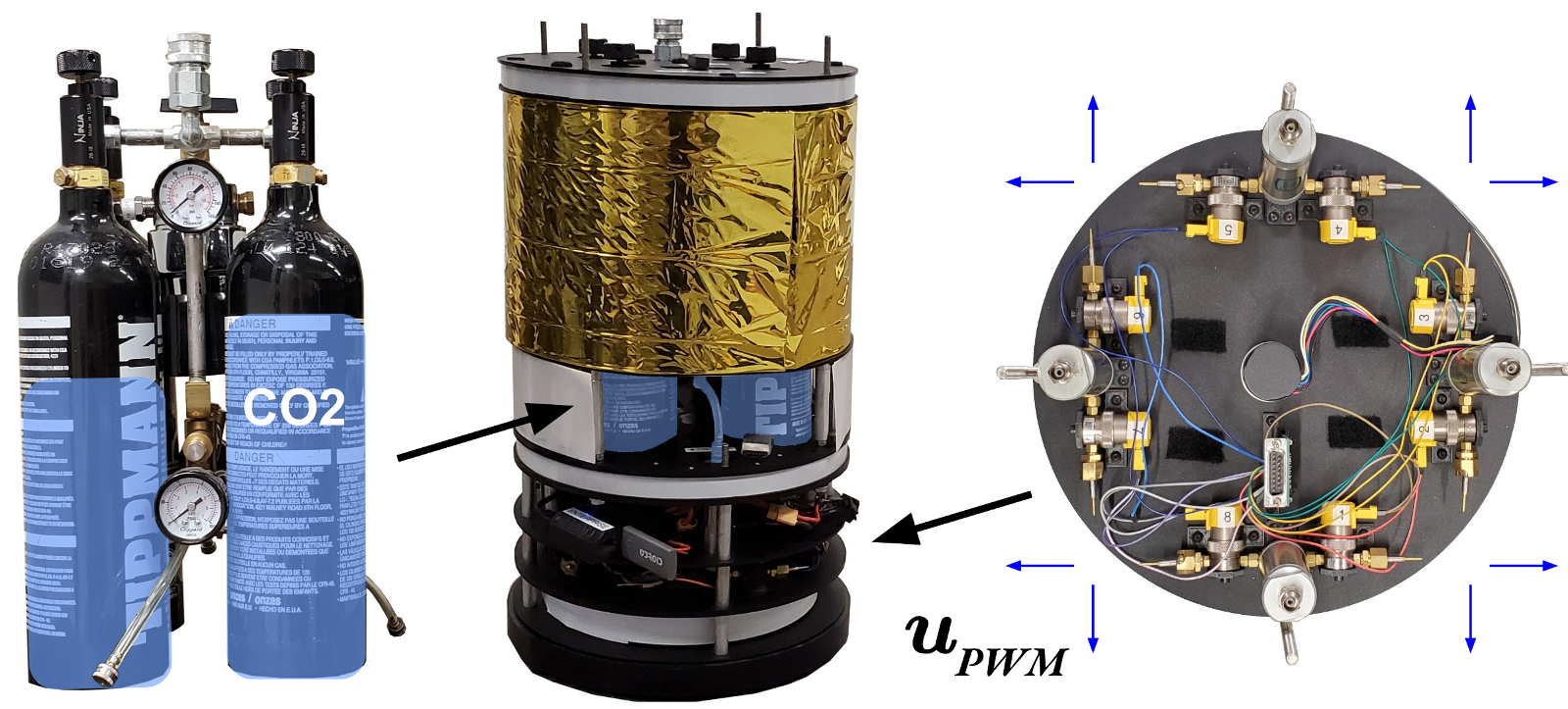}
\vspace{-6mm}
\caption{
The free-flyer robot uses air bearings and cold-gas thrusters to navigate with minimal friction and external forces on a smooth, flat granite table, emulating microgravity dynamics within a plane. 
\textbf{Right }(top view): The eight cold-gas thrusters are mixed to \rev{produce} desired forces and moments on the robot, modulated using 10 Hz PWM signals sent to the thruster valves. 
	The lack of reaction wheel on this system makes it slightly more challenging to control than the simulated free-flyer from Section \ref{sec:sim:ff}.
}
\label{fig:thrusters}
\end{figure}
As the system weighs ${\sim} 16\textrm{kg}$ and each thruster can only apply a maximum force of $0.4\textrm{N}$, an accurate model and dynamically-feasible trajectory are necessary to safely reach a goal region.  
Other challenges include imperfect actuators,  time-varying mass of the system as the gas tanks deplete, and unmodeled friction and tilt of the table surface.  
Our implementation follows Section \ref{sec:sim:ff} (with $\delta\,{=}\,0.1$,  orthogonality and $\beta$-regularization) with the following modifications:
\begin{itemize}[leftmargin=3.5mm]
  \setlength\itemsep{0.5mm}
    \item \textit{Dynamics}: we meta-train using \eqref{eq:ct_dynamics_ff_planar} discretized at $10\textrm{Hz}$ with an additional random force $\mathbf{F}_{\textrm{tilt}}$ to account for the tilt of the table. For offline meta-training, the parameters are randomized as 
\begin{gather*}
m \sim \mathrm{Unif}(8,40) \,\textrm{kg}, 
\ \ 
J \sim \mathrm{Unif}(0.08,0.30) \,\textrm{kg${\cdot}$m${}^2$}, 
\\
p_{\mathrm{o}i} \sim \mathrm{Unif}(\pm 5)\,\textrm{cm},
\ \ 
F_{\mathrm{tilt},i} \sim \mathrm{Unif}(\pm 0.3)\,\textrm{N}, 
\ \ 
	i\,{\in}\,\{x,y\}
.
\end{gather*}
    \item \textit{Regulation}: we regulate the system with a PD \rev{controller  at} the end of each phase as a new plan is computed.\footnote{We solve four \ExploreOCP with different horizons $N$ (\seels, line $7$) which takes a total of $5$ seconds in average with our Python implementation, measured on a computer with a 3.70GHz Intel Core i7-8700K CPU. As the computed control inputs are directly used to steer the system, this delay may be problematic if the robot is in movement as a solution is computed.}  We do not use these intermediate trajectories to learn dynamics. 
    \item \textit{Feedback}: we use an LQR controller to stabilize the trajectory around reachable set centers $\bmu_k$ and reduce uncertainty. \rev{This feedback controller is computed using the prior mean parameters $\kbar_{i,0}$ of \alpaca, linearized around $\x(0)$. }  We account for this controller during planning as in \cite{LewBonalli2020} and apply the control $\ac_k\,{=}\,\ac_{\textrm{seels},k}\,{+}\,\textrm{LQR}(\x_k\,{-}\,\bmu_k)$ at $\smash{\frac{1}{3}}\textrm{Hz}$.
\end{itemize}

\begin{figure*}[!t]
\begin{minipage}{1.0\textwidth}
\centering
\includegraphics[width=0.30\linewidth,trim=40 0 0 0, clip]{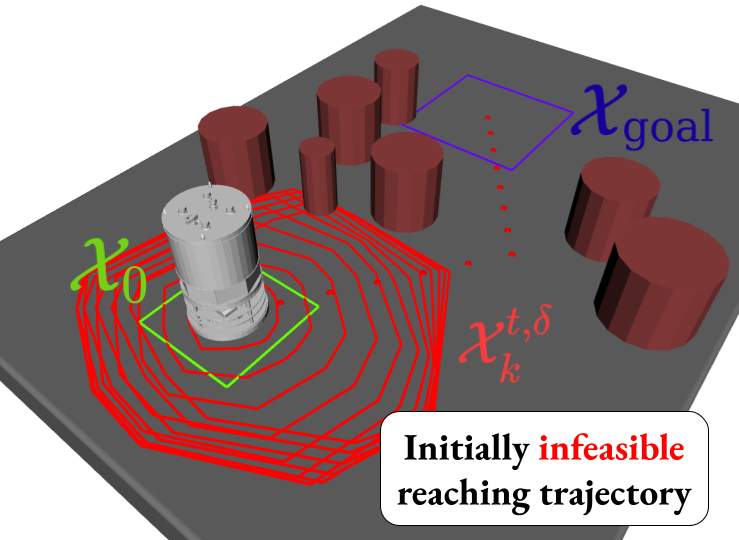}
\hspace{0.5mm}
\includegraphics[width=0.25\linewidth,trim=10 0 30 0, clip]{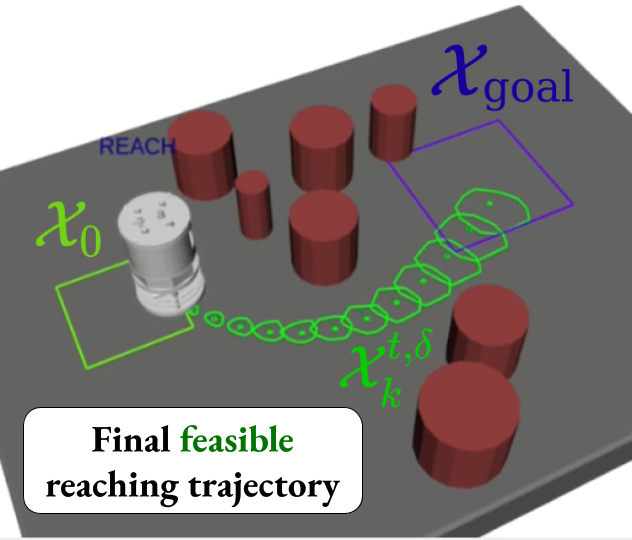}
\hspace{0.5mm}
\vspace{2mm}
\includegraphics[width=0.40\linewidth,trim=10 30 80 0, clip]{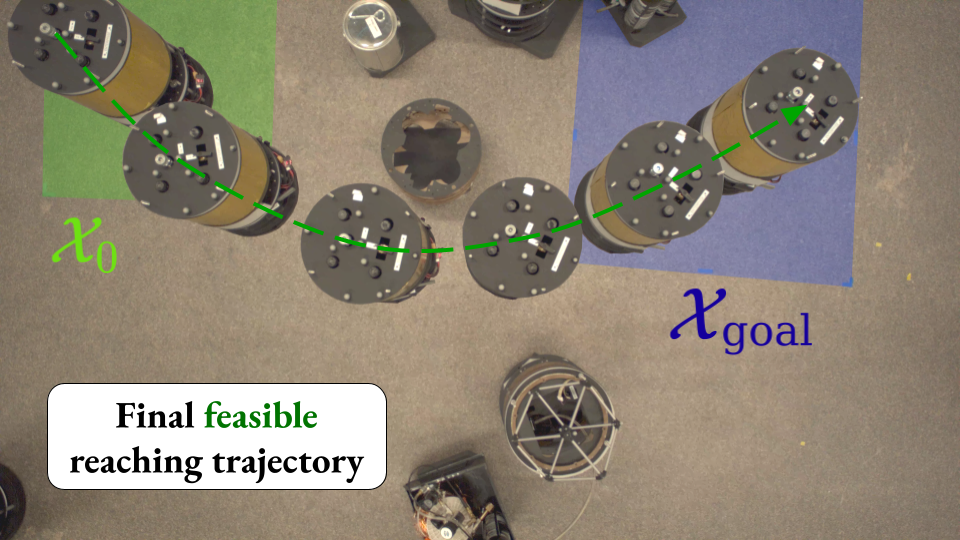}
\end{minipage}
\vspace{2mm}
\begin{minipage}{1.\textwidth}
\centering
\includegraphics[width=0.99\linewidth,trim=25 20 20 28, clip]{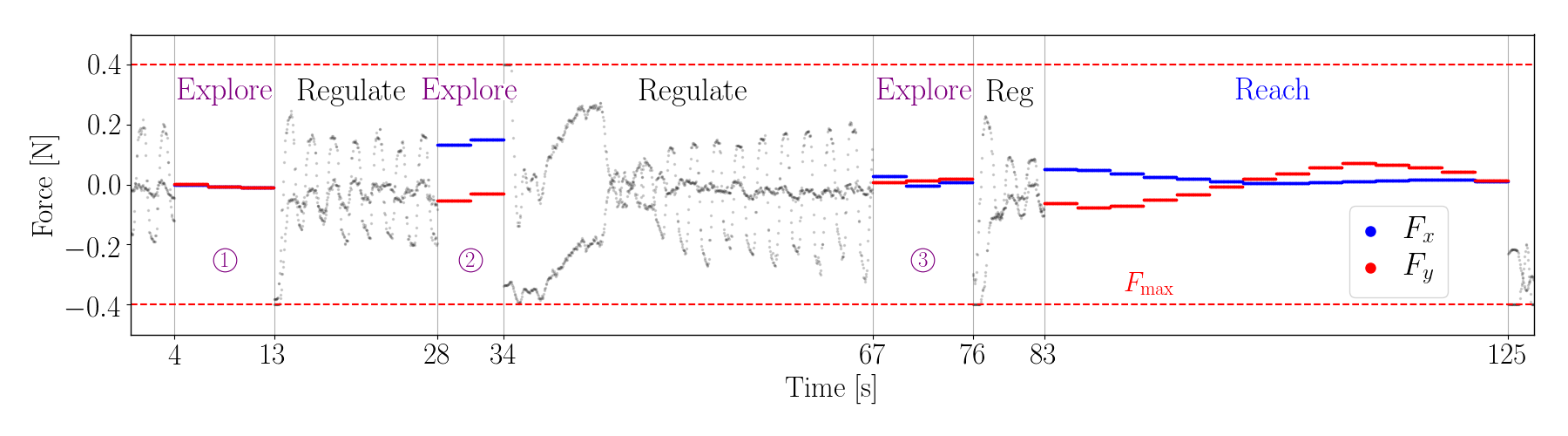}
\end{minipage}

\caption{Visualization and forces applied during hardware experiment. 
\textbf{Top (left) \& bottom (1)}: uncertainty is too high to reach and an exploration trajectory with negligible control inputs is computed to infer external forces. 
\textbf{Bottom (2) and (3)}: \rev{for the second and third exploration phases}, uncertainty is still too high and larger forces are computed to infer inertial properties. 
\textbf{Top (middle, right) \& bottom}: the goal is safely reached. 
}
\label{fig:exps:hardware:experiment}
\end{figure*}

We use a motion capture system and interface our Python implementation with ROS.  
Using \algName, we compute a desired wrench $\ac\,{=}\,[F_x,F_y,M]$ in the world frame, which we then convert to the local frame of the robot, yielding $\ac_{\textrm{body}}$. %
We then translate these forces and torque into corresponding PWM signals for the actuators:
$$
\ac_\textrm{PWM} =
\textrm{saturate}(\bm{A}\ac_{\textrm{body}})\in\R^8.
$$
The constant matrix $\bm{A}\,{\in}\,\R^{8{\times}3}$ accounts for the maximum applicable force of each thruster of $0.4$N. With the control bounds for planning set conservatively to $|F_{x,y}|\leq 0.4$N and $|M|\leq 0.03$Nm, we can ensure that actuator saturation will neither occur during exploration nor exploitation.

This approach of using a lower-level control mapping to turn world-frame forces and torques to gas-thruster commands allows us to use the same dynamics model used in the simulation experiments (Section \ref{sec:results:sim}) for the hardware experiments (retraining it to also account for the slight tilt of the table).
An alternative would be to learn a dynamics model in terms of gas-thruster inputs directly. While this modeling choice would allow factoring in dynamics uncertainty at the actuator level (e.g., actuator failure), this would also yield a more complex class of possible models and thus require a larger model to ensure the assumption \textbf{A\ref{assum:capacity_model}} on model capacity is still satisfied.

With these modifications, we apply \algName on multiple obstacle fields and analyze the representative experiment in Figure \ref{fig:exps:hardware:experiment}. 
Initially, uncertainty is too large to safely reach $\Xgoal$ and \ReachOCP is infeasible. Instead, the agent performs a short exploration trajectory with negligible control inputs to first infer the tilt of the table (\ExploreOCP is necessarily feasible for a planning horizon short enough, see Theorem \ref{thm:probabilistic_feasibility_safety}).  
Indeed, applying larger control inputs would yield larger reachable sets (due to inertial uncertainty) and violate constraints. 
Then, \seels still deems reaching to be unsafe and thus performs two additional exploration trajectories to infer the inertial properties of the system. Once \ReachOCP is feasible, the robot finally reaches the goal. All constraints are satisfied throughout the experiment.

We present further results in the video available at \scalebox{0.95}{\href{https://youtu.be/5ux1Jg4FgEs}{https://youtu.be/5ux1Jg4FgEs}} 
and summarize them below:
\begin{itemize}[leftmargin=3.5mm]
  \setlength\itemsep{0.5mm}
    \item \textit{Baseline}: we compare with a mean-equivalent approach which collides with obstacles despite the LQR controller. This shows that considering uncertainty during planning and learning the dynamics is necessary to reliably solve this task.
    \item \textit{Reliability}: our approach works repeatably in multiple obstacle fields, including a payload-grasping experiment. 
    We also present two experiments where the agent infers its dynamics, reaches a first goal, and is then capable of reaching another destination directly without needing further exploration. All constraints are satisfied throughout the experiments (Theorem \ref{thm:probabilistic_feasibility_safety} can easily be extended to also guarantee the success of such multi-setpoint problems).
    \item \textit{Failure modes} are also discussed.  
    We observed that failures to reach $\Xgoal$ are due to one of the following:
    (1) the model adapting incorrectly during the initial exploration phases due to noisy measurements and sim-to-real mismatch, resulting in large variations in the parameters $\k_i$, which is easy to detect and could be addressed by re-initializing parameters based on the variances $\sigma_i\smash{\bLambda_i^{-1}}$, \rev{using event-triggered learning \cite{Solowjow2020}, or} by increasing $\sigma_i$ to learn less rapidly, 
    (2) infeasible trajectory optimization, or  %
    (3) gas tanks that are too depleted to apply desired forces. 
    In these situations, we observed that no feasible trajectory to the goal is found and the algorithm defaults to exploring safely within $\X_0$. 
\end{itemize}

\section{Conclusion} 
\label{sec:conclusion}

To safely perform tasks under high initial uncertainty, 
we presented a framework to sequentially explore and learn the properties of a dynamical system while guaranteeing safety at all times jointly as a single trajectory-wise chance constraint. 
We demonstrated %
our approach through simulations 
and hardware experiments, where we showed that a free-flying robot emulating a spacecraft can reliably infer its dynamics and safely reach a goal. 

\textit{Future work}. 
Problems with time-varying dynamics (e.g., the system is attached to a new payload during operation and gas tanks are depleting) could be addressed using continual meta-learning approaches \cite{HarrisonSharmaEtAl2020}, which would require a new analysis to guarantee safety.
Future work could also %
derive algorithms to compute invariant sets satisfying \textbf{A\ref{assum:invariant_X0}} by leveraging the prior on the meta-learning model. 
Extensions of \cite{randSets} would strengthen the safety guarantees to the case of a finite number of samples for uncertainty propagation and 
work on free-final-time trajectory optimization would replace the search over horizons $N$ in our algorithm. 
Investigating regret bounds for such constrained problems would also help provide guarantees on the time required to perform a task given known geometric properties of the problem.   
Finally, our framework could be combined with existing learning-based controllers to improve performance.

\section*{Acknowledgments}
\rev{The authors thank Edward Schmerling and the anonymous reviewers for their feedback which significantly helped improve the presentation of this work}. This work was supported in part by NASA under the University Leadership Initiative (Grant \#80NSSC20M0163) and the Early Stage Innovations program, and by DARPA under the Assured Autonomy program. This article solely reflects the opinions and conclusions of its authors.

\bibliographystyle{IEEEtran}
\bibliography{ASL_papers,main}

\newcommand{\noopsort}[1]{} \newcommand{\printfirst}[2]{#1}
  \newcommand{\singleletter}[1]{#1} \newcommand{\switchargs}[2]{#2#1}
\begin{thebibliography}{10}
\providecommand{\url}[1]{#1}
\csname url@samestyle\endcsname
\providecommand{\newblock}{\relax}
\providecommand{\bibinfo}[2]{#2}
\providecommand{\BIBentrySTDinterwordspacing}{\spaceskip=0pt\relax}
\providecommand{\BIBentryALTinterwordstretchfactor}{4}
\providecommand{\BIBentryALTinterwordspacing}{\spaceskip=\fontdimen2\font plus
\BIBentryALTinterwordstretchfactor\fontdimen3\font minus
  \fontdimen4\font\relax}
\providecommand{\BIBforeignlanguage}[2]{{%
\expandafter\ifx\csname l@#1\endcsname\relax
\typeout{** WARNING: IEEEtran.bst: No hyphenation pattern has been}%
\typeout{** loaded for the language `#1'. Using the pattern for}%
\typeout{** the default language instead.}%
\else
\language=\csname l@#1\endcsname
\fi
#2}}
\providecommand{\BIBdecl}{\relax}
\BIBdecl

\bibitem{recht2019}
B.~Recht, ``A tour of reinforcement learning: The view from continuous
  control,'' \emph{{Annual Review of Control, Robotics, and Autonomous
  Systems}}, vol.~2, no.~1, pp. 253--279, 2019.

\bibitem{deisenroth2015}
M.~Deisenroth, D.~Fox, and C.~Rasmussen, ``Gaussian processes for
  data-efficient learning in robotics and control,'' \emph{{IEEE Transactions
  on Pattern Analysis \& Machine Intelligence}}, vol.~37, no.~2, pp. 408--423,
  2015.

\bibitem{levine2016end}
S.~Levine, C.~Finn, T.~Darrell, and P.~Abbeel, ``End-to-end training of deep
  visuomotor policies,'' \emph{{Journal of Machine Learning Research}},
  vol.~17, pp. 1--40, 2016.

\bibitem{Coulson2018DataEnabledPC}
J.~Coulson, J.~Lygeros, and F.~D{\"o}rfler, ``Data-enabled predictive control:
  In the shallows of the {DeePC},'' in \emph{{European Control Conference}},
  2019.

\bibitem{Coulson2020Distributional}
------, ``Distributionally robust chance constrained data-enabled predictive
  control,'' in \emph{{Proc.\ IEEE Conf.\ on Decision and Control}}, 2019.

\bibitem{Berberich2020Robust}
J.~Berberich, J.~K{\"o}hler, M.~A. M{\"u}ller, and F.~Allg{\"o}wer, ``Robust
  constraint satisfaction in data-driven {MPC},'' in \emph{{Proc.\ IEEE Conf.\
  on Decision and Control}}, 2020.

\bibitem{Berberich2020}
------, ``Data-driven model predictive control with stability and robustness
  guarantees,'' \emph{{IEEE Transactions on Automatic Control}}, pp. 1--1,
  2020.

\bibitem{Berberich2021}
J.~Berberich, J.~K\"ohler, M.~A. M\"uller, and F.~Allg\"ower. (2021) Linear
  tracking {MPC} for nonlinear systems. Available at
  \url{https://arxiv.org/abs/2105.08560} and at
  \url{https://arxiv.org/abs/2105.08567}.

\bibitem{Slotine1987}
J.-J.~E. Slotine and W.~Li, ``On the adaptive control of robot manipulators,''
  \emph{{Int.\ Journal of Robotics Research}}, vol.~6, no.~3, pp. 49--59, 1987.

\bibitem{Lopez2021}
B.~T. Lopez and J.-J.~E. Slotine, ``Adaptive nonlinear control with contraction
  metrics,'' \emph{{IEEE Control Systems Letters}}, vol.~5, no.~1, pp.
  205--210, 2021.

\bibitem{Mania2020}
H.~Mania, M.~I. Jordan, and B.~Recht. (2020) Active learning for nonlinear
  system identification with guarantees. Available at
  \url{https://arxiv.org/abs/2006.10277}.

\bibitem{Kakade2020}
S.~Kakade, A.~Krishnamurthy, K.~Lowrey, M.~Ohnishi, and W.~Sun, ``Information
  theoretic regret bounds for online nonlinear control,'' in \emph{{Conf.\ on
  Neural Information Processing Systems}}, 2020.

\bibitem{shi2019ICRAnnlander}
G.~Shi, X.~Shi, M.~O'Connell, R.~Yu, K.~Azizzadenesheli, A.~Anandkumar, Y.~Yue,
  and S.-J. Chung, ``Neural lander: Stable drone landing control using learned
  dynamics,'' in \emph{{Proc.\ IEEE Conf.\ on Robotics and Automation}}, 2019.

\bibitem{schmidhuber1987evolutionary}
J.~Schmidhuber, ``Evolutionary principles in self-referential learning, or on
  learning how to learn: the meta-meta-... hook,'' Ph.D. dissertation,
  Technische Universit{\"a}t M{\"u}nchen, 1987.

\bibitem{santoro2016meta}
A.~Santoro, S.~Bartunov, M.~Botvinick, D.~Wierstra, and T.~Lillicrap,
  ``Meta-learning with memory-augmented neural networks,'' in \emph{{Int.\
  Conf.\ on Machine Learning}}, 2016.

\bibitem{amit2018meta}
R.~Amit and R.~Meir, ``Meta-learning by adjusting priors based on extended
  {PAC-Bayes} theory,'' in \emph{{Int.\ Conf.\ on Machine Learning}}, 2018.

\bibitem{finn2017model}
C.~Finn, P.~Abbeel, and S.~Levine, ``Model-agnostic meta-learning for fast
  adaptation of deep networks,'' in \emph{{Int.\ Conf.\ on Machine Learning}},
  2017.

\bibitem{RichardsAzizanEtAl2021}
S.~M. Richards, N.~Azizan, J.-J.~E. Slotine, and M.~Pavone,
  ``Adaptive-control-oriented meta-learning for nonlinear systems,'' in
  \emph{{Robotics: Science and Systems}}, 2021.

\bibitem{NagabandiICLR2019}
A.~Nagabandi, I.~Clavera, L.~Simin, R.~S. Fearing, P.~Abbeel, S.~Levine, and
  C.~Chelsea~Finn, ``Learning to adapt in dynamic real-world environments
  through meta-reinforcement learning,'' in \emph{{Int.\ Conf.\ on Learning
  Representations}}, 2019.

\bibitem{Belkhale2021}
S.~Belkhale, R.~Li, G.~Kahn, R.~McAllister, R.~Calandra, and S.~Levine,
  ``Model-based meta-reinforcement learning for flight with suspended
  payloads,'' \emph{{IEEE Robotics and Automation Letters}}, 2021.

\bibitem{HarrisonSharmaEtAl2018}
J.~Harrison, A.~Sharma, and M.~Pavone, ``Meta-learning priors for efficient
  online bayesian regression,'' in \emph{{Workshop on Algorithmic Foundations
  of Robotics}}, 2018.

\bibitem{williams2006gaussian}
C.~Williams and C.~E. Rasmussen, \emph{Gaussian processes for machine
  learning}.\hskip 1em plus 0.5em minus 0.4em\relax MIT press, 2006.

\bibitem{Fiedler2021}
C.~Fiedler, C.~W. Scherer, and S.~Trimpe. (2021) Practical and rigorous
  uncertainty bounds for {Gaussian} process regression. Available at
  \url{https://arxiv.org/abs/2105.02796}.

\bibitem{berkenkamp2017safe}
F.~Berkenkamp, M.~Turchetta, A.~Schoellig, and A.~Krause, ``Safe model-based
  reinforcement learning with stability guarantees,'' in \emph{{Conf.\ on
  Neural Information Processing Systems}}, 2017.

\bibitem{koller2018}
T.~Koller, F.~Berkenkamp, M.~Turchetta, and A.~Krause, ``Learning-based model
  predictive control for safe exploration,'' in \emph{{Proc.\ IEEE Conf.\ on
  Decision and Control}}, 2018.

\bibitem{abbasi2011improved}
Y.~Abbasi-Yadkori, D.~P{\'a}l, and C.~Szepesv{\'a}ri, ``Improved algorithms for
  linear stochastic bandits,'' in \emph{{Conf.\ on Neural Information
  Processing Systems}}, 2011.

\bibitem{blackmore2010}
L.~Blackmore, M.~Ono, A.~Bektassov, and B.~C. Williams, ``A probabilistic
  particle-control approximation of chance-constrained stochastic predictive
  control,'' \emph{{IEEE Transactions on Robotics}}, vol.~26, no.~3, pp.
  502--517, 2010.

\bibitem{IvanovicPavone2019}
B.~Ivanovic and M.~Pavone, ``The {Trajectron}: Probabilistic multi-agent
  trajectory modeling with dynamic spatiotemporal graphs,'' in \emph{{IEEE
  Int.\ Conf.\ on Computer Vision}}, 2019.

\bibitem{hewing2018cautious}
L.~Hewing, J.~Kabzan, and M.~N. Zeilinger, ``Cautious model predictive control
  using {Gaussian} process regression,'' \emph{{IEEE Transactions on Control
  Systems Technology}}, vol.~28, no.~6, pp. 2736--2743, 2020.

\bibitem{CastilloRAL2020}
M.~Castillo-Lopez, S.~A. Ludivig, P. Sajadi-Alamdari, J.~L. Sanchez-Lopez,
  M.~A. Olivares-Mendez, and H.~Voos, ``A real-time approach for
  chance-constrained motion planning with dynamic obstacles,'' \emph{{IEEE
  Robotics and Automation Letters}}, vol.~5, no.~2, pp. 3620 -- 3625, 2019.

\bibitem{LewBonalli2020}
T.~Lew, R.~Bonalli, and M.~Pavone, ``Chance-constrained sequential convex
  programming for robust trajectory optimization,'' in \emph{{European Control
  Conference}}, 2020.

\bibitem{Polymenakos2020}
K.~Polymenakos, L.~Laurenti, A.~Patane, J.~P. Calliess, L.~Cardelli,
  M.~Kwiatkowska, A.~Abate, and S.~Roberts, ``Safety guarantees for planning
  based on iterative {Gaussian} processes,'' in \emph{{Proc.\ IEEE Conf.\ on
  Decision and Control}}, 2020.

\bibitem{Khojasteh_L4DC20}
M.~J. Khojasteh, V.~Dhiman, M.~Franceschetti, and N.~Atanasov, ``Probabilistic
  safety constraints for learned high relative degree system dynamics,'' in
  \emph{2nd Annual Conference on Learning for Dynamics \& Control}, 2020.

\bibitem{ChengKhojasteh2020}
R.~Cheng, M.~J. Khojasteh, A.~D. Ames, and J.~W. Burdick, ``Safe multi-agent
  interaction through robust control barrier functions with learned
  uncertainties,'' in \emph{{Proc.\ IEEE Conf.\ on Decision and Control}},
  2020.

\bibitem{FreyRSS2020}
K.~M. Frey, T.~J. Steiner, and J.~P. How, ``Collision probabilities for
  continuous-time systems without sampling,'' in \emph{{Robotics: Science and
  Systems}}, 2020.

\bibitem{SchmerlingPavone2017}
E.~Schmerling and M.~Pavone, ``Evaluating trajectory collision probability
  through adaptive importance sampling for safe motion planning,'' in
  \emph{{Robotics: Science and Systems}}, 2017.

\bibitem{blackmore2011chance}
L.~Blackmore, M.~Ono, and B.~C. Williams, ``Chance-constrained optimal path
  planning with obstacles,'' \emph{{IEEE Transactions on Robotics}}, vol.~27,
  no.~6, pp. 1080--1094, 2011.

\bibitem{Hwangbo2019}
J.~Hwangbo, J.~Lee, A.~Dosovitskiy, D.~Bellicoso, V.~Tsounis, V.~Koltun, and
  M.~Hutter, ``Learning agile and dynamic motor skills for legged robots,''
  \emph{{Science Robotics}}, vol.~4, no.~26, 2019.

\bibitem{BerkenkampThesis2018}
F.~Berkenkamp, ``Safe exploration in reinforcement learning: Theory and
  applications in robotics,'' Ph.D. dissertation, {Institute for Machine
  Learning, ETH Z\"urich}, 2018.

\bibitem{Chow2020SafeRL}
Y.~Chow, O.~Nachum, A.~Faust, E.~Duenez-Guzman, and M.~Mohammad~Ghavamzadeh,
  ``Lyapunov-based safe policy optimization for continuous control,'' in
  \emph{{Conf.\ on Robot Learning}}, 2020.

\bibitem{LewSharmaHarrisonRSS2019}
T.~Lew, A.~Sharma, J.~Harrison, and M.~Pavone, ``Safe learning and control
  using meta-learning,'' in \emph{Robotics: Science and Systems, Workshop on
  Robust Autonomy}, 2019.

\bibitem{NakkaRAL2020}
Y.~N. Nakka, A.~Liu, G.~Shi, A.~Anandkumar, Y.~Yue, and S.~J. Chung,
  ``Chance-constrained trajectory optimization for safe exploration and
  learning of nonlinear systems,'' \emph{{IEEE Robotics and Automation
  Letters}}, vol.~6, no.~2, pp. 389--396, 2020.

\bibitem{LewSDE2021}
T.~Lew, A.~Sharma, J.~Harrison, E.~Schmerling, and M.~Pavone, ``On the problem
  of reformulating systems with uncertain dynamics as a stochastic differential
  equation,'' 2021, available at \url{https://arxiv.org/abs/2111.06084}.

\bibitem{ekalaccuracycobot}
M.~Ekal and R.~Ventura, ``On the accuracy of inertial parameter estimation of a
  free-flying robot while grasping an object,'' \emph{{Journal of Intelligent
  \& Robotic Systems}}, pp. 1--11, 2019.

\bibitem{Zhang2021}
W.~Zhang, M.~Tognon, L.~Ott, R.~Siegwart, and J.~Nieto, ``Active model learning
  using informative trajectories for improved closed-loop control on real
  robots,'' in \emph{{Proc.\ IEEE Conf.\ on Robotics and Automation}}, 2021.

\bibitem{klenske2016dual}
E.~D. Klenske and P.~Hennig, ``Dual control for approximate {Bayesian}
  reinforcement learning,'' \emph{{Journal of Machine Learning Research}},
  vol.~17, no.~1, pp. 1--30, 2016.

\bibitem{arcari2019}
E.~Arcari, L.~Hewing, and M.~N. Zeilinger. (2019) An approximate dynamic
  programming approach for dual stochastic model predictive control. Available
  at \url{http://arxiv.org/abs/1911.03728}.

\bibitem{Geibel2005}
P.~Geibel and F.~Wysotzki, ``Risk-sensitive reinforcement learning applied to
  control under constraints,'' \emph{{Journal of Artificial Intelligence
  Research}}, vol.~34, pp. 81--108, 2020.

\bibitem{Akametalu2014}
A.~K. Akametalu, J.~F. Fisac, J.~H. Gillula, S.~Kaynama, M.~N. Zeilinger, and
  C.~J. Tomlin, ``Reachability-based safe learning with {Gaussian} processes,''
  in \emph{{Proc.\ IEEE Conf.\ on Decision and Control}}, 2014.

\bibitem{Berkenkamp2016ROA}
F.~Berkenkamp, R.~Moriconi, A.~P. Schoellig, and A.~Krause, ``Safe learning of
  regions of attraction for uncertain, nonlinear systems with {Gaussian}
  processes,'' in \emph{{Proc.\ IEEE Conf.\ on Decision and Control}}, 2016.

\bibitem{Chen2018}
Y.~Chen, H.~Peng, J.~Grizzle, and N.~Ozay, ``Data-driven computation of minimal
  robust control invariant set,'' in \emph{{Proc.\ IEEE Conf.\ on Decision and
  Control}}, 2018.

\bibitem{Jankovic2018}
M.~Jankovic, ``Robust control barrier functions for constrained stabilization
  of nonlinear systems,'' \emph{{Automatica}}, vol.~96, pp. 359--367, 2018.

\bibitem{Chen2021}
S.~Chen, M.~Fazlyab, M.~Morari, G.~J. Pappas, and V.~M. Preciado, ``Learning
  region of attraction for nonlinear systems,'' in \emph{{Proc.\ IEEE Conf.\ on
  Decision and Control}}, 2021.

\bibitem{HarrisonSharmaEtAl2019}
J.~Harrison, A.~Sharma, R.~Calandra, and M.~Pavone, ``Control adaptation via
  meta-learning dynamics,'' in \emph{NeurIPS Workshop on Meta-Learning}, 2018.

\bibitem{hochreiter2001learning}
S.~Hochreiter, A.~S. Younger, and P.~R. Conwell, ``Learning to learn using
  gradient descent,'' in \emph{International Conference on Artificial Neural
  Networks}, 2001.

\bibitem{harrison2021uncertainty}
J.~Harrison, ``Uncertainty and efficiency in adaptive robot learning and
  control,'' Ph.D. dissertation, Stanford University, 2021.

\bibitem{Snoek2015}
J.~Snoek, O.~Rippel, K.~Swersky, N.~Satish, N.~Sundaram, M.~M.~A. Patwary,
  P.~Prabhat, and R.~P. Adams, ``Scalable {Bayesian} optimization using deep
  neural networks,'' in \emph{{Int.\ Conf.\ on Learning Representations}},
  2015.

\bibitem{kuleshov2018accurate}
V.~Kuleshov, N.~Fenner, and S.~Ermon, ``Accurate uncertainties for deep
  learning using calibrated regression,'' in \emph{{Int.\ Conf.\ on Machine
  Learning}}, 2018.

\bibitem{Srinivas2010}
N.~Srinivas, A.~Krause, S.~Kakade, and M.~Seeger, ``Gaussian process
  optimization in the bandit setting: No regret and experimental design,'' in
  \emph{{Int.\ Conf.\ on Machine Learning}}, 2010.

\bibitem{Jia2019}
K.~Jia, S.~Li, Y.~Wen, T.~Liu, and D.~Tao, ``Orthogonal deep neural networks,''
  \emph{{IEEE Transactions on Pattern Analysis \& Machine Intelligence}}, 2020.

\bibitem{BanerjeeHarrisonEtAl2020}
S.~Banerjee, J.~Harrison, P.~M. Furlong, and M.~Pavone, ``Adaptive
  meta-learning for identification of rover-terrain dynamics,'' in \emph{{Int.\
  Symp.\ on Artificial Intelligence, Robotics and Automation in Space}}, 2020.

\bibitem{GreeneEconometrics}
W.~H. Greene, \emph{Econometric Analysis}, 5th~ed.\hskip 1em plus 0.5em minus
  0.4em\relax {Prentice Hall}, 2002.

\bibitem{randSets}
T.~Lew and M.~Pavone, ``Sampling-based reachability analysis: A random set
  theory approach with adversarial sampling,'' in \emph{{Conf. on Robot
  Learning}}, 2020.

\bibitem{SinghChowEtAl2018b}
S.~Singh, Y.-L. Chow, A.~Majumdar, and M.~Pavone, ``A framework for
  time-consistent, risk-sensitive model predictive control: Theory and
  algorithms,'' \emph{{IEEE Transactions on Automatic Control}}, vol.~64,
  no.~7, pp. 2905--2912, 2018.

\bibitem{RawlingsMayne2013}
J.~Rawlings and D.~Mayne, \emph{Model predictive control: Theory and
  design}.\hskip 1em plus 0.5em minus 0.4em\relax {Nob Hill Publishing}, 2013.

\bibitem{Limon2009}
D.~Limon, T.~Alamo, D.~M. Raimondo, D.~Mu{\~{n}}oz de~la Pe{\~{n}}a, J.~M.
  Bravo, A.~Ferramosca, and E.~F. Camacho, \emph{Input-to-State Stability: A
  Unifying Framework for Robust Model Predictive Control}.\hskip 1em plus 0.5em
  minus 0.4em\relax {Springer Berlin Heidelberg}, 2009, pp. 1--26.

\bibitem{bar1974dual}
Y.~Bar-Shalom and E.~Tse, ``Dual effect, certainty equivalence, and separation
  in stochastic control,'' \emph{{IEEE Transactions on Automatic Control}},
  vol.~19, no.~5, pp. 494--500, 1974.

\bibitem{MacKay1992}
D.~J.~C. MacKay, ``Information-based objective functions for active data
  selection,'' \emph{{Neural Computation}}, vol.~4, no.~4, pp. 590--604, 1992.

\bibitem{Chowdhury2017}
A.~Chowdhury, S.~R.~Gopalan, ``On kernelized multi-armed bandits,'' in
  \emph{{Int.\ Conf.\ on Machine Learning}}, 2017.

\bibitem{Onoresolvability2012}
M.~Ono, ``Joint chance-constrained model predictive control with probabilistic
  resolvability,'' in \emph{{American Control Conference}}, 2012.

\bibitem{HJIoverviewBansal2017}
S.~Bansal, S.~L. Chen, M.~Herbert, and C.~J. Tomlin, ``{Hamilton}-{Jacobi}
  reachability: A brief overview and recent advances,'' in \emph{{Proc.\ IEEE
  Conf.\ on Decision and Control}}, 2017.

\bibitem{fan2020deep}
D.~D. Fan, A.~Agha-mohammadi, and E.~A. Theodorou, ``Deep learning tubes for
  tube {MPC},'' in \emph{{Robotics: Science and Systems}}, 2020.

\bibitem{IvanovVerisig2019}
R.~Ivanov, J.~Weimer, R.~Alur, G.~J. Pappas, and I.~Lee, ``Verisig: verifying
  safety properties of hybrid systems with neural network controllers,'' in
  \emph{{Hybrid Systems: Computation and Control}}, 2019.

\bibitem{LewJansonEtAl2021}
T.~Lew, L.~Janson, R.~Bonalli, and M.~Pavone, ``A simple and efficient
  sampling-based algorithm for general reachability analysis,'' 2021, available
  at \url{https://arxiv.org/abs/2112.05745}.

\bibitem{MaoSzmukEtAl2016}
Y.~Mao, M.~Szmuk, and B.~A{\c{c}}ikme{\c{s}}e, ``Successive convexification of
  non-convex optimal control problems and its convergence properties,'' in
  \emph{{Proc.\ IEEE Conf.\ on Decision and Control}}, 2016.

\bibitem{BonalliCauligiEtAl2019}
R.~Bonalli, A.~Cauligi, A.~Bylard, and M.~Pavone, ``{GuSTO:} guaranteed
  sequential trajectory optimization via sequential convex programming,'' in
  \emph{{Proc.\ IEEE Conf.\ on Robotics and Automation}}, 2019.

\bibitem{StellatoBanjacEtAl2017}
B.~Stellato, G.~Banjac, P.~Goulart, A.~Bemporad, and S.~Boyd, ``{OSQP}: An
  operator splitting solver for quadratic programs,'' 2017, {Available at }
  \url{https://arxiv.org/abs/1711.08013}.

\bibitem{MurrayManipulation1994}
R.~Murray, S.~S. Sastry, and L.~Zexiang, \emph{A Mathematical Introduction to
  Robotic Manipulation}.\hskip 1em plus 0.5em minus 0.4em\relax {CRC Press},
  1994.

\bibitem{AstrobeeSoftware2018}
L.~Fluckiger, K.~Browne, B.~Coltin, J.~Fusco, T.~Morse, and A.~Symington,
  ``Astrobee robot software: Enabling mobile autonomy on the {ISS},'' in
  \emph{{Int.\ Symp.\ on Artificial Intelligence, Robotics and Automation in
  Space}}, 2018.

\bibitem{Rahimi2007}
A.~Rahimi and B.~Ben~Recht, ``Random features for large-scale kernel
  machines,'' in \emph{{Conf.\ on Neural Information Processing Systems}},
  2007.

\bibitem{Wilson2020}
J.~T. Wilson, V.~Borovitskiy, A.~Terenin, P.~Mostowsky, and M.~P. Deisenroth,
  ``Efficiently sampling functions from {Gaussian} process posteriors,'' in
  \emph{{Int.\ Conf.\ on Machine Learning}}, 2020.

\bibitem{williams2017information}
G.~Williams, N.~Wagener, B.~Goldfain, P.~Drews, J.~M. Rehg, B.~Boots, and E.~A.
  Theodorou, ``Information theoretic mpc for model-based reinforcement
  learning,'' in \emph{{Proc.\ IEEE Conf.\ on Robotics and Automation}}, 2017.

\bibitem{Solowjow2020}
F.~Solowjow and S.~Trimpe, ``Event-triggered learning,'' \emph{{Automatica}},
  vol. 117, 2020.

\bibitem{HarrisonSharmaEtAl2020}
J.~Harrison, A.~Sharma, C.~Finn, and M.~Pavone, ``Continuous meta-learning
  without tasks,'' in \emph{{Conf.\ on Neural Information Processing Systems}},
  2020.

\end{thebibliography}

\appendix

\subsection{Proof of Theorem \ref{thm:conf-sets}}

The proof of Theorem \ref{thm:conf-sets} follows from the proof of \cite[Theorem 2]{abbasi2011improved}, by making substitutions accordingly for our meta-learning model. To do so, we use the following lemma, which \rev{corresponds to } \cite[Theorem 1]{abbasi2011improved}. \rev{We restate this result below by changing notations to match our problem setting, where the index $i=1,\mydots,\xdim$ corresponds to each state component.} %

\begin{mylemma}\rev{\textnormal{\textbf{\cite[Theorem 1]{abbasi2011improved} }} (Self-Normalized Bound for Vector-Valued Martingales).}
\label{lemma:self-normalized}
Let $\{ \mathcal{F}_t \}_{t=0}^\infty$ be a filtration \rev{and $i\in\mathbb{N}$}. 
\rev{Let}  $\{\rev{\epsilon_{i,t}}\}_{t=1}^\infty$ \rev{be} a  real-valued stochastic process such that $\rev{\epsilon_{i,t}}$ is $\mathcal{F}_t$-measurable and conditionally $\sigma_i$-subgaussian\rev{, such that $\E[\exp(\lambda\epsilon_{i,t})\,|\,\mathcal{F}_{t-1}]\leq \exp(\lambda^2\sigma_i^2/2)$ for all $\lambda\in\R$}. 
Let $\{\feat_t\}_{t=1}^\infty$ be a $\R^d$-valued stochastic process such that $\feat_t$ is $\mathcal{F}_{t-1}$-measurable. 

Let $\bLambda_{i,0}$ be a $d\times d$ positive definite matrix, and define
$\bLambda_{i,t}$ as in \eqref{eq:bayesian_updates}. 
Further, for any $t \ge 0$, define $\bS_t = \sum_{s=1}^t \rev{\epsilon_{i,s}} \feat_s$. 
Then, for any $\delta_i > 0$, with probability at least $(1-\delta_i)$
\begin{align}
    \norm{\bS_t}_{\bLambda^{-1}_{i,t}}^2 &\le 
    2 
    \sigma_i^2 
    \log \left( 
        \frac{1}{\delta}
        \frac{\det (\bLambda_{i,t})^{1/2}}{\det (\bLambda_{i,0})^{1/2}}
         \right) 
         \ \ 
         \forall t\geq 0.
         \label{eq:self_normalized_vec_mart}
\end{align}
\end{mylemma}
\begin{proof}
\rev{
See \cite[Theorem 1]{abbasi2011improved}.
}
\end{proof}
\rev{In this work, t}he filtration $\{\mathcal{F}_t\}_{t=0}^{\infty}$ \rev{corresponds to the  natural filtration associated to the stochastic process $\x_t$ satisfying \eqref{eq:full_problem_dynamics}, which} is defined by considering the $\sigma$-algebra $\mathcal{F}_t=\sigma(\feat_1,\mydots,\feat_{t+1},\rev{\epsilon_{i,0}},\mydots,\rev{\epsilon_{i,t}})$, where $\feat_t=\feat(\x_t,\ac_t)$.  
\rev{Lemma \ref{lemma:self-normalized} corresponds to \cite[Theorem 1]{abbasi2011improved} with slightly modified notations, }%
substituting
$(X,\eta,\theta,\bar{V}_t,V)$ with $(\feat,\rev{(\epsilon_{i,1}, \mydots,\epsilon_{i,t})},\k_i,\bLambda_{i,t},\bLambda_{i,0})$ \rev{for any arbitrary $i=1,\mydots,n$}. 

We stress that \eqref{eq:self_normalized_vec_mart} holds jointly for all times $t\,{\geq}\, 0$ such that 
$\Prob(\eqref{eq:self_normalized_vec_mart})\,{\geq}\, (1\,{-}\,\delta_i)$. This result is key to ensure that the trajectory-wise chance constraint \eqref{eq:joint_chance_constraint} is satisfied and to guarantee the safety and recursive feasibility of our framework. 
Next, we restate and prove Theorem \ref{thm:conf-sets}. 

\setcounter{theorem}{0}
\begin{theorem}[Uniformly Calibrated Confidence Sets]
Consider the true system  \eqref{eq:full_problem_dynamics},
\begin{equation*}
    \x_{t+1} = \h(\x_t, \ac_t) + \g(\x_t,\ac_t,\param) + \ep_t,
\end{equation*}
where each $\rev{\epsilon_{i,t}}$ is $\sigma_i$-subgaussian and bounded ($\rev{\epsilon_{i,t}}\in\mathcal{E}_i$).   
Consider the meta-learning model  $\hat{\g}= (\hat{g}_1,{\scalebox{0.9}{$\mydots$}},\hat{g}_n)$ 
defined for each 
\scalebox{1}{$i\,{=}\,1,\mydots,n$} as $\hat{g}_i(\x, \ac,\k_i) = \k_i^\top\feat_i(\x,\ac)$, where $\feat_i: \R^n\,{\times}\,\R^m\,{\rightarrow}\,\R^{\phidim}$ and   
$\k_i\,{\in}\,\R^\phidim$. %
Starting from $(\kbar_{i,0},\bLambda_{i,0})$, with $\kbar_{i,0}\in\R^{\phidim}$, and $\bLambda_{i,0}$ is a $d\times d$ positive definite matrix, define the sequence $\{(\kbar_{i,s},\bLambda_{i,s})\}_{s=0}^t$,  
where $(\kbar_{i,t}, \bLambda_{i,t})$ is computed with online data from \eqref{eq:full_problem_dynamics} using \eqref{eq:bayesian_updates}:
\begin{equation*}
\begin{array}{l}
    \bLambda_{i,t} = \Phi_{i,t-1}^\top \Phi_{i,t-1} \,{+}\, \bLambda_{i,0}, 
    \\[1mm]
    \kbar_{i,t} = \bLambda_{i,t}^{-1} ( \Phi^\top_{i,t-1} \G_{i,t} \,{+}\, \bLambda_{i,0} \kbar_{i,0} ), 
\end{array}
    \ \  
    i\,{=}\,1,\mydots,\xdim,
\end{equation*}
where $\G_{i,t}^\top {=}\, [x_{i,1} {-} h_i(\x_{0},\ac_{0}), \mydots, x_{i,t} {-} h_i(\x_{t{-}1},\ac_{t{-}1})] \,{\in}\, \R^{t}$, and
$\Phi_{i,t{-}1}^\top {=} [\feat_i(\x_0, \ac_0), \mydots, \feat_i(\x_{t{-}1},\ac_{t{-}1})]{\in} \R^{\phidim \times t}$. 
Further, define $\delta_i=\delta/(2\xdim)$, and let
\begin{gather*}
    \beta_{i,t}^\delta 
    \,{=}\,  \sigma_i \Bigg(
    \sqrt{ 2 \log \left(\frac{1}{\delta_i}\frac{\det(\bLambda_{i,t})^{\nicefrac{1}{2}} }{\det(\bLambda_{i,0})^{\nicefrac{1}{2}}}
    \right) }
    {+}
    \sqrt{\frac{\bar{\lambda}(\bLambda_{i,0})}{\underline\lambda(\bLambda_{i,t})}  \chi^2_{\phidim}(1{-}\delta_i)} \Bigg),
    \\[0mm]
    \text{and}\qquad
    \confset_{i,t}^{\delta}(\kbar_{i,t},\bLambda_{i,t}) \,{=}\, \{ \k_i \, {\mid} \,
        \norm{\k_i {-} \kbar_{i,t}}_{\bLambda_{i,t}} {\leq}\,  \beta_{i,t}^\delta
     \}.
\end{gather*}
Then, under Assumptions \ref{assum:capacity_model} and \ref{assum:calib}, 
\begin{gather*}
\Prob \left(
    \kstari \in \confset_{i,t}^{\delta}(\kbar_{i,t},\bLambda_{i,t})
    \ \ 
    \forall t\,{\geq}\, 0
    \right) \geq (1-2\delta_i). 
\end{gather*}
\end{theorem}

\begin{proof}
This proof is a straightforward extension of \cite[Theorem 2]{abbasi2011improved}, where we use Assumption \ref{assum:calib} to provide a probabilistic error bound for the model mismatch over the prior for $\kstari$, and Lemma \ref{lemma:self-normalized} to bound the estimation error due to $\ep_t$ \rev{with a union bound} to obtain $\beta_i$.

Define $\ep^i = (\rev{\epsilon_{i,1}}, \mydots, \rev{\epsilon_{i,t}})^\top$. 
For conciseness, we drop the indices $i$ and $t$ and denote
$(\kstar,\kbar,\bLambda,\kbar_0,\bLambda_0,\Phi,\ep) 
= 
(\kstari,\kbar_{i,t},\bLambda_{i,t},\kbar_{i,0},\bLambda_{i,0},\Phi_{i,t-1},\ep^i)$.  
Under Assumption \ref{assum:capacity_model}, we can write 
$\G_{i,t}=\Phi \kstar + \ep$. Then we rewrite the mean estimate $\kbar$ of $\kstar$ at time $t$ as
\begin{align*}
     \kbar &= ( \bLambda_0 + \Phi^\top \Phi )^{-1} (\bLambda_0 \kbar_0 +  \Phi^\top ( \Phi \kstar + \ep ) )\\
     &= ( \bLambda_0 + \Phi^\top \Phi )^{-1} \Phi^\top \ep + ( \bLambda_0 + \Phi^\top \Phi )^{-1} ( \bLambda_0 + \Phi^\top \Phi ) \kstar 
     \\
     &
     	\quad- ( \bLambda_0 + \Phi^\top \Phi )^{-1}  \bLambda_0 (\kstar - \kbar_0) \\
     &= \Linv \Phi^\top \ep + \kstar - \Linv \bLambda_0 ( \kstar - \kbar_0 ),
\end{align*}
from which we obtain, for any $\ba\in\R^{\phidim}$, that 
$\ba^\top ( \kbar - \kstar ) 
= 
\ba^\top (\Linv \Phi^\top \ep)  -\ba^\top (\Linv \bLambda_0 (\kstar - \kbar_0)).
    $  
With this result, by the Cauchy-Schwarz inequality,
\begin{align}
| \ba^\top ( \kbar - \kstar ) | 
&\leq 
\norm{\ba}_{\Linv_t} \left( \norm{\Phi^\top \epsilon}_{\Linv} + \norm{\bLambda_0 (\kstar - \kbar_0)}_{\Linv} \right) 
\nonumber
\\
&\hspace{-1.6cm}\leq \norm{\ba}_{\Linv} \bigg( \norm{\Phi^\top \ep}_{\Linv} {+} \sqrt{\frac{\maxeig{\bLambda_0}}{\mineig{\bLambda}}}\norm{\kstar {-} \kbar_0}_{\bLambda_0} \bigg),
\label{eq:cauchy}
\end{align}
where the second inequality is obtained as
\begin{align*}
&
\scalebox{0.9}{$
\norm{\bLambda_0 (\kstar - \kbar_0)}^2_{\Linv} 
\le 
\frac{\maxeig{\Linv}}{\mineig{\Linv_0}} \norm{\bLambda_0 (\kstar - \kbar_0)}^2_{\Linv_0}
$}
\\
&\hspace{10mm}
\scalebox{0.9}{$
=
\frac{\maxeig{\bLambda_0}}{\mineig{\bLambda}} \norm{\bLambda_0 (\kstar - \kbar_0)}^2_{\Linv_0} 
=
\frac{\maxeig{\bLambda_0}}{\mineig{\bLambda}} \norm{\kstar - \kbar_0}^2_{\bLambda_0}
$}
.
\end{align*}

By Lemma \ref{lemma:self-normalized}, for any $\delta_i$, 
with probability at least $(1-\delta_i)$, 
\begin{align}
\norm{\Phi^\top \ep}_{\Linv}^2 &\leq
2\sigma_i^2
\log \left( 
    \frac{1}{\delta_i}
    \frac{\det (\bLambda)^{1/2}}{\det (\bLambda_0)^{1/2}}
     \right)
     \quad \forall t\geq 0
     \rev{.}
     \label{eq:foot:a}
\end{align} 
By Assumption \ref{assum:calib}, for $\delta_i =\delta/(2\xdim)$, with probability at least $(1-\delta_i)$, 
\vspace{-3mm}
\begin{align}
\|\kstar - \kbar_{0}\|_{\bLambda_{0}}^2 \leq \sigma_i^2 \chi^2_{\phidim}(1-\delta_i)
     \label{eq:foot:b}
    .
\end{align}

From \eqref{eq:cauchy}, by \rev{a union bound},\footnote{
$\Prob(\eqref{eq:foot:a} \cap \eqref{eq:foot:b})
=1-\Prob(\eqref{eq:foot:a}^C \cup \eqref{eq:foot:b}^C)$, where $A^C$ denotes the negation of $A$. 
Then, by Boole's inequality, 
$1-\Prob(\eqref{eq:foot:a}^C \cup \eqref{eq:foot:b}^C)
\geq
1-\Prob(\eqref{eq:foot:a}^C) - \Prob(\eqref{eq:foot:b}^C)
=
-1 +\Prob(\eqref{eq:foot:a}) + \Prob(\eqref{eq:foot:b})
$. 
Finally, using the lower bounds on the probabilities that \eqref{eq:foot:a} and \eqref{eq:foot:b} occur, we obtain 
$\Prob(\eqref{eq:foot:a} \cap \eqref{eq:foot:b})
\geq -1 +(1-\delta_i) + (1-\delta_i) = 1-2\delta_i$.
} we have that with probability at least $(1-2\delta_i)$, for all $t\geq 0$ and any $\ba\in\R^{\phidim}$, 
\begin{align*}
| \ba^\top ( \kbar - \kstar ) | 
&\le 
\norm{\ba}_{\Linv} 
\sigma_i  
\Bigg( 
   \sqrt{ 2 \log \left(  \frac{1}{\delta_i}  \frac{\det (\bLambda)^{1/2}}{\det (\bLambda_0)^{1/2}} \right) }
    \\
    &\qquad\qquad\qquad+
    \sqrt{\frac{\maxeig{\bLambda_0}}{\mineig{\bLambda}}\chi^2_{\phidim}(1-\delta_i)}
\Bigg).
\end{align*}
Define $\beta_{i,t}^\delta$ as in \eqref{eq:beta_bound}. Then letting $\ba= \bLambda (\kbar - \kstar)$ in the expression above, we obtain
\begin{align*}
\norm{\kbar - \kstar}^2_{\bLambda} 
&\le 
\norm{\bLambda (\kbar - \kstar)}_{\Linv} 
\beta_{i,t}^\delta.
\end{align*}
Since $\norm{\bLambda (\kbar - \kstar)}_{\Linv} = \norm{\kbar_t - \kstar}_{\bLambda}$, we divide both sides by $\norm{\kbar - \kstar}_{\bLambda}$ and obtain
\begin{align*}
    \norm{\kbar - \kstar}_{\bLambda} 
    &\le 
    \beta_{i,t}^\delta
\quad
\forall t\geq 0,
\end{align*}
which holds with probability at least $(1-2\delta_i)$. Using the definition of the confidence sets $\confset_{i,t}^{\delta}$ in  \eqref{eq:conf_set_kappa} and the result above, this concludes our proof.
\end{proof}

\rev{
\subsection{Relaxing Assumption \ref{assum:capacity_model} and extending Theorem \ref{thm:conf-sets}}\label{sec:appendix:proof_missmatch}}

\rev{Our probabilistic guarantees rely on the assumption that the meta-learning model is capable of fitting the true dynamics. 
Specifically, in Assumption \ref{assum:capacity_model}, we assume that for 
all $\param$ and all $i = 1,\mydots,n$,  there exists $\kstari(\param) \in \R^{\phidim}$  such that 
$$ 
|\kstari(\param)^\top \feat_i(\x, \ac)  -
g_i(\x,\ac,\param)|\leq \Delta_i \ \  
\forall\x \in \X, 
\forall\ac \in \U,
$$
where $\Delta_i=0$ for all $i=1,\dots,n$. 
This assumption of the true model belonging to the model class is standard in the literature (see Sections \ref{sec:probabilistic_adapt_guarantees} and \ref{sec:regularizers}) and reasonable in practice. 
Indeed, our simulated and hardware experiments verified the safety guarantees of the \seels under this assumption, which shows that assuming $\Delta_i=0$ is not limiting when using expressive meta-learning models. 
Nevertheless, our results can be generalized to the case where $\Delta_i>0$. 
Indeed, the proof of Theorem \ref{thm:conf-sets} follows by adapting the proof of \cite[Theorem 2]{abbasi2011improved} to use the prior of our Bayesian model in Assumption \ref{assum:calib} and our particular model structure. 
As such, below, we extend \cite[Theorem 2]{abbasi2011improved} to also account for bounded model mismatch. Due to space constraints, we leave the simple extensions of Theorem \ref{thm:conf-sets} and \ref{thm:probabilistic_feasibility_safety} (and the evaluation of the bounds $\Delta_i$, e.g.,  depending on the boundedness of the features $\feat_i(\cdot)$ and the chosen meta-learning architecture) for future work.
}

\rev{
\begin{mycorollary}[Confidence sets for linear models corrupted by subgaussian and bounded noise]\label{cor:conf_sets_robust}
On a probability space $(\Omega,\mathcal{G},\Prob)$, 
let $(\mathcal{F}_t)_{t\in\mathbb{N}}$ be a filtration. 
Let 
$(\epsilon_t)_{t\in\mathbb{N}}$ and $(\gamma_t)_{t\in\mathbb{N}}$ be real-valued stochastic processes adapted to $\mathcal{F}_t$ such that
\begin{itemize}[leftmargin=3.5mm]
  \setlength\itemsep{0.5mm}
\item 
	$\E[\epsilon_t|\mathcal{F}_{t-1}]\leq \exp\left(\lambda^2R^2/2\right)$ for all $\lambda\in\R$ for all $t\in\mathbb{N}$.
\item 
	$|\gamma_t|\leq \Delta$ almost surely for some constant $\Delta<\infty$. 
\end{itemize}
Let $(x_t)_{t\in\mathbb{N}}$ be a $\R^d$-valued stochastic process adapted to $\mathcal{F}_{t-1}$.  
For fixed parameters $\theta\in\R^d$, consider the linear model
$$
y_t = 
\theta^\top x_t 
+ \gamma_t + \epsilon_t
$$
and assume that $\|\theta\|_2\leq S$.  
Given $t$ observations tuples $(x_s,y_s)_{s=1}^t$, define
$X_t=[x_1,\mydots,x_t]^\top\in\R^{t\times d}$, 
$Y_t=[y_1,\mydots,y_t]^\top\in\R^t$, 
 and 
\begin{align*}
    \bar\theta_t = (X_t^\top X_t+\lambda I)^{-1}X_t^\top Y_t,
    \quad
    \bar{V}_t = X_t^\top X_t + \lambda I.
\end{align*}
	Given $\delta\in(0,1)$, 
define the ellipsoidal confidence set
\begin{align*}%
\bar{\mathcal{C}}_t^\delta
=
\left\{\theta \in \R^d :
\|\theta-\bar\theta_t\|_{\bar{V}_t} 
    \leq
    R\beta_t
    +
    \lambda^{\frac{1}{2}}S
    +
    D_t
\right\},
\end{align*} 
where 
	$\beta_t^2 
	=	
	2 \log \left(\frac{1}{\delta}
	\left(\frac{\det(\bar{V}_t)}{\det(\lambda I)}\right)^{\frac{1}{2}}
    \right)
    $ and 
$D_t = \sigma_{\textrm{max}}(L_t^\top)  
 t^{1/2} \Delta$ where $L_t\in\R^{d\times d}$ is such that $\mathcal{X}_t^\top \bar{V}_t \mathcal{X}_t=L_t L_t^\top$ with $\mathcal{X}_t=(X_t^\top X_t+\lambda I)^{-1}X_t^\top$. 
Then,
\begin{align*}%
\Prob \left(
    \theta\in\bar{\mathcal{C}}_t^\delta
    \ 
    \forall t\,{\geq}\, 0
    \right) \geq 1-\delta.
\end{align*}
\end{mycorollary}
\begin{proof} 
Define 
$\bar{y}_t = \theta^\top x_t + \epsilon_t$ so that  $y_t = \bar{y}_t + \gamma_t$. 
Denoting $\bar{Y}_t=[\bar{y}_1,\mydots,\bar{y}_t]$ 
 and 
 $\Gamma_t=[\gamma_1,\mydots,\gamma_t]$, we rewrite $\bar\theta_t$ as 
\begin{align*}
    \bar\theta_t &= 
    (X_t^\top X_t+\lambda I)^{-1}X_t^\top Y_t
    \\
    &=
    (X_t^\top X_t+\lambda I)^{-1}X_t^\top\bar{Y}_t
    +
    (X_t^\top X_t+\lambda I)^{-1}X_t^\top\Gamma_t
    \\
    &=
    \mathcal{X}_t\bar{Y}_t
    +
    \mathcal{X}_t\Gamma_t
    \,\triangleq \, \bar{\theta}_{\epsilon} + \Delta\bar{\theta}_\gamma,
\end{align*}
where 
$\mathcal{X}_t=(X_t^\top X_t+\lambda I)^{-1}X_t^\top$. 
We then obtain that 
\begin{align}\label{eq:gaussian_plus_rob_errors}
\hspace{-2mm}
\|\theta\,{-}\,\bar\theta_t\|_{\bar{V}_t} 
\,{=}\,
\|\theta{-}(\bar{\theta}_{\epsilon} {+} \Delta\bar{\theta}_\gamma)\|_{\bar{V}_t}
\,{\leq}\,
\|\theta\,{-}\,\bar{\theta}_{\epsilon}\|_{\bar{V}_t}
{+}
\|\Delta\bar{\theta}_\gamma\|_{\bar{V}_t}.
\end{align}
The second term 
$\|\Delta\bar{\theta}_\gamma\|_{\bar{V}_t}$ is bounded with probability one since $|\gamma_t|\leq\Delta$. 
Indeed, denoting $L_t\in\R^{d\times d}$ for the Cholesky decomposition $\mathcal{X}_t^\top \bar{V}_t \mathcal{X}_t=L_t L_t^\top$, we have that
\begin{align*}
\|\Delta\bar{\theta}_\gamma\|_{\bar{V}_t}^2
&=
\| \mathcal{X}_t \Gamma_t \|_{\bar{V}_t}^2
=
\Gamma_t^\top \mathcal{X}_t^\top \bar{V}_t \mathcal{X}_t \Gamma_t
=
\Gamma_t^\top L_tL_t^\top \Gamma_t
\\
&=
(L_t^\top\Gamma_t)^\top 
 (L_t^\top
 \Gamma_t)
 =
 \| L_t^\top  \Gamma_t \|_2^2
	\\
	&\leq 
 (\| L_t^\top \|_2 \cdot  
 \| \Gamma_t \|_2)^2
	\\
	&\leq 
 (\sigma_{\textrm{max}}(L_t^\top)  
 \sqrt{t}\| \Gamma_t \|_\infty)^2
	\\
	&\leq 
 (\sigma_{\textrm{max}}(L_t^\top)  
 t^{1/2} \Delta)^2
 = 
 D_t^2.
\end{align*}
Combining this bound that holds almost surely with the result from
Theorem \cite[Theorem 2]{abbasi2011improved} to bound the first term $\|\theta-\bar{\theta}_{\epsilon}\|_{\bar{V}_t}$ in \eqref{eq:gaussian_plus_rob_errors} at all times with probability at least $1-\delta$, this concludes our proof.
\end{proof}
}

\subsection{Robust Obstacle Avoidance Constraints Implementation}\label{sec:obs_avoidance}
We describe our implementation of obstacle avoidance constraints in the special case where all obstacles are spheres in $\R^3$ and the system is a manipulator. Other types of constraints (e.g., linear constraints), obstacles (e.g., polyhedral and non-convex obstacles), as well as point-mass systems, are particular cases of the following derivations, as we discuss next.

Consider a function $\bp:\R^n{\rightarrow}\R^3$ mapping the state of the system (its configuration) to a point that should not intersect with obstacles (e.g., for a manipulator, $\bp$ is the end-effector position as a function of the joint angles. For a point-mass system, $\bp$ maps to the positional variables only.).

Before treating the robust case under uncertainty, consider this constraint for a single state $\x\in\R^n$. 
Given a spherical obstacle 
$\smash{\Obs=\{\bar\bp\in\R^3 \ | \  \|\bar\bp-\bp_{\textrm{o}}\|_2< r_{\textrm{o}}, \,\bp_{\textrm{o}}\in\R^3,r_{\textrm{o}}> 0\}}$, we define the obstacle avoidance constraint as $\bp(\x)\notin\Obs$, which is equivalent to
$\|\bp(\x)-\bp_{\textrm{o}}\|_2-r_{\textrm{o}}\geq 0$. 
Defining 
$d(\x)=\|\bp(\x)-\bp_{\textrm{o}}\|_2-r_{\textrm{o}}$, we observe that $d(\x)\geq 0$ is a non-convex constraint. To tackle this type of constraint, sequential convex programming (SCP) consists of solving a sequence of convex programs where any non-convex constraint is linearized. Specifically, the linearization of this constraint at $\x^j\,{\in}\,\R^n$ is
$$
d(\x^j) + \smash{\left(\frac{\bp(\x^j)-\bp_{\textrm{o}}}{\|\bp(\x^j)-\bp_{\textrm{o}}\|_2}\right)^\top}
\left(\frac{\partial\bp}{\partial\x}(\x^j)\right)(\x-\x^j)\geq 0,
$$
where the term multiplying $(\x-\x^j)$ corresponds to $\smash{\frac{\partial d}{\partial\x}(\x^j)}$, 
and 
$\smash{\frac{\partial\bp}{\partial\x}(\x^j)\,{\in}\,\R^{3{\times}6}}$ denotes the Jacobian of $\bp$ evaluated at $\smash{\x^j}$. 
This constraint is linear in $\x$ and can be handled by off-the-shelf convex optimization algorithms. 
Note that the case where $\Obs$ is an arbitrary non-convex set can be handled similarly by replacing $d(\x)$ with a signed distance function, see \cite{LewBonalli2020}.

Next, we treat the robust constraint $\bp(\x)\notin\Obs\ \forall \x\in\X$, with $\X\subset\R^6$ a reachable set. 
A conservative reformulation consists of defining an outer-bounding rectangular set $\D=\smash{\{\x\in\R^6\ | \ |x_i-\mu_i|\leq \delta_i, \bmu\in\R^6,\delta_i>0\}}$ such that $\X\subseteq\D$. 
In this work, as we use a sampling-based algorithm to approximate each reachable set, $\X$ is finite, i.e., $\smash{\X\,{=}\,\{\x^j\}_{j{=}1}^M}$. Thus, we simply choose the nominal predicted state $\bmu_k$ (see \ReachOCP) for the center $\bmu$ of each rectangular set $\D$, and we choose $\delta_i=\max_j|x_i^j-\mu_i|$ for its half-width. 

Then, we define an  outer-bounding ellipsoidal set $\cE = $ $\smash{\{\x\in\R^n\ | \ (\x-\bmu)^\top\bQ^{-1}(\x-\bmu)\leq 1\}}$, where the center $\bmu$ is the same as the center of the rectangular set $\D$ and $\bQ=n\cdot\diag(\delta_i^2, i\,{=}\,1,\mydots,n)$ is a diagonal positive definite matrix.\footnote{Observe that $\bp$ typically  only depends on positional and angular variables, so the entries of the last rows of $\smash{\frac{\partial d}{\partial\x}(\x^j)}$ are zero. Thus, one can reduce conservatism by formulating a constraint with $\bQ=n_{\textrm{pos}}\cdot\smash{\diag(\delta_i^2, i\,{=}\,1,\mydots,n_{\textrm{pos}})}$ instead, where $n_{\textrm{pos}}\,{<}\,n$ (e.g., for a three-joint manipulator, $n_{\textrm{pos}}\,{=}\,3\,{<}\,6\,{=}\,n$).} This construction guarantees that $\X\subseteq\D\subseteq\cE$, so that $\cE$ outer-bounds $\X$. Finally, we use this set to reformulate the robust contraint $d(\x)\geq 0 \ \forall \x\in\X$ as 
$\smash{d(\bmu) - \sqrt{\bJ^\top\bQ\bJ}\geq 0}$, 
where 
$\bJ=\frac{\partial d}{\partial\x}(\bmu)$, 
and $\bmu$ is the nominal state of the trajectory, see \ReachOCP. 
For a justification of this reformulation, refer to \cite{koller2018,randSets,LewBonalli2020}. 
Note that this constraint is non-convex in $\bmu$, and also depends on the control trajectory $\ac$ due to the dependence on $\bQ$ (which depends on the sampled trajectory when performing reachability analysis). 
We iteratively linearize this last constraint within each SCP iteration to obtain a quadratic program which is efficiently solved using OSQP \cite{StellatoBanjacEtAl2017}. %

\begin{IEEEbiography}[{\includegraphics[width=1in,height=1.25in,clip,keepaspectratio]{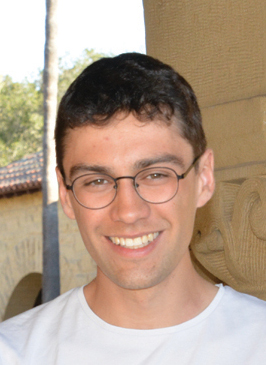}}]{Thomas Lew} is a Ph.D. candidate in Aeronautics and Astronautics at Stanford University. He received his BSc. degree in Microengineering from \'Ecole Polytechnique F\'ed\'erale de Lausanne in 2017 and his MSc. degree in Robotics from ETH Zurich in 2019. His research focuses on the intersection between optimization, control theory, and machine learning techniques for robotics and aerospace applications.
\end{IEEEbiography}

\begin{IEEEbiography}[{\includegraphics[width=1in,height=1.25in,clip,keepaspectratio]{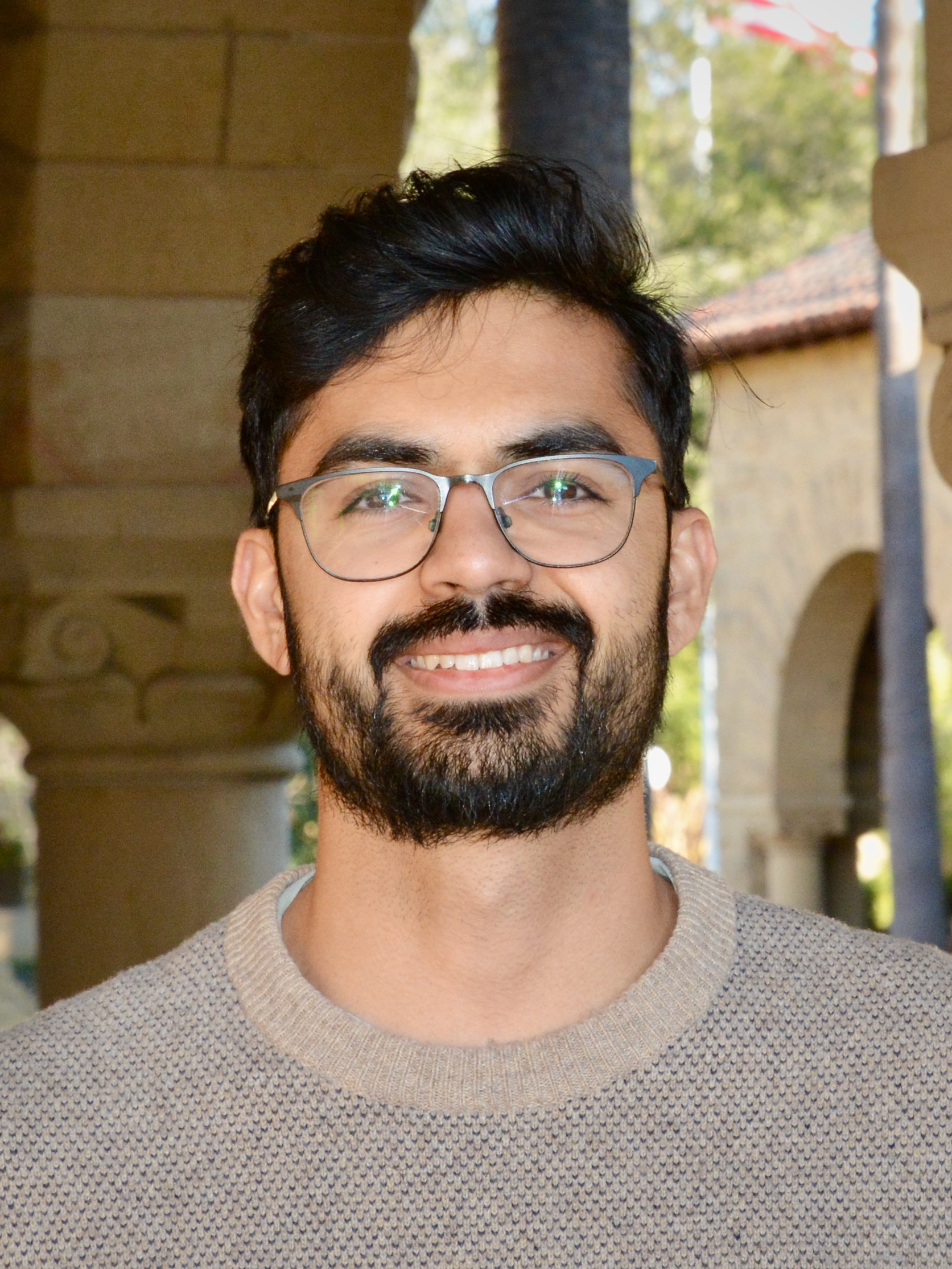}}]{Apoorva Sharma}
is a PhD candidate in the Autonomous Systems Lab at Stanford University. He received an M.S. degree in Aeronautics and Astronautics from Stanford University in 2018, and a B.S. degree in Engineering from Harvey Mudd College in 2016. His research interests center around the application of machine learning models in control and sequential decision making applications, with a focus on Bayesian methods and quantifying uncertainty to enable safe application of machine learning systems in robot autonomy and decision making. 
\end{IEEEbiography}

\begin{IEEEbiography}[{\includegraphics[width=1in,height=1.25in,clip,keepaspectratio]{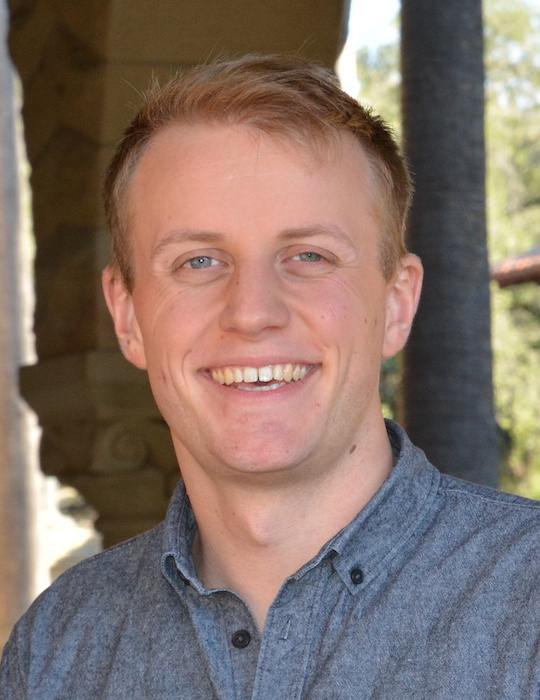}}]{James Harrison}
is a Ph.D.~candidate in the Autonomous Systems Lab at Stanford University. He received an M.S.~degree from Stanford University in 2018 and a B.Eng.~degree from McGill University in 2015, both in mechanical engineering. His research interests include few-shot, adaptive, and open-world learning, Bayesian deep learning, and applications in safe robot autonomy, decision-making, and control. 
\end{IEEEbiography}

\begin{IEEEbiography}[{\includegraphics[width=1in,height=1.25in,clip,keepaspectratio]{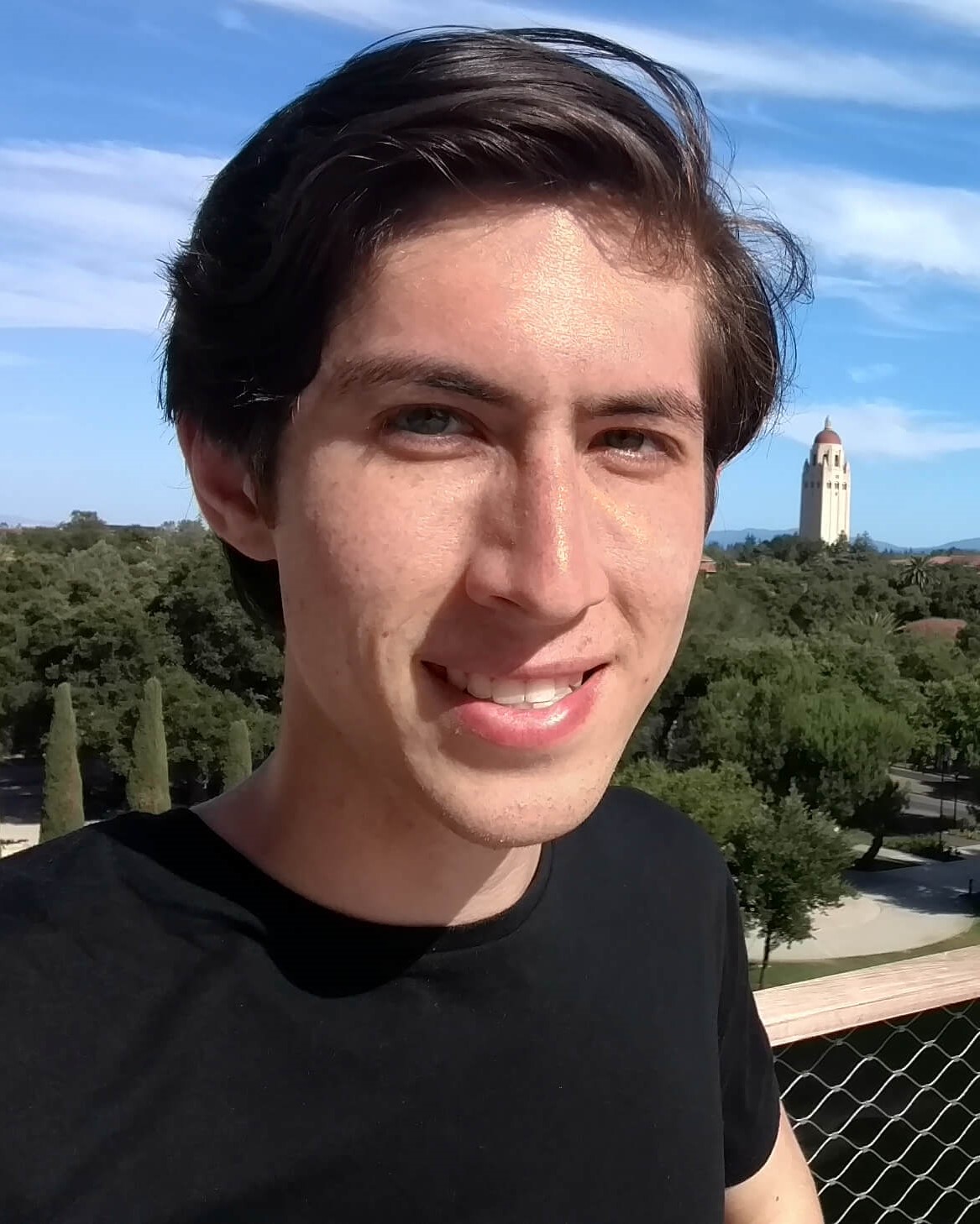}}]{Andrew Bylard} is a Ph.D. candidate in Aeronautics and Astronautics at Stanford University. He received a B.Eng. with a dual concentration in Mechanical Engineering and Electrical Engineering from Walla Walla University in 2014 and a M.Sc. in Aeronautics and Astronautics from Stanford University in 2016. Andrew's research interests include real-time trajectory planning and optimization and unconventional space robotics, including using gecko-inspired adhesives for microgravity manipulation and repurposing rollable extendable booms for long-reach mobile manipulators in reduced gravity. He has been a lead developer at the Stanford Space Robotics Facility, where he designed robot test beds used to perform spacecraft contact dynamics experiments and develop autonomous space robotics capabilities under simulated frictionless, microgravity conditions.
\end{IEEEbiography}

\FloatBarrier
\vfill
\begin{IEEEbiography}[{\includegraphics[width=1in,height=1.25in,clip,keepaspectratio]{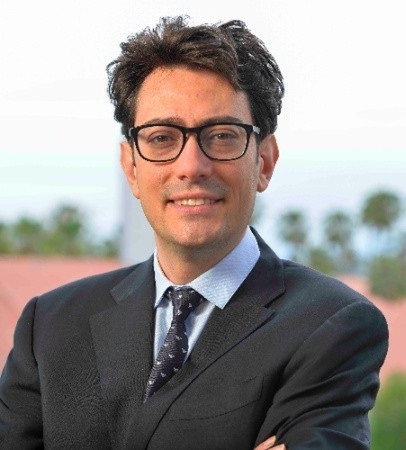}}]{Marco Pavone} 
is   an   Associate Professor of Aeronautics and Astronautics at Stanford University,  where he is the Director of the Autonomous Systems Laboratory. Before  joining  Stanford, he  was  a  Research  Technologist  within the  Robotics  Section  at  the  NASA  Jet Propulsion  Laboratory.   He  received  a Ph.D. degree in Aeronautics and Astronautics from the Massachusetts Institute of  Technology  in  2010.   His  main  research  interests  are  in the  development  of  methodologies  for  the  analysis,  design, and  control  of  autonomous  systems,  with  an  emphasis  on self-driving cars, autonomous aerospace vehicles, and future mobility systems.  He is a recipient of a number of awards, including  a  Presidential  Early  Career  Award  for  Scientists and Engineers, an ONR YIP Award, an NSF CAREER Award, and a NASA Early Career Faculty Award.  He was identified by the American Society for Engineering Education (ASEE) as one of America’s 20 most highly promising investigators under the age of 40.  He is currently serving as an Associate Editor for the IEEE Control Systems Magazine.
\end{IEEEbiography}

\FloatBarrier
\vfill
\end{document}